\documentclass{article}
\PassOptionsToPackage{dvipsnames,table}{xcolor}
\usepackage{aibl_arxiv} % loads geometry, times, xcolor, graphicx, tcolorbox,
\usepackage{url}
\usepackage{amsmath,amssymb}
\usepackage{booktabs}
\usepackage{multirow}
\usepackage{pifont}
\usepackage{titletoc}
\usepackage{float}
\usepackage{amsthm}
\usepackage{array}
\usepackage{enumitem}
\tcbuselibrary{most} % tcolorbox already loaded by aibl_arxiv
\newtheorem{proposition}{Proposition}
\theoremstyle{remark}
\newtheorem*{remark}{Remark}
\tcolorboxenvironment{proposition}{
  enhanced, breakable, colback=gray!4, colframe=black!35,
  boxrule=0.4pt, arc=1pt, left=6pt, right=6pt, top=3pt, bottom=3pt
}
\usepackage{algpseudocode}
\usepackage{algorithm}

\definecolor{ourshl}{HTML}{E2EFDA}
\newcommand{\venue}[1]{\,{\scriptsize\textcolor{gray}{#1}}}
\newcommand{\EAF}{E_{\mathrm{AF}}}
\newcommand{\Edir}{E_{\mathrm{direct}}}
\newcommand{\xsrc}{x^{\mathrm{src}}}

\begin{document}

\aiblmaketitle
  {ReState: Noise-Level-Aware State Correction for Training-Free Object Removal in Few-Step Diffusion Models}
  {Donghwan Lee$^{1,*}$, \ Sanghyun Jo$^{1, 2,*}$, \ Sihyun Kim$^{1}$, \ Ziseok Lee$^{1}$, \ Kyungsu Kim$^{1,\dagger}$}
  {$^{1}$Seoul National University \quad $^{2}$OGQ\\[0.15em]
   {\small $^{*}$equal contribution, $^{\dagger}$corresponding author: \texttt{kyskim@snu.ac.kr}}}
  {Training-free object removal commonly redirects self-attention away from the target region, but on few-step distilled diffusion models this often leaves object-shaped residuals. We identify a temporal mismatch: a structured prediction forms inside the mask at the first forward, and the short trajectory offers few later chances to revise it, while attention rerouting acts mainly in later steps. Across four distilled backbones, magnitude-matched single-forward interventions show that background-directed state correction affects the final result more than attention rerouting. We therefore propose ReState, which corrects the masked clean prediction toward background evidence retrieved from the input using a noise-level-aware weight that vanishes at the final forward. ReState retains object-excluding attention as a complementary constraint and adds no denoiser evaluations. Across two paired benchmarks and two few-step backbones, ReState achieves the best masked PSNR, SSIM, and LPIPS among training-free methods, improving PSNR by $3.17$~dB on average over the strongest baseline. On OBER-Test, it also outperforms DesignEdit and Attentive Eraser at their native step counts on all three metrics, at more than $12\times$ lower latency. The same configuration transfers to two additional distilled backbones without retuning. Code will be released upon publication.}
  {\aiblresource{Project Page}{https://tdrlee.github.io/ReState}
  \aiblresource{Code}{https://github.com/tdrlee/ReState}
   \aiblresource{Lab}{https://aibl.snu.ac.kr}}
\suppressfloats[t]
% ==================================================================
% Figure 1: motivation -----------------------------------------------------
\begin{figure}[t]
\centering
\includegraphics[width=\linewidth]{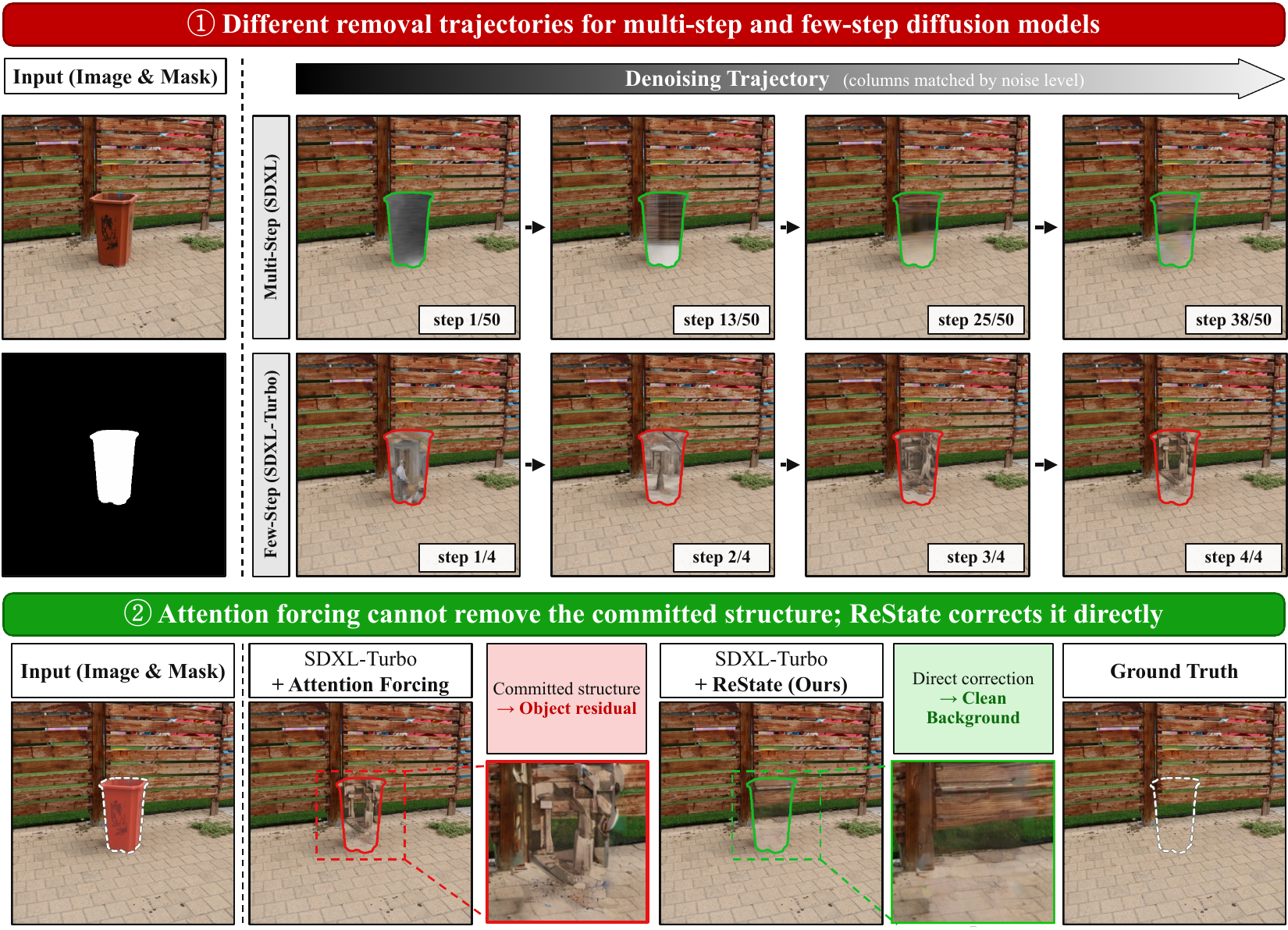}

\caption{\textbf{Motivation.} With attention forcing applied, the multi-step model (SDXL) starts diffuse and forms the background gradually along the trajectory (top), while the few-step model (SDXL-Turbo) already fixes a structured image in the mask at the first forward, and the short trajectory does not absorb it. This committed structure is what attention forcing then fails to remove, leaving a residual (bottom); ReState instead corrects the prediction directly and removes it.}
\label{fig:teaser}
\vspace{-13pt}
\end{figure}

\section{Introduction}
\label{sec}

Object removal aims to reconstruct the background occluded by a masked object. Training-free diffusion methods rely on pretrained generative priors without task-specific fine-tuning. Most existing training-free methods achieve object removal by intervening in self-attention: they suppress attention to target-region keys so that the network draws information from the surrounding context~\citep{designedit2024,attentiveeraser2025,adaeraser2026,pandora2026}. We refer to this family of interventions as \emph{attention forcing} (Section~\ref{sec:af}). In their multi-step setting, these methods typically require tens of denoising forwards. Native few-step distilled models offer a substantially faster alternative, but attention forcing does not transfer cleanly to this regime: although it improves removal over no intervention, object contours and textures often remain inside the mask. The central question is therefore not only how to redirect attention, but why this intervention fails to alter content to which a short trajectory has already committed.
To answer this question, we compare matched single-forward interventions (Figure~\ref{fig:teaser}). From the same latent at the same noise level, we either apply object-excluding attention for one forward or directly move the masked clean prediction toward retrieved background evidence by exactly the same $\ell_2$ norm. Sampling then continues without further intervention. Across four distilled backbones, direct correction at the first forward produces a larger terminal removal effect than attention forcing. On the Turbo backbones, attention forcing at the first forward has no measurable effect on the final image, whereas direct correction of the same norm does. These results reveal a temporal mismatch: early in a few-step trajectory, changing the predicted state can affect the final image directly, while changing the attention route affects the trajectory only indirectly and may take effect too late.

Therefore, we propose ReState, a training-free method that directly corrects the masked clean prediction toward background evidence retrieved from unmasked source regions. The correction is strongest when noise remains and vanishes at the final forward, leaving the model free to synthesize final details. Object-excluding attention is applied throughout sampling as a complementary constraint; its additional benefit is concentrated near final synthesis. The retrieval features are shared with the first generation forward, so ReState preserves the native number of denoiser evaluations. Our contributions are:
\begin{itemize}
\setlength{\topsep}{-15pt}
\setlength{\itemsep}{-1pt}
\setlength{\parsep}{0pt}
\item \textbf{Diagnosis.} We use matched single-forward interventions to show that background-directed direct correction has a larger first-forward terminal removal effect than attention forcing across four distilled backbones. On the Turbo family, attention forcing at the first forward has no measurable effect on the final image.
\item \textbf{Method.} We propose ReState, which assigns a role to each forward of a few-step trajectory: the first collects background evidence from the input image, intermediate forwards with remaining noise write it into the masked clean prediction, and the final forward keeps object-excluding attention alone. It is the only training-free method that writes an explicit background target into the predicted state at the native few-step cost (Table~\ref{tab:concept}), and it adds no denoiser evaluations.
\item \textbf{Results.} Across two paired benchmarks and two main few-step backbones, ReState achieves the best masked PSNR, SSIM, and LPIPS among training-free methods, improving masked PSNR by $3.17$~dB over the strongest baseline. On OBER-Test, it outperforms DesignEdit and Attentive Eraser at their native step counts on all three masked metrics while running more than $12\times$ faster, and a single frozen configuration transfers to SDXL-Lightning and Hyper-SDXL without per-backbone tuning.
\end{itemize}
% ==================================================================
\section{Related Work}
\label{sec:related}

\begin{table}[t]
\centering
\caption{\textbf{ReState is the only training-free removal method that writes an explicit background target into the predicted state at the native few-step cost.} Cells classify each method as published (Appendix~\ref{app:auxpass}); they state design properties, not measured performance, which Table~\ref{tab:main} reports.}
\label{tab:concept}
\footnotesize
\setlength{\tabcolsep}{3pt}
\renewcommand{\arraystretch}{1.15}
\newcommand{\yes}{\textcolor{green!60!black}{\checkmark}}
\newcommand{\no}{\textcolor{red}{\ding{55}}}
\resizebox{\textwidth}{!}{%
\begin{tabular}{>{\raggedright\arraybackslash}m{4.4cm} cccccc c}
\toprule
Property & \shortstack{DesignEdit\\\venue{AAAI'25}} & \shortstack{Attentive Eraser\\\venue{AAAI'25}} & \shortstack{AdaEraser\\\venue{ICML'26}} & \shortstack{PANDORA\\\venue{ICME'26}} & \shortstack{PredErase\\\venue{arXiv'26}} & \shortstack{FuLLaMa\\\venue{WACV'26}} & \shortstack{\textbf{ReState}\\(Ours)} \\
\midrule
Writes an \textbf{explicit background target} into the predicted state & \no & \no & \no & \no & \yes & \no & \yes \\
Adds \textbf{no forwards or external models} beyond the backbone & \yes & \no & \no & \no & \no & \no & \yes \\
Operates at the \textbf{native few-step} schedule & \no & \no & \no & \no & \no & \no & \yes \\
\bottomrule
\end{tabular}%
}
\vspace{-8pt}
\end{table}

\noindent\textbf{Training-based and optimization-based removal.}
Training-based object removal methods learn removal-specific behavior through dedicated data or adaptation. PowerPaint learns task prompts for removal \citep{powerpaint2024}, CLIPAway steers an inpainting latent away from the target identity \citep{clipaway2024}, and SmartEraser and EntityErasure use removal-specific supervision \citep{smarteraser2025,entityerasure2025}. EraseLoRA adapts a low-rank adapter at test time \citep{eraselora2026}, while EraDiff calibrates the diffusion trajectory for erasure \citep{eradiff2025}. A separate line trains fast removal models that generate results in one to four forwards \citep{flashclear2026,osor2026}. These approaches can achieve high background fidelity, but they require removal supervision, test-time optimization, or dedicated training. ReState instead operates on a frozen pretrained model without removal-specific training.

\noindent\textbf{Training-free attention-based removal.}
Training-free methods keep the generator frozen and modify self-attention so that image-token attention suppresses keys in the target region and draws information from surrounding context. DesignEdit \citep{designedit2024}, Attentive Eraser \citep{attentiveeraser2025}, AdaEraser \citep{adaeraser2026}, and PANDORA \citep{pandora2026} instantiate this idea through different suppression rules, schedules, or guidance mechanisms; FuLLaMa additionally searches sampling parameters with derivative-free optimization \citep{fullama2026}; attention routing also extends to dense-scene removal~\citep{dors2026} and, combined with local in-mask perturbations, to training-free geometric editing~\citep{freefine2025}; none of these supplies an explicit background target for the masked prediction. PredErase~\citep{prederase2026} does steer the masked region toward a target, a representation predicted by an external I-JEPA model, through gradients back-propagated through the decoder of a multi-step FLUX.2 sampler. We use the hard object-excluding attention of this family as a complementary constraint and do not claim it as a contribution. The distinction of ReState is that it does not rely on attention rerouting alone: it directly corrects the masked clean prediction toward a target retrieved with the features of the same forward, without an external model or additional evaluations, when the trajectory is most sensitive to that correction.

\noindent\textbf{Direct state manipulation in diffusion sampling.}
Diffusion methods for inverse problems and inpainting often modify the sampling trajectory or clean prediction to enforce measurement consistency, for example through known-region replacement \citep{lugmayr2022repaint, avrahami2023blended}, null-space projection \citep{wang2022zeroshot}, diffusion posterior sampling \citep{chung2022diffusion}, or plug-and-play proximal correction of the clean prediction \citep{diffpir2023}. In removal, such consistency constrains only the known region outside the mask, which every compared method already enforces by source-latent anchoring; it specifies no target inside the mask. ReState likewise modifies a clean prediction, and its correction objective (Eq.~\ref{eq:prox}) is an instance of the quadratic proximal data-correction of this line, with three differences: the target is a pseudo-target $B$ retrieved from the unmasked source rather than a measured observation, no gradient of a measurement model is computed, and the correction runs on the native four-step schedule of a distilled backbone. Direct state manipulation itself is therefore not our novelty. Rather, matched intervention measurements identify direct correction toward retrieved background evidence as the effective early intervention for few-step removal and determine its noise-level-dependent schedule.

\noindent\textbf{Few-step distilled diffusion.}
Adversarial and consistency distillation enable image generation in one to four forwards \citep{sauer2023adversarial,lin2024sdxllightning,ren2024hypersd,luo2023latent}. Beyond generation quality and speed, recent work studies training-free text-driven editing in this regime: TurboEdit~\citep{turboedit2024} adapts inversion-based editing to SDXL-Turbo, InfEdit~\citep{infedit2024} edits with an inversion-free consistency sampler, and ChordEdit~\citep{chordedit2026} performs one-step editing and likewise observes that training-free editors break down with few steps. These methods edit an image toward a target prompt and do not address mask-based removal (Appendix~\ref{app:editing}). We instead study whether an external removal intervention applied at a single forward survives to the terminal image. This terminal-effect analysis reveals the distinct temporal roles of attention rerouting and direct prediction correction, and motivates their combination in ReState.

% ==================================================================
\section{Method}
\label{sec:method}

Figure~\ref{fig:method} overviews ReState. The method retrieves background evidence from the first generation forward, corrects the masked clean prediction during sampling, and retains object-excluding attention as a complementary constraint throughout the native schedule.

\begin{figure}[t]
\centering
\includegraphics[width=\linewidth]{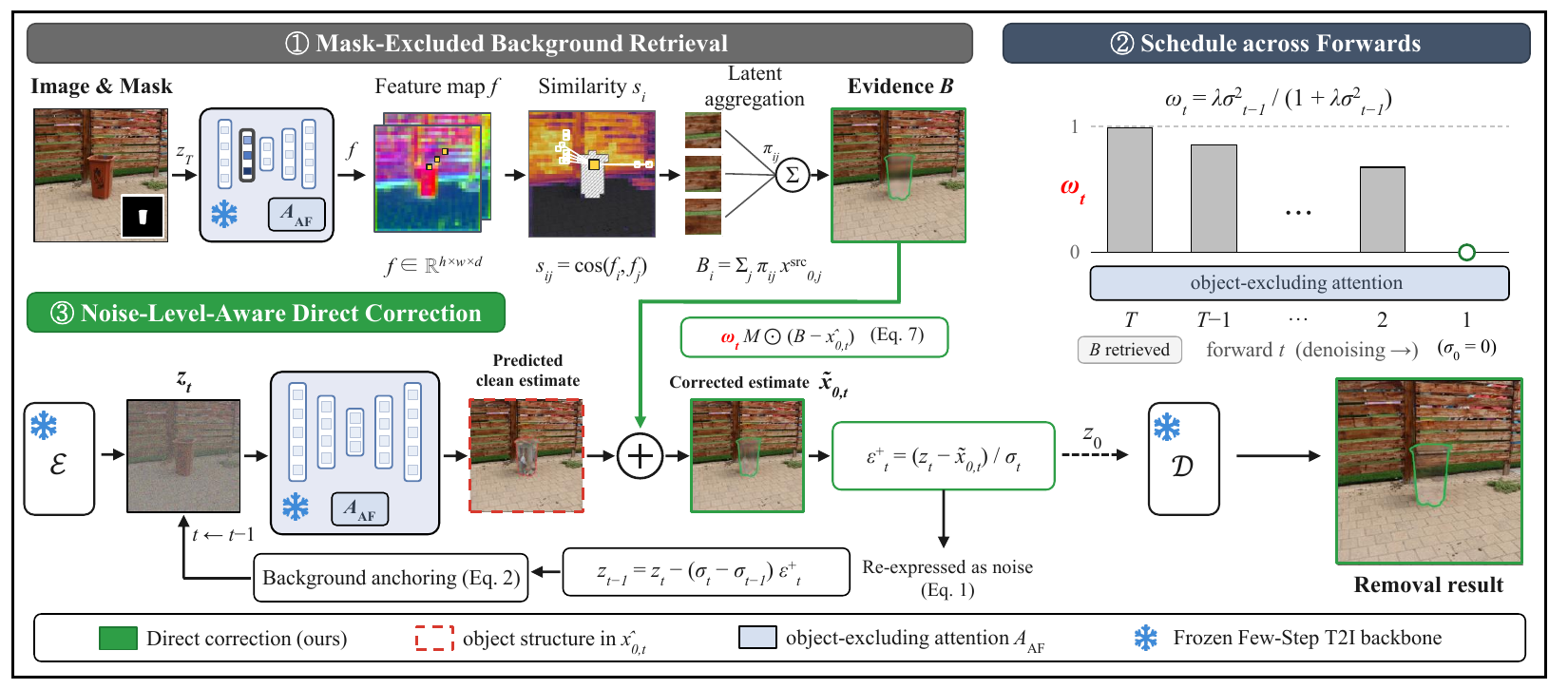}

\caption{\textbf{Overview of ReState.} Features from the first generation forward retrieve background evidence $B$ from unmasked source regions without an additional denoiser evaluation. At each step, ReState moves the masked clean prediction $\hat{x}_{0,t}$ toward $B$ with a noise-level-aware weight $\omega_t$. The weight vanishes at the final forward, while object-excluding attention remains active throughout sampling.}
\label{fig:method}
\vspace{-13pt}
\end{figure}

\subsection{Problem Setup}
\label{sec:setup}

Given a source image $\xsrc$, a binary removal mask $M$, a text condition $c$, and a frozen few-step model with native schedule $\sigma_T > \cdots > \sigma_0 = 0$ of $T \geq 2$ forwards, let $\Omega_m$ and $\Omega_b$ denote the masked and unmasked latent positions. We use $M$ also for the corresponding latent-grid mask. Sampling starts from the noised source latent $z_T=\mathcal{E}(\xsrc)+\sigma_T\epsilon$~\citep{rombach2022ldm}, with $\epsilon\sim\mathcal{N}(0,I)$. We use variance-exploding notation; the parameterization and scheduler conversions for all backbones are given in Appendix~\ref{app:parameterization}.

For the $\epsilon$-parameterized models used here, the denoiser induces the clean prediction
\begin{equation}
\hat{x}_{0,t}
=
z_t-\sigma_t\,\epsilon_\theta(z_t,\sigma_t,c;\mathcal{A}_t),
\label{eq:x0hat}
\end{equation}
where $\mathcal{A}_t$ denotes the attention used at step $t$. At fixed $z_t$, correcting $\hat{x}_{0,t}$ is therefore equivalent to correcting the noise prediction consumed by the native sampler. After each transition, we anchor the unmasked region to the source at the matching noise level:
\begin{equation}
z_{t-1}
\leftarrow
M\odot z_{t-1}
+
(1-M)\odot\left(\mathcal{E}(\xsrc)+\sigma_{t-1}\epsilon\right).
\label{eq:blend}
\end{equation}
Thus, all mechanisms below modify only the masked region.

\subsection{Existing Object-Excluding Attention}
\label{sec:af}

Existing training-free removal methods modify self-attention with a mask-dependent bias on target-region keys. We refer to this family of interventions as \emph{attention forcing}. Its hard and soft forms are
\begin{equation}
\Phi^{\mathrm{hard}}_{q,s}
=
\begin{cases}
-\infty, & s\in\Omega_m,\\
0, & \text{otherwise,}
\end{cases}
\qquad
\Phi^{\mathrm{soft}}_{q,s}
=
\log\eta_t(s),\quad \eta_t(s)\in(0,1],
\label{eq:phi}
\end{equation}
where $\eta_t(s)=1$ outside the mask and the hard form is the limit as $\eta_t(s)\to0$. ReState uses the hard form at every forward.

Attention forcing prevents self-attention from reading target-region keys, but it does not identify which background content a masked prediction should approach. ReState therefore uses it as a route constraint and supplies the missing target through direct correction. Appendix~\ref{app:af-details} specifies the attention layers, mask processing, and key-column implementation.

\subsection{Mask-Excluded Background Retrieval}
\label{sec:retrieval}

ReState retrieves an explicit background proposal from the source image. Retrieving correspondences from the unmasked source to fill masked content is a classical idea \citep{patchmatch2009,contextualattention2018}, and diffusion features are known to carry dense correspondence \citep{dift2023}. During the first generation forward, we capture a feature $f_i$ at each latent position and reuse the same forward's denoiser output for sampling. For each masked position $i$, we compare its feature with features at unmasked positions, retain the top fraction $r$, and aggregate their clean source latents:
\begin{equation}
\mathcal{N}_r(i)
=
\operatorname{TopK}_{j:\,M_j=0}\!\left(s_{ij}\right),
\qquad
s_{ij}
=
\frac{f_i^\top f_j}{\lVert f_i\rVert_2\lVert f_j\rVert_2},
\qquad
K=\lceil rN_{\mathrm{bg}}\rceil,
\end{equation}
\begin{equation}
\pi_{ij}
=
\frac{\exp(s_{ij}/\tau)}
{\sum_{k\in\mathcal{N}_r(i)}\exp(s_{ik}/\tau)},
\qquad
B_i
=
\sum_{j\in\mathcal{N}_r(i)}
\pi_{ij}\,x^{\mathrm{src}}_{0,j},
\qquad i\in\Omega_m .
\label{eq:retrieve}
\end{equation}
We use $r=0.01$. Because $B$ is built from clean source latents outside the mask, it is a background proposal in the same space as $\hat{x}_{0,t}$ rather than a ground-truth reconstruction. The feature source, temperature, and implementation details are given in Appendix~\ref{app:retrieval}.

\subsection{Noise-Level-Aware Direct Correction}
\label{sec:correction}

The correction is defined as the solution of a proximal objective. At step $t$, we seek a clean prediction that remains close to the model prediction while moving the mask interior toward $B$, with the fidelity term weighted by the noise remaining after the upcoming transition:
\begin{equation}
\tilde{x}_{0,t}
=
\arg\min_x\;
\frac{1}{2\sigma_{t-1}^{2}}
\lVert x-\hat{x}_{0,t}\rVert_2^2
+
\frac{\lambda}{2}
\lVert M\odot(x-B)\rVert_2^2 .
\label{eq:prox}
\end{equation}

The objective combines the model prediction, with precision $\sigma_{t-1}^{-2}$, and the evidence $B$, with a fixed precision $\lambda$; it separates over latent positions and admits a closed-form minimizer, so the correction needs no gradient and no inner iterations. We adopt this form as a design choice rather than a claim of optimality. Outside the mask, the minimizer is the original prediction; inside the mask, setting the coordinate-wise gradient to zero gives
\begin{equation}
\tilde{x}_{0,t}
=
\hat{x}_{0,t}
+
\omega_t M\odot(B-\hat{x}_{0,t}),
\qquad
\omega_t
=
\frac{\lambda\sigma_{t-1}^{2}}
{1+\lambda\sigma_{t-1}^{2}} .
\label{eq:update}
\end{equation}

The update has three immediate operational properties. First, it leaves the mask exterior unchanged. Second, it contracts the masked distance to the evidence by the factor $1-\omega_t$. Third, $0\leq\omega_t<1$, and the correction strength increases monotonically with the noise remaining after the transition. These properties characterize the update itself. Using $\sigma_{t-1}$ rather than $\sigma_t$ encodes the terminal boundary condition in the closed form. Early in the schedule, when the prediction can still be revised, the correction is strong. At the final transition, $\sigma_{t-1}=\sigma_0=0$ and hence $\omega_t=0$: the final forward is not directly corrected toward the averaged evidence $B$, but instead synthesizes final details under object-excluding attention alone. The weight profile itself is not the source of performance: a magnitude-matched constant schedule with the same terminal release is never significantly better and at most marginally worse (Appendix~\ref{app:abl-timing}). What the closed form contributes is that the terminal boundary and the intermediate weights follow from one objective with the single hyperparameter $\lambda$, rather than from a hand-designed temporal gate. Appendix~\ref{app:derivation} gives the complete derivation, and Appendix~\ref{app:abl-strength} and Appendix~\ref{app:abl-timing} verify the strength profile and final-forward boundary. Algorithm~\ref{alg:restate} summarizes the complete procedure. ReState retrieves $B$ once from the first generation forward, applies the correction of Eq.~\ref{eq:update} at every step, maps the corrected clean prediction back to a noise prediction, and advances the native sampler. Sharing the first-forward features adds no denoiser evaluation.

\begin{algorithm}[t]
\caption{ReState sampling}
\label{alg:restate}
\begin{algorithmic}[1]
\Require source $\xsrc$, mask $M$, condition $c$; frozen denoiser $\epsilon_\theta$ with object-excluding attention $\mathcal{A}^{\mathrm{AF}}$; native schedule $\sigma_T>\cdots>\sigma_0=0$; autoencoder $(\mathcal{E},\mathcal{D})$; parameters $\lambda$, $r$
\Ensure removal result $\hat{I}$
\State $z_T\gets\mathcal{E}(\xsrc)+\sigma_T\epsilon$, \quad $\epsilon\sim\mathcal{N}(0,I)$
\State $\hat{\epsilon}_T,\{f_i\}\gets\epsilon_\theta(z_T,\sigma_T,c;\mathcal{A}^{\mathrm{AF}})$ \Comment{shared first forward}
\State $B\gets\textsc{Retrieve}(\{f_i\},M)$ \Comment{Eq.~\ref{eq:retrieve}; no additional denoiser evaluation}
\For{$t=T,\ldots,1$}
\State \textbf{if} $t<T$ \textbf{then} $\hat{\epsilon}_t\gets\epsilon_\theta(z_t,\sigma_t,c;\mathcal{A}^{\mathrm{AF}})$
\State $\hat{x}_{0,t}\gets z_t-\sigma_t\hat{\epsilon}_t$ \Comment{Eq.~\ref{eq:x0hat}}
\State $\omega_t\gets\lambda\sigma_{t-1}^{2}/(1+\lambda\sigma_{t-1}^{2})$
\State $\tilde{x}_{0,t}\gets\hat{x}_{0,t}+\omega_t\,M\odot(B-\hat{x}_{0,t})$ \Comment{Eq.~\ref{eq:update}}
\State $z_{t-1}\gets z_t-(\sigma_t-\sigma_{t-1})\,(z_t-\tilde{x}_{0,t})/\sigma_t$ \Comment{native Euler step on the corrected prediction}
\State $z_{t-1}\gets M\odot z_{t-1}+(1-M)\odot\left(\mathcal{E}(\xsrc)+\sigma_{t-1}\epsilon\right)$ \Comment{Eq.~\ref{eq:blend}}
\EndFor
\State \Return $\hat{I}\gets\mathcal{D}(z_0)$
\end{algorithmic}
\end{algorithm}
% ==================================================================
\section{Experiments}
\label{sec:experiments}

\begin{table}[t]
\centering
\caption{\textbf{Object removal on OBER-Test and RORD: on the same frozen backbone and the same four-step sampling schedule, ReState opens a large gap on every masked metric.} Rows with $+$ apply the method on the backbone of their block; bold is best per column. Latency is the median per image and VRAM the peak allocation in GB, both on OBER-Test (Appendix~\ref{app:protocol}). Training-based comparisons are in Appendix~\ref{app:quant}. The DesignEdit port reduces to hard key masking and coincides with the attention-forcing-only row of Table~\ref{tab:ablation} (Appendix~\ref{app:ports}).}
\label{tab:main}
\small
\setlength{\tabcolsep}{4pt}
\begin{tabular}{l ccc ccc cc}
\toprule
& \multicolumn{3}{c}{OBER-Test} & \multicolumn{3}{c}{RORD} & & \\
\cmidrule(lr){2-4}\cmidrule(lr){5-7}
Method & PSNR$\uparrow$ & SSIM$\uparrow$ & LPIPS$\downarrow$ & PSNR$\uparrow$ & SSIM$\uparrow$ & LPIPS$\downarrow$ & Lat.\,(s)$\downarrow$ & VRAM$\downarrow$ \\
\midrule
SD-Turbo (4 steps) & 11.27 & 0.203 & 0.380 & 9.86 & 0.162 & 0.438 & 0.174 & 4.58 \\
$+$ Attentive Eraser\venue{AAAI'25} & 7.28 & 0.063 & 0.516 & 7.22 & 0.060 & 0.555 & 0.200 & 4.58 \\
$+$ DesignEdit\venue{AAAI'25} & 14.74 & 0.316 & 0.283 & 13.48 & 0.263 & 0.355 & 0.184 & 4.58 \\
$+$ AdaEraser\venue{ICML'26} & 12.27 & 0.237 & 0.347 & 10.48 & 0.177 & 0.420 & 0.843 & 5.08 \\
\rowcolor{ourshl}$+$ \textbf{ReState (Ours)} & \textbf{18.19} & \textbf{0.485} & \textbf{0.186} & \textbf{16.59} & \textbf{0.440} & \textbf{0.280} & 0.305 & 4.58 \\
\midrule
SDXL-Turbo (4 steps) & 9.86 & 0.133 & 0.435 & 7.55 & 0.074 & 0.516 & 0.247 & 7.1 \\
$+$ Attentive Eraser\venue{AAAI'25} & 7.09 & 0.046 & 0.514 & 7.35 & 0.065 & 0.538 & 0.338 & 7.1 \\
$+$ DesignEdit\venue{AAAI'25} & 15.50 & 0.348 & 0.269 & 14.34 & 0.326 & 0.334 & 0.281 & 7.1 \\
$+$ AdaEraser\venue{ICML'26} & 11.71 & 0.205 & 0.372 & 9.47 & 0.122 & 0.462 & 1.145 & 7.1 \\
\rowcolor{ourshl}$+$ \textbf{ReState (Ours)} & \textbf{18.75} & \textbf{0.498} & \textbf{0.194} & \textbf{17.21} & \textbf{0.451} & \textbf{0.273} & 0.373 & 7.1 \\
\bottomrule
\end{tabular}
\vspace{-13pt}
\end{table}

\noindent\textbf{Experimental setup.}
We evaluate on OBER-Test~\citep{zhao2026objectclear} (163 GT-paired images) and RORD~\citep{rord2022} (343 GT-paired images), both of which provide a source image, removal mask, and object-free background. Main removal results use SDXL-Turbo and SD-Turbo~\citep{sauer2023adversarial}; SDXL at 50 steps~\citep{podell2024sdxl} is the multi-step reference. We additionally evaluate the intervention analysis and full removal transfer on SDXL-Lightning~\citep{lin2024sdxllightning} and Hyper-SDXL~\citep{ren2024hypersd}. All training-free methods use the same forward-only removal pipeline: an empty prompt, guidance scale $0$, seed $0$, no inversion, and source-latent anchoring outside the mask at every step. We compare DesignEdit~\citep{designedit2024}, Attentive Eraser~\citep{attentiveeraser2025}, and AdaEraser~\citep{adaeraser2026} under the fully shared four-step pipeline; the same attention rules at 50 steps in the shared pipeline and the official implementations at their native step counts are reported in Appendix~\ref{app:multistep-baseline}, and training- and optimization-based comparisons in Appendix~\ref{app:quant}. We report masked PSNR, SSIM, and LPIPS~\citep{zhang2018lpips} against the paired background, together with median end-to-end latency and effective denoiser evaluations. Metric definitions and preprocessing are in Appendix~\ref{app:stats}. Appendix~\ref{app:evaluation} adds perceptual and distribution metrics, an object-only benchmark, and evaluation with object-effect masks and over the whole image; Appendix~\ref{app:qualitative} shows qualitative results without contour overlays, and Appendix~\ref{app:user-study} reports a user study and a VLM evaluation. All metric differences are paired at the image level and tested with a two-sided sign-flip permutation test using $10{,}000$ permutations; the four first-forward backbone comparisons additionally use Holm correction.

\noindent\textbf{Main removal results.}
Table~\ref{tab:main} reports the main comparison, and Figure~\ref{fig:qualitative} shows examples. ReState achieves the highest masked PSNR and SSIM and the lowest LPIPS in all four backbone--benchmark settings. Relative to DesignEdit, the strongest training-free baseline, it improves masked PSNR by $3.17$~dB on average. Attentive Eraser falls below the no-intervention baseline at four steps because its SARG term amplifies the noise prediction at every forward. This mechanism is calibrated for 50-step sampling: under the shared 50-step reference configuration, the same rule reaches $18.79$~dB on OBER-Test, versus $7.09$~dB for its four-step port (Appendix~\ref{app:multistep-baseline}). Re-tuning its guidance scale for four steps raises the port to $15.34$~dB, still $3.4$~dB below ReState (Appendix~\ref{app:ports}). With only four forwards, the same amplification is too large for a trajectory in which each step commits a substantial fraction of the final image. The identical ReState configuration ($\lambda=8$, $r=0.01$, probe layer, and schedule keying) also transfers to SDXL-Lightning and Hyper-SDXL without per-backbone tuning. On both benchmarks, it improves over DesignEdit by $1.46$ to $2.39$~dB in masked PSNR, with gains on every masked metric; all stated gaps are significant under the protocol of Appendix~\ref{app:stats}. Full results appear in Appendix~\ref{app:additional-backbones}. The ranking is also unchanged across three seeds, where ReState varies least (Appendix~\ref{app:seeds}). On MULAN, an object-only benchmark, ReState is also the best training-free method on every metric, $2.6$~dB above the strongest baseline in masked PSNR (Appendix~\ref{app:mulan}). On OBER-Test, it is also the best training-free method on DISTS, CMMD, and ReMOVE on both backbones (Appendix~\ref{app:metrics}).

% Figure 3: qualitative -----------------------------------------------------
\begin{figure}[t]
\centering
\includegraphics[width=\linewidth]{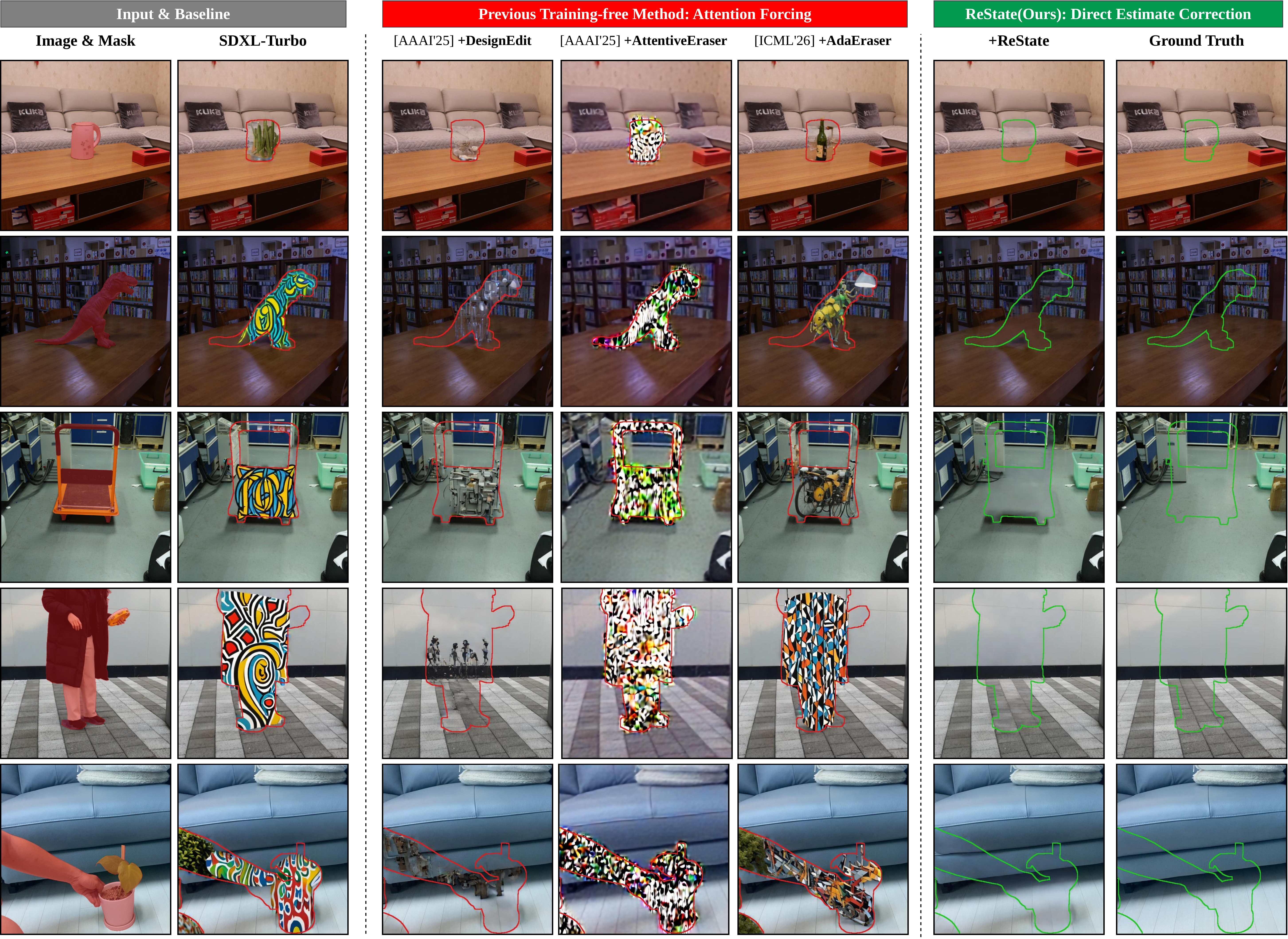}

\caption{\textbf{Qualitative results on GT-paired benchmarks.}
Previous training-free methods often leave object-shaped residuals in the mask, whereas ReState more closely reconstructs the paired background.}
\label{fig:qualitative}
\vspace{-10pt}
\end{figure}

\noindent\textbf{Intervention asymmetry and timing.}
We ask which intervention, applied at a single forward, still changes the final image. At a chosen forward $t$, two runs start from the same state $z_t$: a base run without intervention, and a run with exactly one intervention at that forward, after which sampling continues as usual. The first intervention is attention forcing. Applying object-excluding attention at that forward changes the masked clean prediction by
$\Delta^{\mathrm{AF}}_t=M\odot(\hat{x}^{\mathrm{AF}}_{0,t}-\hat{x}^{\mathrm{base}}_{0,t})$.
The second is direct correction, which moves the same clean prediction toward the retrieved evidence $B$ by the same norm. With $u_t=M\odot(B-\hat{x}^{\mathrm{base}}_{0,t})$ and a small $\varepsilon>0$,
\begin{equation}
\Delta^{\mathrm{direct}}_t
=
\frac{\lVert\Delta^{\mathrm{AF}}_t\rVert_2}
{\lVert u_t\rVert_2+\varepsilon}\,u_t .
\end{equation}
The two interventions thus have the same size and differ only in direction. To measure how much of the object remains in an output $I$, we use the residual $p(I,M)$, which is $1$ when the masked region matches the source and $0$ when it matches the ground-truth background (Appendix~\ref{app:intervention}). The terminal removal effect is the reduction of this residual in the final image,
\begin{equation}
E_k(t)
=
p\!\left(I^{\mathrm{base}}_0,M\right)
-
p\!\left(I^{k@t}_0,M\right),
\qquad
k\in\{\mathrm{AF},\mathrm{direct}\},
\label{eq:effect}
\end{equation}
so a positive $E$ means that the single intervention brings the final image closer to the background.

The direction of the perturbation matters, not its size. At the first forward of SDXL-Turbo, the correction toward $B$ gives $E=+0.0251$ (permutation $p=1.0\times10^{-4}$), whereas a random direction and a direction orthogonal to the background, at the same norm, give $-0.0004$ and $-0.0006$, neither of which is significant ($p=0.53$ and $p=0.36$).

Table~\ref{tab:mechanism} reports $E$ for both interventions at every forward of four distilled backbones; magnitudes are compared within a backbone. Three observations follow. First, at the first forward, direct correction has the larger effect on all four backbones (all Holm-corrected $p\leq4\times10^{-4}$). Second, on SDXL-Turbo and SD-Turbo, the first-forward effect of attention forcing is not significantly different from zero ($p=0.27$ and $p=0.52$); on SDXL-Lightning and Hyper-SDXL, it is positive but smaller than that of direct correction. Third, on SDXL-Turbo, the effect of attention forcing grows as the noise decreases, from $+0.002$ at the first forward to $+0.033$ at the last, whereas direct correction is effective from the first forward.

On the 50-step reference, attention forcing with a fixed number of steps likewise works better when placed late than at high noise (Appendix~\ref{app:controls}). The state should therefore be corrected strongly while noise remains, and the final forward should be left uncorrected (Appendix~\ref{app:abl-timing}); the weight of Eq.~\ref{eq:update} provides both.

\begin{table}[t]
\centering
\caption{\textbf{Direct correction acts from the first forward.}
Terminal removal effect $E$ of a single matched intervention across noise levels and distilled backbones. Direct correction ($\Edir$, bold) exceeds attention forcing ($\EAF$) at every noise level and backbone.}
\label{tab:mechanism}
\small
\setlength{\tabcolsep}{5pt}
\begin{tabular}{@{}l cc cc cc cc@{}}
\toprule
& \multicolumn{2}{c}{SDXL-Turbo} & \multicolumn{2}{c}{SD-Turbo} & \multicolumn{2}{c}{SDXL-Lightning} & \multicolumn{2}{c}{Hyper-SDXL} \\
\cmidrule(lr){2-3}\cmidrule(lr){4-5}\cmidrule(lr){6-7}\cmidrule(lr){8-9}
$\sigma$ & $\EAF$ & $\Edir$ & $\EAF$ & $\Edir$ & $\EAF$ & $\Edir$ & $\EAF$ & $\Edir$ \\
\midrule
14.6 & $+.002$ & $\mathbf{+.025}$ & $-.001$ & $\mathbf{+.029}$ & $+.034$ & $\mathbf{+.072}$ & $+.030$ & $\mathbf{+.052}$ \\
4.1  & $+.016$ & $\mathbf{+.039}$ & $+.018$ & $\mathbf{+.045}$ & $+.052$ & $\mathbf{+.059}$ & $+.049$ & $\mathbf{+.062}$ \\
1.6  & $+.028$ & $\mathbf{+.040}$ & $+.026$ & $\mathbf{+.036}$ & $+.029$ & $\mathbf{+.041}$ & $+.027$ & $\mathbf{+.042}$ \\
0.7  & $+.033$ & $\mathbf{+.049}$ & $+.021$ & $\mathbf{+.040}$ & $+.023$ & $\mathbf{+.046}$ & $+.020$ & $\mathbf{+.046}$ \\
\bottomrule
\end{tabular}
\end{table}

\begin{table}[t]
\centering
\caption{\textbf{Component and timing ablation.}
Direct correction supplies most of the gain; attention forcing contributes mainly at the last forward.}
\label{tab:ablation}
\small
\setlength{\tabcolsep}{5pt}
\newcommand{\yes}{\textcolor{green!60!black}{\checkmark}}
\newcommand{\no}{\textcolor{red}{\ding{55}}}
\begin{tabular}{ccc ccc}
\toprule
\shortstack{Direct\\correction} & \shortstack{AF\\(all forwards)} & \shortstack{AF\\(last only)} & PSNR$\uparrow$ & SSIM$\uparrow$ & LPIPS$\downarrow$ \\
\midrule
\no & \no & \no & 9.86 & 0.133 & 0.435 \\
\no & \no & \yes & 12.15 & 0.182 & 0.370 \\
\no & \yes & \no & 15.50 & 0.348 & 0.269 \\
\yes & \no & \no & 17.62 & 0.477 & 0.237 \\
\yes & \no & \yes & 18.65 & 0.495 & 0.196 \\
\rowcolor{ourshl}\yes & \yes & \no & \textbf{18.75} & \textbf{0.498} & \textbf{0.194} \\
\bottomrule
\end{tabular}
\vspace{-5pt}
\end{table}

\noindent\textbf{Component and timing ablation.}
Table~\ref{tab:ablation} separates the two components on SDXL-Turbo. Direct correction alone reaches $17.62$~dB, compared with $15.50$~dB for full-schedule attention forcing alone. Adding attention forcing only at the final forward raises the result to $18.65$~dB, and applying it throughout reaches $18.75$~dB. Conversely, last-forward attention forcing without direct correction reaches only $12.15$~dB, below full-schedule attention forcing. Attention forcing therefore contributes on top of a state that direct correction has already prepared; its additional benefit is concentrated late, although the full schedule retains a small but significant gain (Appendix~\ref{app:abl-timing}). The gain is also not explained by the retrieved proposal alone: with the same $B$ and the same four forwards, the per-forward correction outperforms writing $B$ into the latent once (Appendix~\ref{app:abl-prefill}). Inserting $B$ at a later forward comes closer and is reported there as a faster variant. Decoding $B$ directly into the mask, without denoising, gives $16.55$~dB for the retrieved $B$ and $15.13$~dB for a global background mean, against $18.75$ and $18.16$~dB when the same evidence is used in the per-forward correction (Table~\ref{tab:evidence}).

\noindent\textbf{Efficiency and multi-step comparison.}
Table~\ref{tab:runtime} reports hook-measured effective denoiser evaluations and end-to-end latency. ReState remains at the native four evaluations because the retrieval probe shares the first generation forward; feature matching and the masked update are its only additional operations. Against DesignEdit and Attentive Eraser on SDXL, run with their official implementations at their native step counts, four-step ReState is better on all three masked metrics at $14.2\times$ and $12.4\times$ lower latency. Appendix~\ref{app:multistep-baseline} additionally runs the attention rule of every compared method at 50 steps in the shared pipeline: four-step ReState is better than all of them in SSIM and LPIPS, its PSNR is within $0.04$~dB of the best, and it is $6.7\times$ faster than 50-step attention forcing with hard key masking.

\begin{table}[t]
\centering
\caption{\textbf{Cost and comparison against the multi-step reference.} Latency is median per image; forwards counts the effective model passes. The SDXL rows of DesignEdit (50 inversion and 50 sampling steps) and Attentive Eraser (40 steps) run their official implementations at their native step counts (Appendix~\ref{app:multistep-baseline}).}
\label{tab:runtime}
\small
\begin{tabular}{lccccc}
\toprule
Method & Forwards & Latency (s)$\downarrow$ & PSNR$\uparrow$ & SSIM$\uparrow$ & LPIPS$\downarrow$ \\
\midrule
SDXL-Turbo & 4 & 0.247 & 9.86 & 0.133 & 0.435 \\
$+$ Attentive Eraser\venue{AAAI'25} & 8 & 0.338 & 7.09 & 0.046 & 0.514 \\
$+$ Attentive Eraser, re-tuned & 8 & 0.338 & 15.34 & 0.335 & 0.277 \\
$+$ DesignEdit\venue{AAAI'25} & 4 & 0.281 & 15.50 & 0.348 & 0.269 \\
$+$ AdaEraser\venue{ICML'26} & 8 & 1.145 & 11.71 & 0.205 & 0.372 \\
\rowcolor{ourshl}$+$ \textbf{ReState (Ours)} & 4 & 0.373 & \textbf{18.75} & \textbf{0.498} & \textbf{0.194} \\
\midrule
SDXL $+$ DesignEdit\venue{AAAI'25} & 250 & 5.30 ($14.2\times$) & 18.51 & 0.483 & 0.206 \\
SDXL $+$ Attentive Eraser\venue{AAAI'25} & 80 & 4.63 ($12.4\times$) & 18.19 & 0.446 & 0.245 \\
\rowcolor{ourshl} \textbf{SDXL $+$ ReState (Ours; App.~\ref{app:multistep-baseline})} & 50 & 2.84 ($7.6\times$) & 18.83 & 0.485 & 0.181 \\
\bottomrule
\end{tabular}
\end{table}

\begin{table}[t]
\centering
\caption{\textbf{The gain comes from how the evidence is used, and evidence quality is the remaining bottleneck.} Each evidence $B$ is decoded directly into the mask without denoising, or supplied to the identical per-forward correction (OBER-Test, SDXL-Turbo).}
\label{tab:evidence}
\small
\begin{tabular}{llccc}
\toprule
Evidence $B$ & Use of $B$ & PSNR$\uparrow$ & SSIM$\uparrow$ & LPIPS$\downarrow$ \\
\midrule
Global background mean & Decoded directly & 15.13 & 0.466 & 0.337 \\
Retrieved & Decoded directly & 16.55 & 0.448 & 0.299 \\
\midrule
Global background mean & Per-forward correction & 18.16 & 0.481 & 0.214 \\
\rowcolor{ourshl}Retrieved (default) & Per-forward correction & 18.75 & 0.498 & 0.194 \\
Ground-truth background (oracle) & Per-forward correction & 22.32 & 0.588 & 0.133 \\
\bottomrule
\end{tabular}
\vspace{-3pt}
\end{table}

% ==================================================================
\section{Discussion}
\label{sec:discussion}

\textbf{Why ReState combines direct correction with attention forcing.}
The two components are effective at different times. Direct correction changes the final image from the first forward, whereas on the Turbo family attention forcing at the first forward has no measurable effect (Table~\ref{tab:mechanism}). A few-step trajectory also keeps what its first forward produces: when the clean prediction is changed slightly at the first forward, the change in the final output is $1.38$ times as large on four-step SDXL-Turbo, but only $0.30$ times on the 50-step reference (Appendix~\ref{app:retention}). The background therefore has to be supplied while noise remains, and attention forcing cannot supply it, because it only prevents self-attention from reading target-region keys. Its benefit comes late: of the $1.13$~dB that it adds to direct correction alone, $1.03$~dB is obtained by applying it at the final forward only (Table~\ref{tab:ablation}). ReState assigns the forwards accordingly: those with remaining noise write the background evidence into the masked clean prediction, and the final forward synthesizes details under object-excluding attention alone. The first-forward comparison identifies which intervention reaches the final image; it does not mean that the first forward itself must be corrected. On the four-step schedule, correcting only the second and third forwards gives $18.74$~dB against $18.75$~dB for the default (Appendix~\ref{app:abl-timing}), whereas with two steps, where only the first forward is corrected, ReState reaches $16.38$~dB against $13.12$~dB for attention forcing alone (Appendix~\ref{app:abl-steps}). Direct correction also does not depend on the attention rule: it raises the re-tuned attention rule of Attentive Eraser by $3.2$~dB and that of AdaEraser by $6.1$~dB (Appendix~\ref{app:abl-combination}).

\textbf{ReState is designed for the native few-step schedule.}
The baselines that degrade at four steps do so because of the few-step setting, not because of how they were ported: in the same pipeline at 50 steps, the SARG of Attentive Eraser reaches $18.79$~dB, against $7.09$~dB at four steps (Appendix~\ref{app:multistep-baseline}). Conversely, ReState is not meant to be moved to 50 steps unchanged, where it gives $17.57$~dB. Because the weight in Eq.~\ref{eq:update} depends only on the noise level, on a 50-step schedule the correction remains strong across many low-noise forwards where fine details are synthesized, repeatedly writing the averaged evidence $B$ into the predictions. Applying the correction only at forwards with $\sigma_{t-1}\geq1.61$ restores $18.83$~dB (Table~\ref{tab:runtime}). On the four-step schedule, no such restriction is needed: only the final forward is left for this synthesis, and its weight is already zero.

\textbf{Scope and limitations.}
The absence of a first-forward effect of attention forcing is shown on the Turbo family only; the larger effect of direct correction holds on the four distilled backbones; and the removal results cover two paired benchmarks and two main backbones, with SDXL-Lightning and Hyper-SDXL in Appendix~\ref{app:additional-backbones}. The difference between the interventions appears with distilled weights on a short schedule: the original SDXL weights on the same four-step schedule do not show it, and running SDXL-Turbo with 8 or 16 forwards reduces the first-forward effect of direct correction (Appendix~\ref{app:controls}). Both settings are outside the native use of the model and serve as controls. ReState requires at least two forwards: with one step, the only forward is the final one, whose weight is zero, so only attention forcing remains (Appendix~\ref{app:abl-steps}). As a first extension, on five DiT backbones trained with flow matching, ReState raises masked PSNR on OBER-Test from $10.0$--$11.3$~dB without intervention to $17.3$--$19.6$~dB (Appendix~\ref{app:dit}).

The main limitation is the quality of the background evidence. With the ground-truth background as $B$, masked PSNR rises from $18.75$ to $22.32$~dB (Table~\ref{tab:evidence}), so the evidence, not the correction, limits the result. Recomputing the retrieval at every forward already gives $19.02$~dB (Appendix~\ref{app:abl-probe}), so improving the retrieval appears more useful than further changes to attention. The second limitation is texture: because $B$ averages background latents, the filled region is smoother than the real background (Appendix~\ref{app:frequency}; Figure~\ref{fig:appendix-limit}). Since masked PSNR and SSIM can favor such smooth fills, we add perceptual metrics (Appendix~\ref{app:metrics}) and a user study with a VLM evaluation (Appendix~\ref{app:user-study}).

% ==================================================================
\section{Conclusion}
\label{sec:conclusion}

Attention forcing on few-step distilled diffusion models often leaves object-shaped residuals. We attribute this to a temporal mismatch: a structured prediction forms inside the mask at the first forward and the short trajectory offers few chances to revise it, whereas attention forcing acts mainly at later forwards. Matched single-forward interventions support this: at the first forward, direct correction toward background evidence has a larger terminal removal effect than attention forcing on four distilled backbones. ReState therefore corrects the masked clean prediction toward retrieved background evidence with a noise-level-aware weight that vanishes at the final forward, and keeps object-excluding attention at every forward. Across two paired benchmarks and two few-step backbones, it achieves the best masked metrics among training-free methods without additional denoiser evaluations. Extending the correction to one-step models and the intervention analysis to DiT and video backbones is left for future work.

\section*{Acknowledgments}

This work was partly supported by the KHIDI grant funded by the Korean government (MOHW) [No.RS-2025-02307233, No.RS-2026-25613012], the NRF or IITP grants funded by the Korean government (MSIT) [No.05-26-04-0094, No.RS-2026-25472075, No.RS-2025-02305581, No.RS-2025-25442338, and No.RS-2021-II211343], the ITIP grant funded by the Korean government (MOTIR) [No.RS-2026-25549946], the Research grant from SNU, and the Strategic Hub grant for International Research Collaboration of SNU. 
Kyungsu Kim is affiliated with the School of Transdisciplinary Innovations, Department of Biomedical Science, Interdisciplinary Program in AI (IPAI) \& Graduate School of AI, Medical Research Center, and AI Institute at SNU.

\bibliography{main}
\bibliographystyle{plainnat}

% ==================================================================
\clearpage
\appendix
\captionsetup[table]{skip=6pt}
\raggedbottom

% Sections: alphabetic
\renewcommand{\thesection}{\Alph{section}}

% Tables and figures: restart numbering per section with the section letter
\makeatletter
\@addtoreset{table}{section}
\@addtoreset{figure}{section}
\makeatother
\renewcommand{\thetable}{\thesection\arabic{table}}
\renewcommand{\thefigure}{\thesection\arabic{figure}}
\def\theHtable{\thesection.\arabic{table}}
\def\theHfigure{\thesection.\arabic{figure}}
\setcounter{table}{0}
\setcounter{figure}{0}

% Equations: restart numbering per section with the section letter
\makeatletter
\@addtoreset{equation}{section}
\makeatother
\renewcommand{\theequation}{\thesection\arabic{equation}}
\def\theHequation{\thesection.\arabic{equation}}
\setcounter{equation}{0}

\phantomsection
{\huge\centering\bfseries Appendix\par}
\addcontentsline{toc}{section}{Appendix}
\vspace{0.8em}
{\large\centering\bfseries Contents\par}
\vspace{0.4em}

\startcontents[appendix]
\begingroup
\setlength{\parskip}{0pt}
\titlecontents{section}[1.7em]
  {\addvspace{0.4em}\bfseries}
  {\contentslabel{1.7em}}{}
  {\titlerule*[0.7pc]{.}\contentspage}
\titlecontents{subsection}[4.1em]
  {\addvspace{0.1em}}
  {\contentslabel{2.4em}}{}
  {\titlerule*[0.7pc]{.}\contentspage}
\titlecontents{subsubsection}[7.2em]
  {\addvspace{0em}}
  {\contentslabel{3.1em}}{}
  {\titlerule*[0.7pc]{.}\contentspage}
\printcontents[appendix]{}{1}{\setcounter{tocdepth}{3}}
\endgroup
\clearpage

% ============================================================
\section{Method Details}
\label{app:method}

\subsection{Exact Diffusion Parameterization}
\label{app:parameterization}

We state the clean-prediction and sampler relations of Section~\ref{sec:setup} in a backbone-independent form. A clean latent is corrupted as
\begin{equation}
z_t = x_0 + \sigma_t\,\epsilon \ \ \text{(VE)},
\qquad
z_t = \sqrt{\bar\alpha_t}\,x_0 + \sqrt{1-\bar\alpha_t}\,\epsilon \ \ \text{(VP)},
\label{eq:app-corruption}
\end{equation}
and the VP form matches the VE form under $\sigma_t = \sqrt{(1-\bar\alpha_t)/\bar\alpha_t}$ after rescaling. The network output enters only through its parameterization, and the clean prediction is
\begin{equation}
\hat{x}_{0,t} =
\begin{cases}
z_t - \sigma_t\,\epsilon_\theta(z_t) & \text{($\epsilon$-prediction)},\\[2pt]
\sqrt{\bar\alpha_t}\,z_t - \sqrt{1-\bar\alpha_t}\,v_\theta(z_t) & \text{($v$-prediction, VP coordinates)},\\[2pt]
x_\theta(z_t) & \text{(sample-prediction)}.
\end{cases}
\label{eq:app-x0general}
\end{equation}
In every case the correction of Eq.~\ref{eq:update} is applied to $\hat{x}_{0,t}$ in this clean space, and the corrected prediction $\tilde{x}_{0,t}$ is mapped back to the native output type by inverting the same relation, e.g.
\begin{equation}
\epsilon^{+}_t = \frac{z_t - \tilde{x}_{0,t}}{\sigma_t}
\label{eq:app-invert}
\end{equation}
for an $\epsilon$-parameterization, so the scheduler transition to $z_{t-1}$ runs unmodified. The method is therefore defined at the level of the clean prediction and is agnostic to the output parameterization and to the deterministic or stochastic form of the sampler. The same holds for flow-matching models, whose corruption is $z_t = (1-\sigma_t)\,x_0 + \sigma_t\,\epsilon$ and whose network predicts the velocity $v_\theta$: the clean prediction is $\hat{x}_{0,t} = z_t - \sigma_t\,v_\theta(z_t)$, the weight of Eq.~\ref{eq:update} is evaluated at the equivalent noise level $s_t = \sigma_t/(1-\sigma_t)$, and the corrected prediction is mapped back by $v^{+}_t = (z_t - \tilde{x}_{0,t})/\sigma_t$. The DiT backbones of Appendix~\ref{app:dit} use this form.

\textbf{Per-backbone instantiation.} All four distilled backbones, SDXL-Turbo, SD-Turbo, SDXL-Lightning, and Hyper-SDXL,\footnote{Checkpoints: \texttt{stabilityai/sdxl-turbo}, \texttt{stabilityai/sd-turbo}, \texttt{ByteDance/SDXL-Lightning} (4-step UNet), and \texttt{ByteDance/Hyper-SD} (4-step SDXL LoRA).} are $\epsilon$-parameterized and run the identical native four-step schedule of the Euler discrete scheduler with trailing timesteps $(999, 749, 499, 249)$, giving $\sigma = 14.617,\ 4.082,\ 1.613,\ 0.693$ with $\sigma_0 = 0$; these are the four noise levels of the intervention grid in Table~\ref{tab:mechanism}. SDXL-Turbo and SD-Turbo operate at $512^2$ and SDXL-Lightning and Hyper-SDXL at $1024^2$. The scheduler exposes VE-form noise levels directly, so no explicit VP-to-VE conversion is performed at sampling time. At the final transition $\sigma_0 = 0$ no special-casing is required: the correction weight vanishes by the structure of the objective (Appendix~\ref{app:derivation}), and the sampler returns the clean prediction itself.

\subsection{Object-Excluding Attention Details}
\label{app:af-details}

With projected queries, keys, and values $Q,K,V$ and key dimension $d$, an attention-forcing layer adds the bias $\Phi$ of Eq.~\ref{eq:phi} to the logits and reads
\begin{equation}
\tilde{A} = \mathrm{softmax}\!\left(QK^{\top}/\sqrt{d} + \Phi\right), \qquad h_q = \sum_s \tilde{a}_{q,s}\, v_s ,
\label{eq:af}
\end{equation}
where $\tilde{a}_{q,s}$ are the entries of $\tilde{A}$; the members of Section~\ref{sec:af} differ only in $\Phi$.

\textbf{Attention convention.} As implemented, the bias $\Phi$ acts on the key axis: the columns of object image keys, $s \in \Omega_m$ at the layer's token grid, receive $\phi_{\cdot,s} = \log \eta_t(s)$, which is $-\infty$ in the hard form, and background key columns receive zero. The same key-column bias is shared by all query rows, so background queries are likewise prevented from referencing object keys, and it is broadcast across all heads of a layer. Cross-attention layers are never modified. 

The intervention covers every self-attention layer operating on the image-token grid; at $512^2$ this is 70 layers of the SDXL UNet, 10 at the $32^2$ grid and 60 at the $16^2$ grid, and the $1024^2$ backbones are covered at their correspondingly scaled grids. The latent-resolution mask is downsampled to each token grid by max-pooling, so a cell counts as object if any of its pixels is masked. ReState applies the hard form at all four forwards; the soft form appears only in the baselines of Section~\ref{sec:af}. Algorithm~\ref{alg:af} summarizes the per-layer computation.

\begin{algorithm}[H]
\caption{Object-excluding self-attention at one layer}
\label{alg:af}
\begin{algorithmic}[1]
\Require projected $Q, K, V$ on the image-token grid; token-grid mask $m$; suppression $\eta_t$
\State $b_s \gets \log \eta_t(s)$ for keys $s$ with $m_s = 1$, else $b_s \gets 0$ \Comment{$-\infty$ in the hard form}
\State \Return $\mathrm{softmax}\!\left(QK^{\top}/\sqrt{d} + \mathbf{1}\,b^{\top}\right) V$ \Comment{key-column bias, shared by all query rows and heads}
\end{algorithmic}
\end{algorithm}

\subsection{Background Retrieval Details}
\label{app:retrieval}

\textbf{Feature source.} Features $f_i$ are taken from the output of the first up-block of the UNet during the probe forward. The probe is the first generation forward itself: its features are captured by a hook and its noise prediction is reused at the first sampling step, so retrieval adds no auxiliary pass and the pipeline stays at four denoiser evaluations. The native resolution of this map at $512^2$ is $32 \times 32$; it is bilinearly resampled to the $64 \times 64$ grid of the latent, so each latent position receives one feature vector for matching.

\textbf{Matching.} For each masked position $i \in \Omega_m$, cosine similarities between the $\ell_2$-normalized feature $f_i$ and all background features $\{f_j : j \in \Omega_b\}$ are computed. The top $r = 1\%$ of background positions are retained, $k = \lceil 0.01\,\lvert\Omega_b\rvert \rceil$, which is between 26 and 41 positions with a median of 38 on OBER-Test, and their weights are the softmax of the retained similarities at temperature $\tau = 0.1$.

\textbf{Evidence space.} The evidence is $B_i = \sum_{k} w_k\,\mathcal{E}(\xsrc)_{s_k}$, the weighted average of the clean source latents at the retrieved positions. $B$ therefore lives in the same clean latent space as $\hat{x}_{0,t}$, and Eq.~\ref{eq:update} mixes commensurable quantities.

\subsection{Derivation of the Closed-Form Correction}
\label{app:derivation}

\textbf{Setting.} Fix a step $t$ with $\sigma_{t-1} > 0$ and drop the step index. The objective of Eq.~\ref{eq:prox} is
\begin{equation}
J(x) \;=\; \frac{1}{2\sigma_{t-1}^2}\,\lVert x - \hat{x}_0 \rVert_2^2 \;+\; \frac{\lambda}{2}\,\lVert M \odot (x - B) \rVert_2^2 .
\end{equation}

\begin{proposition}[Closed-form correction]
\label{prop:correction}
For $\sigma_{t-1} > 0$, $J$ is strictly convex and has a unique minimizer $x^*$, given coordinate-wise by $x_i^* = \hat{x}_{0,i}$ for $i \notin \Omega_m$ and
\begin{equation}
x_i^* \;=\; \hat{x}_{0,i} + \omega_t \left(B_i - \hat{x}_{0,i}\right),
\qquad
\omega_t \;=\; \frac{\lambda \sigma_{t-1}^2}{1 + \lambda \sigma_{t-1}^2},
\label{eq:app-closedform}
\end{equation}
for $i \in \Omega_m$, which is Eq.~\ref{eq:update}. 
Moreover:
\begin{enumerate}[label=(\roman*), leftmargin=2.2em, topsep=3pt, itemsep=2pt]
\item \emph{Mask-exterior invariance:} $(1 - M) \odot x^* = (1 - M) \odot \hat{x}_0$.
\item \emph{Contraction:} $x^*_i - B_i = (1 - \omega_t)\,(\hat{x}_{0,i} - B_i)$ for $i \in \Omega_m$,
      with the fixed factor $1 - \omega_t = 1/(1 + \lambda\sigma_{t-1}^2)$.
\item \emph{Bounds and monotonicity:} $0 \leq \omega_t < 1$; $\omega_t$ is increasing in
      $\sigma_{t-1}$, equals $0$ at $\sigma_{t-1} = 0$, and approaches $1$ only as
      $\sigma_{t-1} \to \infty$.
\end{enumerate}
\end{proposition}

\begin{proof}
$J$ is a sum of two convex quadratics whose Hessian is diagonal with entries $1/\sigma_{t-1}^2 + \lambda M_i > 0$, so $J$ is strictly convex and the unique minimizer is characterized by $\nabla J(x^*) = 0$. The objective separates over coordinates. For $i \notin \Omega_m$ the second term vanishes and $x_i^* = \hat{x}_{0,i}$, which gives (i). For $i \in \Omega_m$, stationarity reads
\begin{equation}
\frac{x_i^* - \hat{x}_{0,i}}{\sigma_{t-1}^2} + \lambda\left(x_i^* - B_i\right) = 0
\quad\Longrightarrow\quad
x_i^* = \frac{\hat{x}_{0,i} + \lambda \sigma_{t-1}^2 B_i}{1 + \lambda \sigma_{t-1}^2},
\end{equation}
which rearranges to Eq.~\ref{eq:app-closedform}. Subtracting $B_i$ from both sides gives (ii), and (iii) follows from $\omega_t = 1 - 1/(1+\lambda\sigma_{t-1}^2)$.
\end{proof}

\begin{remark}[Final transition]
At the last step $\sigma_{t-1} = \sigma_0 = 0$ the fidelity weight $1/\sigma_{t-1}^2$ diverges and the minimizer degenerates to $x^* = \hat{x}_0$; equivalently $\omega_t \to 0$ in the closed form. The final forward therefore receives no correction by the structure of the objective rather than by an added rule, and synthesis at that step is constrained only by the object-excluding attention of Eq.~\ref{eq:af}. 

With the schedule of SDXL-Turbo and $\lambda = 8$ the realized weights are $\omega = 0.993,\ 0.954,\ 0.794,\ 0$ over the four forwards. Sensitivity of removal quality to $\lambda$, together with the realized $\omega$ profile at each value, is reported with the method ablations in Table~\ref{tab:appendix-lsweep}. The necessity of the boundary itself is verified in Table~\ref{tab:appendix-timing}: keying $\omega$ to $\sigma_t$ leaves the final forward a correction and collapses perceptual quality.
\end{remark}

% ============================================================
\section{Experimental Details}
\label{app:protocol}

\textbf{Datasets.} OBER-Test is the GT-paired removal benchmark released with ObjectClear~\citep{zhao2026objectclear}: each sample pairs a source image containing the target object with a ground-truth background image in which the object is absent, together with an object mask and an object-effect mask; we use the object mask as the removal mask $M$ for all images. We use the released set as is, without filtering or re-annotation; the 163 images constitute the full benchmark. 

RORD~\citep{rord2022} is released as $1920 \times 1080$ video sequences: each sequence contains frames with the removable objects present, a single object-free ground-truth frame of the same scene, and per-frame polygon annotations of those objects. From each of the 343 sequences we select one frame containing a single removal target, filtering candidates with the released annotations and confirming each choice by manual review, and pair it with the sequence's ground-truth frame; the removal mask is the binary union of the selected frame's object polygons.

Each triplet is then converted to $512^2$ by a mask-centered square crop rather than a full-frame downsample: the crop side is the smallest square that contains the mask with a $10\%$ margin per side, floored at $512$ pixels, the identical window is applied to the image, the mask, and the ground truth, and the crop is area-resized to $512^2$. Small objects are therefore evaluated at native resolution, and large objects incur only the minimal downsampling.

\textbf{Additional benchmarks.} Two further sets are used only in the appendix. \emph{RORD-Val.} RORD annotates each object with a coarse mask that also covers its shadow and reflection~\citep{rord2022}. The authors of OBER~\citep{zhao2026objectclear} selected 343 validation frames of RORD and additionally released a refined object mask for each; we use this set, with both masks, only for the object-effect mask experiments of Appendix~\ref{app:effect}. It covers the same 216 scenes as our RORD frames but mostly different frames (15 in common), and it is converted to $512^2$ by the same mask-centered crop, here centered on the object-effect mask. \emph{MULAN.} MULAN~\citep{mulan2024} decomposes an image into a background layer and object layers, so its ground truth removes only the object layer. We use the MULAN-COCO removal pairs released by the authors of RORem~\citep{rorem2025} and draw 500 pairs uniformly at random with a fixed seed, without filtering (Appendix~\ref{app:mulan}).

\textbf{Baselines.} All three training-free baselines run on the identical forward-only removal pipeline as ReState, so the comparison isolates the intervention: the prompt is the empty string, few-step sampling is guidance-free (scale $0.0$), the seed is fixed to $0$, and no inversion is used. The official DesignEdit aligns the source through DDIM inversion~\citep{song2021ddim}, while AdaEraser constructs its reference from the noised source latent at the matching noise level; our few-step port re-injects the clean source latent outside the mask at every step, which anchors the background in latent space while only the masked region is synthesized. 

DesignEdit is instantiated as the hard key-masking of Section~\ref{sec:af} over all self-attention layers and forwards (Appendix~\ref{app:af-details}). Attentive Eraser keeps its official SARG scale and suppression scale and applies its attention redirection at all self-attention layers, whereas its official implementation restricts it to the up-block layers (Appendix~\ref{app:ports}); SARG replaces classifier-free guidance~\citep{ho2021classifier}, so its guidance is disabled by construction. AdaEraser runs its adaptive suppression with a reference forward alongside each target forward; AdaEraser has no public implementation and is implemented from its paper, which describes the two as one batched call, whereas our port runs them sequentially, which is mathematically identical since batch elements do not interact in the denoiser, and is counted as two forwards in Appendix~\ref{app:auxpass}. Training-based comparisons use released checkpoints with their official default inference settings; the exact checkpoints, steps, and guidance scales are listed with the full tables in Appendix~\ref{app:quant}.

\textbf{Training data of the comparisons.} Table~\ref{tab:appendix-training-data} lists the training or adaptation data of every training-based removal method compared in Appendix~\ref{app:quant}; the two base inpainting models (SDXL-Inpainting and FLUX.1-Fill~\citep{fluxfill2024}) are omitted, as they carry web-scale pretraining rather than removal supervision. Overlap with a benchmark is judged at the dataset-family level regardless of split, from the training data stated in each paper or from a release that precedes the benchmark: OBER (2025) postdates every method except OSOR, and RORD (2022) precedes every method, so the RORD column rests on the stated training data alone.
\begin{table}[H]
\centering
\newcommand{\yes}{\textcolor{green!60!black}{\checkmark}}
\newcommand{\no}{\textcolor{red}{\ding{55}}}
\caption{\textbf{Training data of the training-based removal methods and its overlap with the evaluation benchmarks.} \protect\yes{}: stated to include the benchmark family; \protect\no{}: no stated overlap.}
\label{tab:appendix-training-data}
\small
\setlength{\tabcolsep}{5pt}
\renewcommand{\arraystretch}{1.15}
\begin{tabular}{llcc}
\toprule
Method & Training or adaptation data & OBER & RORD \\
\midrule
LaMa\venue{WACV'22} & Places-Challenge, 4.5M-image subset (Big LaMa) & \no & \no \\
PowerPaint\venue{ECCV'24} & OpenImages v6 (seg.\ subset) + LAION-Aesthetics v2 5+ & \no & \no \\
CLIPAway\venue{NeurIPS'24} & frozen base; embedding projection trained on COCO & \no & \no \\
EntityErasure\venue{CVPR'25} & Open Images V7, about 1.75M (pseudo-labeled) & \no & \no \\
SmartEraser\venue{CVPR'25} & Syn4Removal, 1M (Open Images v7, SAM, COCONut) & \no & \no \\
EraseLoRA\venue{ECCV'26} & none (dataset-free test-time adaptation) & \no & \no \\
OSOR (SDXL-Inp.)\venue{ECCV'26} & CORNE, 280K pairs (curated from NHR-Edit) & \no & \no \\
OSOR (FLUX.1-Fill)\venue{ECCV'26} & CORNE + refinement on the captured subset of OBER & \yes & \no \\
\bottomrule
\end{tabular}
\end{table}

\textbf{Pretraining data of the backbones.} Table~\ref{tab:appendix-training-data} concerns removal-specific training data; it does not establish that the frozen backbones have not seen benchmark images during pretraining, and we state the two claims separately. OBER-Test was captured by its authors for the dataset and first released in 2025, after the four distilled backbones (SDXL-Turbo and SD-Turbo in 2023; SDXL-Lightning and Hyper-SDXL in 2024), so its images could not have been part of their pretraining data. RORD was released in 2022 and predates them: its frames were recorded for the dataset and are distributed as an archive, but the pretraining sets of the backbones are not fully documented, so for RORD we claim only the absence of known removal-specific training overlap. All methods within each block of Table~\ref{tab:main} use the same frozen backbone, so the comparison does not depend on differences between pretrained models.

\textbf{Runtime protocol.} Latencies in Table~\ref{tab:runtime} are measured on a single NVIDIA A100 80GB in fp16 at batch size $1$ with seed $0$. Each entry is the median per-image end-to-end wall-clock over OBER-Test with the first three images excluded as warmup: the timed call spans the full generation, including the VAE encode, the text-encoder pass, all denoising forwards, and the VAE decode, and excludes only one-time model loading. Forward counts are measured with a pre-forward hook on the denoiser backbone: the nominal count is the number of backbone calls, and the effective count sums the batch sizes of those calls, so batch duplication inside one call is counted (Appendix~\ref{app:auxpass}). All experiments run in fp16 with PyTorch 2.7.1, diffusers 0.37.0.dev0, and CUDA 12.8; LPIPS uses the AlexNet backbone of \citet{zhang2018lpips} at its released weights, and every run fixes seed 0.

\subsection{Metrics and Statistics}
\label{app:stats}

\textbf{Metrics.} All metrics are computed at each backbone's operating resolution (Appendix~\ref{app:parameterization}): the result $R$, the ground truth $G$, and the mask $M$ are resized to that resolution, with area interpolation for images and nearest-neighbor for the mask, which is then binarized. A 4-pixel border is excluded from all PSNR and SSIM computations. Masked PSNR averages the squared error over masked pixels and the three RGB channels on the 8-bit scale,
\begin{equation}
\mathrm{MSE}_M = \frac{1}{3\,\lvert M \rvert} \sum_{i:\,M_i=1} \sum_{c=1}^{3} \left( R_{i,c} - G_{i,c} \right)^2,
\qquad
\mathrm{PSNR}_M = 10 \log_{10}\!\frac{255^2}{\mathrm{MSE}_M},
\label{eq:app-psnr}
\end{equation}
where $\lvert M \rvert$ is the number of masked pixels. Masked SSIM~\citep{wang2004ssim} computes the standard SSIM map $S_c$ of each channel on the full image, using an $11 \times 11$ Gaussian window with $\sigma = 1.5$ and stabilizers $C_1 = (0.01 \cdot 255)^2$, $C_2 = (0.03 \cdot 255)^2$, and averages the map over the masked positions,
\begin{equation}
\mathrm{SSIM}_M = \frac{1}{3} \sum_{c=1}^{3} \frac{1}{\lvert M' \rvert} \sum_{i:\,M'_i=1} S_c(i),
\label{eq:app-ssim}
\end{equation}
where $M'$ restricts the mask to the valid region of the windowed map. LPIPS admits no pixel-wise restriction, so masked LPIPS is computed on a composite that differs from the ground truth only inside the mask,
\begin{equation}
X = M \odot R + (1 - M) \odot G,
\qquad
\mathrm{LPIPS}_M = \mathrm{LPIPS}\big( X\vert_{\beta},\, G\vert_{\beta} \big),
\label{eq:app-lpips}
\end{equation}
where $\beta$ is the bounding box of $M$ padded by 32 pixels with a minimum side of 64, and LPIPS uses the AlexNet backbone~\citep{zhang2018lpips}. All metrics are computed per image and averaged over the benchmark. Since masked PSNR and SSIM read only masked pixels and the LPIPS composite equals the ground truth outside the mask, discrepancies outside the mask, such as lighting drift or passers-by in RORD's separately captured ground-truth frames, do not enter any reported metric.

\textbf{Statistics.} All effects and metric differences are paired at the image level. Significance statements in the main text and this appendix use a paired sign-flip permutation test over images with $10{,}000$ permutations; where a family of related tests is reported (the four backbone gaps of Table~\ref{tab:appendix-firstforward}), Holm correction is applied and all conclusions survive it.

\subsection{Auxiliary-Pass Accounting of Baselines}
\label{app:auxpass}

\textbf{Effective forwards.} We count the number of network forward passes each method performs per removal, measured with a forward hook on the denoiser rather than taken from the nominal step count. Table~\ref{tab:appendix-auxpass} reports the hook-measured counts for the methods run in our pipeline on the native four-step schedule. 
Attentive Eraser duplicates the batch to evaluate the suppressed and unsuppressed branches of its redirection guidance inside one call, doubling the effective computation at every step; AdaEraser runs a separate reference forward alongside each target forward; DesignEdit modifies attention inside the existing forwards and adds none. 

ReState reuses the first-forward features for retrieval, so it likewise stays at the native count. For the two methods of Table~\ref{tab:concept} that are not run in our pipeline, the auxiliary-pass classification follows the published procedures: PANDORA combines a stored-latent pass with its latent guidance, and FuLLaMa refines an initial removal through repeated diffusion passes, scoring each iterate with LVLM and CLIP evaluations during its runtime parameter search.

\begin{table}[H]
\centering
\caption{\textbf{ReState adds no auxiliary pass and stays at the native forward count.} Hook-measured on the denoiser under the native four-step schedule: nominal counts backbone calls and effective sums their batch sizes (Appendix~\ref{app:protocol}), so the two doubling mechanisms differ; Attentive Eraser duplicates the batch, AdaEraser calls twice. ReState's retrieval probe is shared with the first generation forward and adds none.}
\label{tab:appendix-auxpass}
\vspace{4pt}
\setlength{\tabcolsep}{10pt}
\renewcommand{\arraystretch}{1.15}
\begin{tabular}{l ccl}
\toprule
Method & Nominal & Effective & Auxiliary pass \\
\midrule
SDXL-Turbo & 4 & 4 & none \\
$+$ Attentive Eraser\venue{AAAI'25} & 4 & 8 & SARG branch pair \\
$+$ DesignEdit\venue{AAAI'25} & 4 & 4 & none \\
$+$ AdaEraser\venue{ICML'26} & 8 & 8 & reference forward \\
\rowcolor{ourshl}$+$ \textbf{ReState (Ours)} & 4 & 4 & none \\
\bottomrule
\end{tabular}
\end{table}

\subsection{Baseline Ports and Re-Tuning}
\label{app:ports}

\textbf{Ports.} The four-step rows of Table~\ref{tab:main} run the attention rule of each method inside the shared pipeline of Appendix~\ref{app:protocol}, so that only the intervention differs between rows. Table~\ref{tab:appendix-ports} lists the components of each official implementation and how the port treats them. The DesignEdit port keeps the key masking of self-attention and does not use DDIM inversion, the two-layer generation with layer fusion, classifier-free guidance, or proximal guidance; it therefore coincides with the attention-forcing-only row of Table~\ref{tab:ablation}. The Attentive Eraser port keeps the official SARG scale, the official suppression scale, and the same guidance form, and applies the attention redirection at all self-attention layers instead of the up-block layers only. AdaEraser has no public implementation, and we implement it from its paper~\citep{adaeraser2026}; the paper itself reports a clear degradation on four-step SDXL-Lightning, which is consistent with its row in Table~\ref{tab:main}. The official implementations of DesignEdit and Attentive Eraser at their native step counts are reported in Appendix~\ref{app:multistep-baseline}.

\begin{table}[!t]
\centering
\caption{\textbf{The four-step ports keep the attention rule of each method and leave out its multi-step components.} Components of the official implementations and their treatment in the ports. AdaEraser has no public implementation, so its official column follows the paper.}
\label{tab:appendix-ports}
\vspace{4pt}
\small
\setlength{\tabcolsep}{5pt}
\renewcommand{\arraystretch}{1.15}
\begin{tabular}{>{\raggedright\arraybackslash}p{2.9cm} >{\raggedright\arraybackslash}p{5.4cm} >{\raggedright\arraybackslash}p{4.5cm}}
\toprule
Component & Official implementation & Four-step port \\
\midrule
\multicolumn{3}{l}{\textit{DesignEdit}\venue{AAAI'25}} \\
Backbone and steps & SDXL, DDIM, 50 steps & SDXL-Turbo or SD-Turbo, 4 steps \\
DDIM inversion & 50 steps; gives the start latent and the background trajectory & not used; starts from the noised source latent \\
Guidance & classifier-free guidance 7.5 and proximal guidance & not used \\
Two-layer generation and layer fusion & reconstruction and inpainting layers, fused outside the mask & single sample; source-latent anchoring outside the mask at every step \\
Key masking & all self-attention layers, steps 1--39 of 50 & all self-attention layers, all 4 steps \\
\midrule
\multicolumn{3}{l}{\textit{Attentive Eraser}\venue{AAAI'25}} \\
Backbone and steps & SDXL, DDIM, 40 of 50 steps (strength 0.8) & SDXL-Turbo or SD-Turbo, 4 steps \\
Attention redirection & up-block self-attention layers, all steps & all self-attention layers, all steps \\
SARG scale & 9 & 9 (re-tuned in Table~\ref{tab:appendix-sarg}) \\
Similarity suppression & scale 0.3, first 10 of 40 steps & scale 0.3, first of 4 steps \\
Background blending & outside the mask at every step & same \\
\midrule
\multicolumn{3}{l}{\textit{AdaEraser}\venue{ICML'26}} \\
Backbone and steps & SDXL; step count not stated & SDXL-Turbo or SD-Turbo, 4 steps \\
Reference & noised source latent at each step, no inversion & same \\
Presence score and suppression & all self-attention layers & same \\
Reference and target passes & one batched call & two sequential calls \\
Background blending & outside the mask at every step & same \\
\bottomrule
\end{tabular}
\end{table}

\textbf{Re-tuning.} The official SARG scale of Attentive Eraser is set for 40-step sampling, so we sweep it on both Turbo backbones (Table~\ref{tab:appendix-sarg}). On SDXL-Turbo, the scale is selected by masked PSNR on one benchmark and reported on the other; both directions select $0.9$, which raises OBER-Test from $7.09$ to $15.34$~dB and RORD from $7.35$ to $14.18$~dB. On SD-Turbo, we take the scale that is best on OBER-Test itself, $2.25$, which reaches $14.95$~dB. ReState remains $3.0$ to $3.4$~dB higher in masked PSNR.

\begin{table}[!htb]
\centering
\caption{\textbf{Re-tuning the SARG scale recovers the Attentive Eraser port at four steps, and ReState remains more than 3~dB higher.} The first row is the official scale; the selected scale of each backbone is in bold.}
\label{tab:appendix-sarg}
\vspace{4pt}
\small
\setlength{\tabcolsep}{4.5pt}
\renewcommand{\arraystretch}{1.15}
\begin{tabular}{l ccc ccc ccc}
\toprule
& \multicolumn{6}{c}{SDXL-Turbo} & \multicolumn{3}{c}{SD-Turbo} \\
\cmidrule(lr){2-7}\cmidrule(lr){8-10}
& \multicolumn{3}{c}{OBER-Test} & \multicolumn{3}{c}{RORD} & \multicolumn{3}{c}{OBER-Test} \\
\cmidrule(lr){2-4}\cmidrule(lr){5-7}\cmidrule(lr){8-10}
SARG scale & PSNR$\uparrow$ & SSIM$\uparrow$ & LPIPS$\downarrow$ & PSNR$\uparrow$ & SSIM$\uparrow$ & LPIPS$\downarrow$ & PSNR$\uparrow$ & SSIM$\uparrow$ & LPIPS$\downarrow$ \\
\midrule
9 (official) & 7.09 & 0.046 & 0.514 & 7.35 & 0.065 & 0.538 & 7.28 & 0.063 & 0.516 \\
4.5 & 7.61 & 0.054 & 0.509 & 7.54 & 0.056 & 0.537 & 9.89 & 0.121 & 0.441 \\
2.25 & 12.92 & 0.264 & 0.377 & 12.56 & 0.277 & 0.416 & \textbf{14.95} & \textbf{0.329} & \textbf{0.277} \\
0.9 & \textbf{15.34} & \textbf{0.335} & \textbf{0.277} & \textbf{14.18} & \textbf{0.310} & \textbf{0.342} & 14.46 & 0.303 & 0.286 \\
\midrule
\rowcolor{ourshl}\textbf{ReState (Ours)} & 18.75 & 0.498 & 0.194 & 17.21 & 0.451 & 0.273 & 18.19 & 0.485 & 0.186 \\
\bottomrule
\end{tabular}
\end{table}

% ============================================================
\section{Intervention Analysis}
\label{app:analysis}

\subsection{Complete First-Forward Analysis}
\label{app:intervention}

\textbf{Protocol.} For a chosen forward $t$, the base and intervened runs start from the identical state $z_t$ under the identical seed. The attention-forcing arm replaces $\mathcal{A}_t$ by $\mathcal{A}^{\mathrm{AF}}$ at that single forward; the direct arm perturbs the clean prediction by $\Delta^{\mathrm{direct}}_t$, rescaled to $\lVert \Delta^{\mathrm{AF}}_t \rVert_2$ as in the main text, and the perturbed prediction is re-expressed as a noise prediction through Eq.~\ref{eq:x0hat} before the sampler update. All subsequent forwards run without any intervention. The source-to-GT directional residual $p(I, M)$ involves no learned scorer; the only model is the VAE encoder. With $z = \mathcal{E}(I)$, $z_{\mathrm{src}} = \mathcal{E}(\xsrc)$, and $z_{\mathrm{gt}}$ the encoding of the object-free ground truth, it is the scalar projection of the output onto the source-to-ground-truth direction over the latent positions of the mask,
\begin{equation}
p(I, M) \;=\; \frac{\big\langle z - z_{\mathrm{gt}},\; z_{\mathrm{src}} - z_{\mathrm{gt}} \big\rangle_{M}}{\lVert z_{\mathrm{src}} - z_{\mathrm{gt}} \rVert_{M}^{2}},
\label{eq:app-pscore}
\end{equation}
so $p = 1$ when the output matches the source and $p = 0$ when the output matches the object-free ground truth. Effects are averaged over $n=163$ images. All p-values in this appendix are from a paired sign-flip permutation test over images with $10{,}000$ permutations; $p = 1.0\times10^{-4}$ is the smallest value expressible at this count.

\textbf{Direction controls.} The matched norm alone does not produce removal. Keeping the perturbation norm fixed and changing only its direction gives, at the first forward on SDXL-Turbo: toward retrieved background $+0.0251$ ($p = 1.0\times10^{-4}$); toward the ground-truth background (oracle direction) $+0.0487$ ($p = 1.0\times10^{-4}$); random direction $-0.0004$ ($p = 0.53$); direction orthogonal to the background direction $-0.0006$ ($p = 0.36$). Equal-norm perturbations therefore reduce the source-to-GT residual only when they point toward background content, which validates the norm-matched comparison of the main text. The terminal effect also tracks realized quality at the image level: across the first-forward runs, the per-image gap $\Delta E$ between the direct and attention arms correlates with the corresponding masked PSNR and LPIPS changes of the same outputs (Spearman $\rho = 0.46$ and $0.38$ over $n{=}163$ images, both $p < 10^{-6}$), and among the 133 of 163 images with $\Delta E > 0$, $94.7\%$ also improve in masked PSNR and $93.2\%$ in LPIPS.

\textbf{Full first-forward results.} Table~\ref{tab:appendix-firstforward} reports the means of Table~\ref{tab:mechanism} with their tests: the permutation p-value that the effect of attention forcing is zero, and that the gap $d$ is zero. The gaps of all four distilled backbones remain significant after Holm correction across the four tests (all corrected p-values $\le 4\times10^{-4}$), while the first-forward effect of attention forcing on the Turbo family is not significantly different from zero. With the original SDXL weights on the same grid both interventions act symmetrically and the gap vanishes. Figure~\ref{fig:appendix-scatter} shows the same comparison at the image level: on the four distilled backbones 71 to 82 percent of images lie above the diagonal, while the weight control sits at 49 percent, not significantly different from chance.

\begin{figure}[H]
\centering
\includegraphics[width=\textwidth]{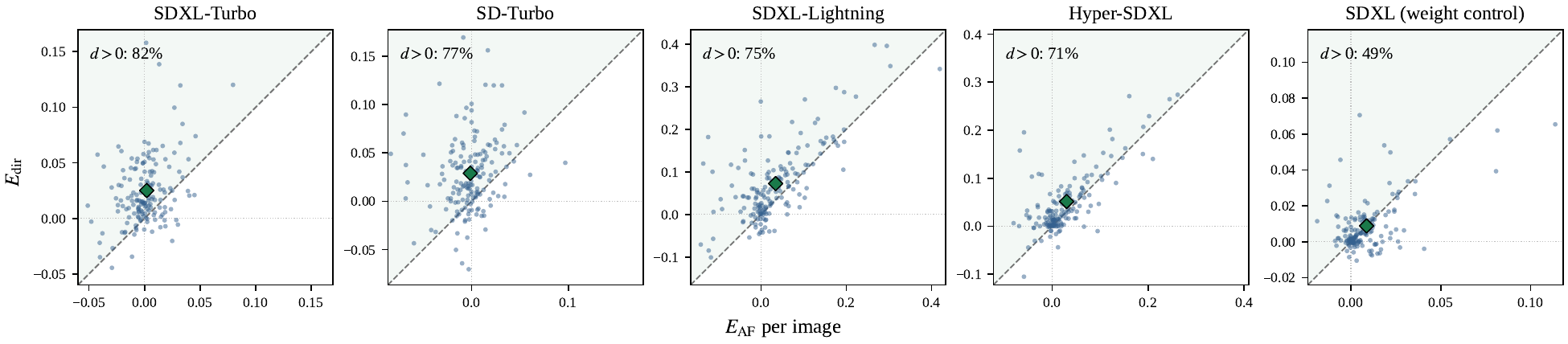}
\vspace{-13pt}
\caption{\textbf{The first-forward asymmetry holds image by image on the distilled backbones and vanishes with the original weights.} Each point is one image; the diamond marks the mean of Table~\ref{tab:appendix-firstforward}, and $d > 0$ is the fraction of images above the diagonal, where the direct correction wins. Axes are clipped at the $0.5$th and $99.5$th percentiles and scaled per panel, as the backbones operate at different resolutions.}
\label{fig:appendix-scatter}

\end{figure}

\begin{table}[h]
\centering
\caption{\textbf{At the first forward, direct correction has a larger effect than attention forcing on all four distilled backbones, but not with the original SDXL weights.} Means restate the first row of Table~\ref{tab:mechanism} at full precision. The p-values are paired permutation tests (Appendix~\ref{app:stats}), whose resolution floor is $1.0\times10^{-4}$; the original-SDXL row is the weight control.}
\label{tab:appendix-firstforward}
\setlength{\tabcolsep}{6pt}
\renewcommand{\arraystretch}{1.15}
\begin{tabular}{lccccc}
\toprule
Backbone & $\EAF$ & $\Edir$ & $d$ & $p(\EAF{=}0)$ & $p(d{=}0)$ \\
\midrule
SDXL-Turbo & $+0.0020$ & $+0.0251$ & $+0.0231$ & $0.27$ & $1.0\times10^{-4}$ \\
SD-Turbo & $-0.0013$ & $+0.0289$ & $+0.0303$ & $0.52$ & $1.0\times10^{-4}$ \\
SDXL-Lightning & $+0.0339$ & $+0.0723$ & $+0.0384$ & $1.0\times10^{-4}$ & $1.0\times10^{-4}$ \\
Hyper-SDXL & $+0.0302$ & $+0.0516$ & $+0.0214$ & $1.0\times10^{-4}$ & $1.0\times10^{-4}$ \\
\midrule
Original SDXL (4-step grid) & $+0.0087$ & $+0.0088$ & $+0.0001$ & $1.0\times10^{-4}$ & $0.91$ \\
\bottomrule
\end{tabular}
\end{table}

\subsection{Multi-Step Reference Intervention}
\label{app:multistep}

Applying the identical matched single-forward protocol along the 50-step SDXL reference trajectory yields Table~\ref{tab:appendix-multistep}. Both interventions have terminal effects an order of magnitude smaller than on the native four-step trajectory, and no first-forward asymmetry comparable to Table~\ref{tab:appendix-firstforward} appears: a single early intervention, of either kind, is largely absorbed by the remaining trajectory. This is the reference behavior against which the few-step results should be read. Figure~\ref{fig:appendix-timing} places the two trajectories side by side.

\begin{table}[!htb] 
\centering
\caption{\textbf{On the reference trajectory both single-forward interventions stay an order of magnitude below the few-step effects.} Exact values at the eight probed noise levels; the corresponding curves are the right panel of Figure~\ref{fig:appendix-timing}.}
\label{tab:appendix-multistep}
\setlength{\tabcolsep}{10pt}
\renewcommand{\arraystretch}{1.15}
\begin{tabular}{cccc}
\toprule
Forward index & Noise level $\sigma$ & $\EAF$ & $\Edir$ \\
\midrule
1 & 13.12 & $+0.0015$ & $+0.0012$ \\
4 & 9.34 & $+0.0024$ & $+0.0020$ \\
8 & 6.20 & $+0.0030$ & $+0.0033$ \\
13 & 3.95 & $+0.0036$ & $+0.0040$ \\
19 & 2.46 & $+0.0028$ & $+0.0036$ \\
25 & 1.62 & $+0.0027$ & $+0.0033$ \\
31 & 1.09 & $+0.0021$ & $+0.0029$ \\
38 & 0.67 & $+0.0014$ & $+0.0025$ \\
\bottomrule
\end{tabular}
\end{table}

\begin{figure}[!htb]
\centering
\includegraphics[width=\textwidth]{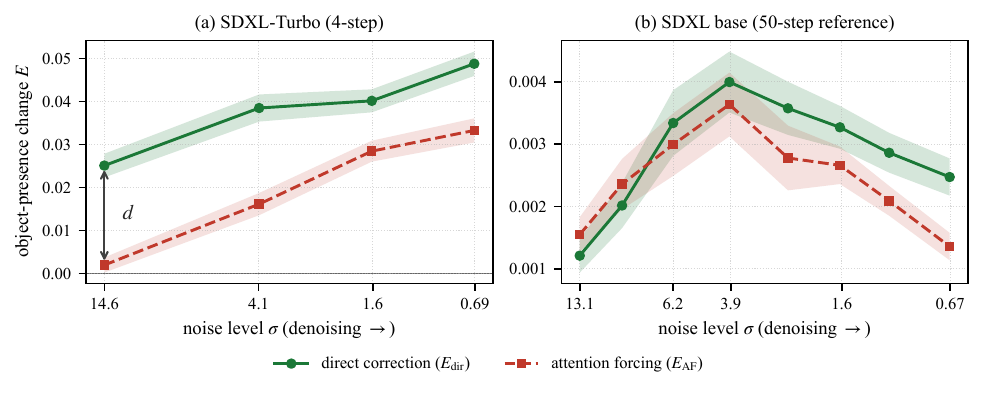}
\vspace{-13pt}
\caption{\textbf{On the reference trajectory both interventions are an order of magnitude weaker and the first-forward asymmetry vanishes.} Terminal intervention effect $E$ of a single intervened forward at each noise level; bands are $\pm 1$ standard error over $n = 163$ images. The left panel plots the SDXL-Turbo column of Table~\ref{tab:mechanism}; the right panel plots Table~\ref{tab:appendix-multistep}, on a vertical scale an order of magnitude smaller.}
\label{fig:appendix-timing}
\end{figure}

\subsection{Diagnostic Schedule Controls}
\label{app:controls}

Together with Tables~\ref{tab:appendix-firstforward} and~\ref{tab:appendix-multistep}, the two controls of this section complete a two-by-two design over the weights (distilled or original) and the schedule (four-step or dense): the intervention asymmetry appears only where distilled weights meet the coarse schedule. Both control configurations run outside the native inference regime of the respective model and are therefore diagnostic rather than a causal factorization. First, the original SDXL weights on the same four-step grid do not reproduce the intervention gap: both interventions act moderately and symmetrically (last row of Table~\ref{tab:appendix-firstforward}). Second, densifying the same SDXL-Turbo weights to 8 and 16 denoising evaluations shrinks the terminal effect of the first direct correction monotonically while the effect of attention forcing stays at its floor (Table~\ref{tab:appendix-ksweep}): additional denoising opportunities reduce how much a single early correction survives to the output, and do not unlock an early effect of attention forcing.

\begin{table}[h]
\centering
\caption{\textbf{Densifying the same Turbo weights shrinks the direct effect and does not unlock attention forcing.} First-forward effects at 4, 8, and 16 denoising evaluations; the native row anchors to Table~\ref{tab:appendix-firstforward}.}
\label{tab:appendix-ksweep}
\setlength{\tabcolsep}{10pt}
\renewcommand{\arraystretch}{1.15}
\begin{tabular}{lcccc}
\toprule
Denoising evaluations & $\EAF$ & $p(\EAF{=}0)$ & $\Edir$ & $p(\Edir{=}0)$ \\
\midrule
4 & $+0.0020$ & $0.27$ & $+0.0251$ & $1.0\times10^{-4}$ \\
8 & $+0.0018$ & $0.21$ & $+0.0122$ & $1.0\times10^{-4}$ \\
16 & $+0.0018$ & $0.09$ & $+0.0059$ & $1.0\times10^{-4}$ \\
\bottomrule
\end{tabular}
\end{table}

On the 50-step reference we additionally fix the number of forced forwards and vary only their placement, so that any difference is attributable to timing rather than to the amount of intervention; the main text quotes the direction of this control, and Table~\ref{tab:appendix-placement} gives the configurations and numbers. The two eight-forward windows sit at the extremes of the schedule, $\sigma \in [13.1, 6.2]$ and $\sigma \in [0.44, 0.04]$, and the late window wins by $0.9$~dB. Both fall far short of the sustained configurations: extending attention forcing over the last 38 forwards recovers most of the reference quality, and only the full schedule reaches it. Timing matters at a fixed budget, but on the dense trajectory attention forcing needs sustained application to act at all.

\begin{table}[H]
\centering
\caption{\textbf{At a fixed count, late attention forcing beats high-noise placement, and only the full schedule reaches the reference.} 50-step SDXL with the guidance of the reference configuration; the all-50 row is the 50-step attention-forcing reference of Table~\ref{tab:runtime}.}
\label{tab:appendix-placement}
\setlength{\tabcolsep}{10pt}
\renewcommand{\arraystretch}{1.15}
\begin{tabular}{lcccc}
\toprule
Placement & $\sigma$ range & PSNR$\uparrow$ & SSIM$\uparrow$ & LPIPS$\downarrow$ \\
\midrule
8 forwards, high-noise range & $13.1 \rightarrow 6.2$ & 11.57 & 0.255 & 0.399 \\
8 forwards, late range & $0.44 \rightarrow 0.04$ & 12.47 & 0.284 & 0.353 \\
Last 38 forwards & $3.95 \rightarrow 0.04$ & 17.17 & 0.429 & 0.221 \\
All 50 forwards & $13.1 \rightarrow 0.04$ & 17.86 & 0.441 & 0.209 \\
\bottomrule
\end{tabular}
\end{table}

\subsection{Retention Analysis}
\label{app:retention}

\textbf{Retention.} We measure how much a perturbation injected at the first forward survives to the final output. At forward $k$ the predicted clean latent is pushed a fraction $\beta = 0.1$ toward the ground-truth background inside the mask, $\hat{x}_0' = (1-\beta)\,\hat{x}_0 + \beta\,z_{\mathrm{gt}}$, re-expressed as a noise prediction, and the trajectory then runs without further intervention. Retention is the ratio of the final to the immediate displacement between the injected and base runs, $s_k = \lVert \hat{x}_0^{\mathrm{final,inj}} - \hat{x}_0^{\mathrm{final,base}} \rVert_M \,/\, \lVert z^{\mathrm{inj}} - z^{\mathrm{base}} \rVert_M$, with masked RMS norms and the denominator measured at the input of forward $k{+}1$. The ratio is verified to be independent of the dose $\beta$, so it is the slope of the linear displacement response and applies to small perturbations along this direction. At the first forward the median over images is $1.38$ on the native four-step SDXL-Turbo trajectory and $0.30$ on the 50-step reference: the few-step trajectory preserves, and can even amplify, what the first forward establishes, while the dense trajectory attenuates it.

\subsection{Residual-Frequency Analysis}
\label{app:frequency}

We compare the high-band energy of the masked region against the ground-truth background. The band is the radial range $[0.45, 1.0)$ of Nyquist of the two-dimensional FFT power, computed on the grayscale bounding-box crop of the mask resized to $128^2$ under a Hann window; the ground-truth crop is processed identically, and each image contributes the ratio of the two. Table~\ref{tab:appendix-frequency} reports geometric means over images. Without intervention the masked region carries $8.3\times$ the high-band energy of the true background, the spectral signature of the object itself. Attention forcing alone reduces this only to $1.92\times$: the residual is object-shaped high-frequency structure. The direct correction alone lands at $0.30\times$, below parity, as expected from evidence $B$ that averages background latents. ReState recovers to $0.51\times$: the object-shaped excess is eliminated, and the attention-constrained late forwards restore part of the missing texture. The remaining gap to parity is a smoother-than-real fill, not residual object structure, which is the opposite failure mode of attention forcing. Figure~\ref{fig:appendix-limit} shows this gap on four images: the results are plausible at the original scale, and the softer texture becomes visible under magnification, to a degree that varies across images.

\begin{table}[H]
\centering
\caption{\textbf{Attention forcing leaves object-shaped high-frequency structure; ReState eliminates it.} Masked high-band energy as a per-image ratio to the identically processed ground truth, geometric mean over OBER-Test (SDXL-Turbo).}
\label{tab:appendix-frequency}
\setlength{\tabcolsep}{10pt}
\renewcommand{\arraystretch}{1.15}
\begin{tabular}{lc}
\toprule
Output & High-band ratio to ground truth \\
\midrule
No intervention & 8.34 \\
Attention forcing only & 1.92 \\
Direct correction only & 0.30 \\
\rowcolor{ourshl}\textbf{ReState} & 0.51 \\
\bottomrule
\end{tabular}
\end{table}

% ============================================================
% Appendix D: Method Ablations
% ============================================================
\section{Method Ablations}
\label{app:ablations}

Each ablation in this section varies one component of the default ReState configuration (Appendix~\ref{app:protocol}) on SDXL-Turbo and reports masked PSNR, SSIM, and LPIPS on OBER-Test. Tables report point estimates, and every statistical statement in the text is a two-sided paired permutation test against the default configuration under the protocol of Appendix~\ref{app:stats}.

% ------------------------------------------------------------
\subsection{Correction Strength}
\label{app:abl-strength}

Table~\ref{tab:appendix-lsweep} sweeps $\lambda$ over a $32\times$ range with all other settings fixed. Removal quality is flat: masked PSNR varies by at most $0.14$~dB across the sweep, and the default $\lambda = 8$ is within $0.01$~dB of the optimum. The realized weights $\omega_t$ shift substantially with $\lambda$, so the insensitivity is a property of the correction itself rather than of a narrow operating point.

This flatness is a statement about the strength profile within the structure that the closed form imposes, with strong corrections while noise remains and none at the final forward; that the structure itself is necessary is shown in Section~\ref{app:abl-timing}.

\begin{table}[H]
\centering
\caption{\textbf{Removal quality is insensitive to the correction strength $\lambda$.} Realized weights $\omega_t = \lambda\sigma_{t-1}^2/(1+\lambda\sigma_{t-1}^2)$ at the first three forwards; $\omega_4 = 0$ for all $\lambda$ by the closed form. The default $\lambda = 8$ is highlighted.}
\label{tab:appendix-lsweep}
\setlength{\tabcolsep}{9pt}
\renewcommand{\arraystretch}{1.15}
\begin{tabular}{c ccc ccc}
\toprule
& \multicolumn{3}{c}{Realized $\omega$} & \multicolumn{3}{c}{OBER-Test} \\
\cmidrule(lr){2-4} \cmidrule(lr){5-7}
$\lambda$ & $\omega_1$ & $\omega_2$ & $\omega_3$ & PSNR$\uparrow$ & SSIM$\uparrow$ & LPIPS$\downarrow$ \\
\midrule
1  & 0.943 & 0.722 & 0.325 & 18.61 & 0.487 & 0.194 \\
2  & 0.971 & 0.839 & 0.490 & 18.70 & 0.492 & 0.193 \\
4  & 0.985 & 0.912 & 0.658 & 18.74 & 0.496 & 0.194 \\
\rowcolor{ourshl}8  & 0.993 & 0.954 & 0.794 & 18.75 & 0.498 & 0.194 \\
16 & 0.996 & 0.977 & 0.885 & 18.73 & 0.498 & 0.195 \\
32 & 0.998 & 0.988 & 0.939 & 18.72 & 0.498 & 0.195 \\
\bottomrule
\end{tabular}
\end{table}

% ------------------------------------------------------------
\subsection{Correction Timing and Schedule Structure}
\label{app:abl-timing}

Table~\ref{tab:appendix-timing} restricts the direct correction to subsets of the four forwards, with object-excluding attention active at every forward in all rows. The windows are ordered primarily by the last corrected forward: windows ending at the third forward reach $18.47$ to $18.75$~dB on OBER-Test, windows ending at the second reach $18.18$ to $18.35$~dB, and correcting only the first forward reaches $16.85$~dB. Within a fixed last forward, earlier corrections add a secondary refinement ($+0.27$~dB from $\{3\}$ to $\{2,3\}$, $+0.17$~dB from $\{2\}$ to $\{1,2\}$). This ordering follows from state propagation: every uncorrected forward re-predicts the masked content with whatever generative freedom the remaining noise allows, so the final result is governed by the freedom that remains after the last applied correction. The same absorption principle governs trajectory length in Appendix~\ref{app:controls}, where additional denoising evaluations reduce how much a single early correction survives to the output. These windows measure deployed quality with the remaining schedule uncorrected; the persistence of a single matched perturbation is the separate measurement of Appendix~\ref{app:retention}. The two measurements are consistent: limited absorption (Appendix~\ref{app:retention}) means an early correction survives, and the ordering by the last corrected forward follows because every uncorrected forward re-predicts the masked content; the window ordering is therefore a consequence of the same absorption principle, not a counterexample to it.

Within this pattern the first-forward correction becomes redundant once the later forwards are corrected: the $\{2,3\}$ window is not significantly different from the full schedule on PSNR and LPIPS ($p = 0.54$ and $p = 0.90$), with a marginal SSIM difference ($0.496$ versus $0.498$, $p = 0.002$). ReState nevertheless keeps the full schedule, because the closed form of Eq.~\ref{eq:update} assigns every forward its weight without a window hyperparameter, and pruning to $\{2,3\}$ would introduce a discrete rule for no measurable gain. The first-forward correction carries the gain when later corrections are absent: with a two-step schedule, the correction is applied only at the first forward, and ReState reaches $16.38$~dB against $13.12$~dB for attention forcing alone (Appendix~\ref{app:abl-steps}).

The last row removes the boundary condition instead: keying $\omega$ to $\sigma_t$ rather than $\sigma_{t-1}$ leaves the final forward a correction of $\omega_4 \approx 0.79$, whose corrected prediction is the final output. Quality collapses perceptually, with LPIPS degrading from $0.194$ to $0.284$ and PSNR losing $1.49$~dB (both $p = 1.0\times10^{-4}$). This row also answers why the schedule does not shift weight toward the later forwards although single late windows imprint best: a late correction locks in whatever it writes, including the smoothness of the averaged evidence $B$, so the correction should be strong while the prediction is still revisable and absent once it is not. The final forward is left uncorrected, and the $\sigma_{t-1}$ keying supplies exactly this as a boundary condition of the objective rather than as an added rule.

Restoring that boundary by hand under $\sigma_t$ keying, with an explicit final-forward skip, is statistically not significantly different from the default on every metric ($18.71$ versus $18.75$~dB, $p = 0.23$; LPIPS $0.195$ versus $0.194$, $p = 0.08$). The keying is therefore not a performance choice: $\sigma_{t-1}$ absorbs the final-forward exception into the closed form. A complementary sweep isolates the terminal weight: with $\omega_1$--$\omega_3$ fixed at their closed-form values, raising $\omega_4$ from $0$ to $0.2$, $0.5$, and $0.794$ degrades masked LPIPS monotonically from $0.194$ to $0.210$, $0.251$, and $0.283$ (significant from $\omega_4 = 0.2$, $p = 1.0\times10^{-4}$), with PSNR and SSIM following from $\omega_4 \geq 0.5$. The benefit of the terminal release is therefore separable from the earlier weights: a terminal direct correction pulls the output toward the smoothness of the averaged evidence $B$ exactly where fine texture is synthesized.

A constant-strength schedule isolates the weight profile itself: $\omega_1 = \omega_2 = \omega_3 = c$ with $\omega_4 = 0$, so the terminal release is kept and only the profile changes (Table~\ref{tab:appendix-timing}). With $c$ matched to the mean of the closed-form weights ($c = 0.914$), the default is never significantly worse on any metric, and the differences are at most $0.03$~dB and $0.0012$ LPIPS. Quality is flat for $c$ between $0.794$ and $1.0$, where $c = 1.0$ is the binary schedule $(1, 1, 1, 0)$, and it degrades as the constant becomes weak: $c = 0.5$ loses $0.2$~dB and $c = 0.25$ loses $1.2$~dB. The weight profile is therefore not where the performance comes from; the operative factors are the terminal release and an adequate overall strength, and the closed form of Eq.~\ref{eq:update} supplies both from one objective with the single hyperparameter $\lambda$, with no window or constant to tune. The same rule gives the weights of shorter schedules; at the one-step extreme, the only forward is terminal and the correction vanishes (Appendix~\ref{app:abl-steps}).

\begin{table}[!htb]
\centering
\caption{\textbf{The correction is cumulative and the final forward is left uncorrected.} The direct correction is applied only at the listed forwards; object-excluding attention stays active at every forward in all rows. The last row keys $\omega$ to $\sigma_t$, so the final forward retains a correction of $\omega_4 \approx 0.79$ whose output is final. Constant rows replace the closed-form weights by a single constant $c$ at the first three forwards while keeping $\omega_4 = 0$.}
\label{tab:appendix-timing}
\setlength{\tabcolsep}{7pt}
\renewcommand{\arraystretch}{1.15}
\begin{tabular}{lccc}
\toprule
Corrected forwards & PSNR$\uparrow$ & SSIM$\uparrow$ & LPIPS$\downarrow$ \\
\midrule
None (attention forcing only) & 15.50 & 0.348 & 0.269 \\
$\{1\}$ & 16.85 & 0.411 & 0.237 \\
$\{2\}$ & 18.18 & 0.469 & 0.203 \\
$\{3\}$ & 18.47 & 0.482 & 0.200 \\
$\{1,2\}$ & 18.35 & 0.477 & 0.198 \\
$\{2,3\}$ & 18.74 & 0.496 & 0.194 \\
\rowcolor{ourshl}$\{1,2,3\}$ (default) & 18.75 & 0.498 & 0.194 \\
Constant $c = 0.25$, $\{1,2,3\}$ & 17.54 & 0.443 & 0.217 \\
Constant $c = 0.5$, $\{1,2,3\}$ & 18.53 & 0.486 & 0.197 \\
Constant $c = 0.794$ ($= \omega_3$), $\{1,2,3\}$ & 18.76 & 0.496 & 0.194 \\
Constant $c = 0.914$ (mean-matched), $\{1,2,3\}$ & 18.73 & 0.497 & 0.195 \\
Constant $c = 1.0$ (binary), $\{1,2,3\}$ & 18.68 & 0.497 & 0.196 \\
$\{1,2,3,4\}$ ($\omega$ keyed to $\sigma_t$) & 17.26 & 0.466 & 0.284 \\
\bottomrule
\end{tabular}

\end{table}

% ------------------------------------------------------------
\subsection{Background Evidence Source}
\label{app:abl-evidence}

Replacing the retrieved $B$ (Table~\ref{tab:appendix-evidence}) shows both directions of the dependence: a global background mean, which removes spatial correspondence, costs $0.59$~dB, while the ground-truth background as an oracle target gains $3.57$~dB. The correction mechanism is therefore not the bottleneck; the quality of the evidence is. The global-mean row is not a plain mean fill, since it uses the mean as $B$ within the same sampling. Decoding $B$ directly into the mask, without any denoising, gives $15.13$~dB for the global mean and $16.55$~dB for the retrieved $B$ (Table~\ref{tab:appendix-evidence}); using the same $B$ within ReState is $3.0$ and $2.2$~dB higher, with better perceptual metrics as well.

\begin{table}[!htb]
\centering
\caption{\textbf{Evidence quality, not the correction mechanism, is the remaining bottleneck.} Each row supplies a different evidence $B$ to the identical correction. The first two rows decode $B$ into the mask directly, without denoising. The table extends Table~\ref{tab:evidence} with DISTS and CMMD (Appendix~\ref{app:metrics}).}
\label{tab:appendix-evidence}
\small
\setlength{\tabcolsep}{4.5pt}
\renewcommand{\arraystretch}{1.15}
\begin{tabular}{lccccc}
\toprule
Evidence $B$ & PSNR$\uparrow$ & SSIM$\uparrow$ & LPIPS$\downarrow$ & DISTS$\downarrow$ & CMMD$\downarrow$ \\
\midrule
Global background mean, decoded directly & 15.13 & 0.466 & 0.337 & 0.203 & 2.112 \\
Retrieved, decoded directly & 16.55 & 0.448 & 0.299 & 0.196 & 2.152 \\
\midrule
Global background mean & 18.16 & 0.481 & 0.214 & 0.133 & 1.398 \\
\rowcolor{ourshl}Retrieved (default) & 18.75 & 0.498 & 0.194 & 0.129 & 1.358 \\
Ground-truth background (oracle) & 22.32 & 0.588 & 0.133 & 0.095 & 0.994 \\
\bottomrule
\end{tabular}
\end{table}

\subsection{Per-Forward Correction versus One-Shot Latent Initialization}
\label{app:abl-prefill}

The preceding ablations vary the evidence and the schedule; this one varies only how the same evidence enters sampling, and asks whether the gain of ReState is explained by obtaining a good proposal $B$ alone. Both configurations share the first generation forward, the retrieved proposal $B$, the seed, and the noise realization, and both run four denoiser evaluations with identical attention forcing and anchoring. ReState applies its per-forward correction as in the main text. The initialization alternative instead writes $B$ into the trajectory once~\citep{meng2022sdedit}: at the noise level $\sigma_{\mathrm{next}}$ entering the second forward, the masked region is re-initialized as $z^{\mathrm{init}} = M \odot (B + \sigma_{\mathrm{next}}\,\epsilon) + (1-M) \odot (z_{\mathrm{src}} + \sigma_{\mathrm{next}}\,\epsilon)$, after which the remaining three forwards run without any prediction correction. This differs from the $\{1\}$ window of Appendix~\ref{app:abl-timing}, which corrects the clean prediction and lets the sampler form the next latent; the initialization reconstructs the noisy latent itself, so the two are separate measurements. Masked LPIPS was declared the primary metric before the runs, since a one-shot fill can favor PSNR through smoothing. Table~\ref{tab:appendix-prefill} reports the outcome: ReState is better on the primary metric on both benchmarks ($p = 1.0\times10^{-4}$; per-image win rate $62.0\%$ on OBER-Test, $62.7\%$ on RORD), and on masked PSNR and SSIM as well (win rates $69$--$90\%$). The gain of ReState is therefore not explained by the background proposal alone: at the same target and the same budget, applying the proposal as a per-forward prediction correction outperforms writing it into the latent once.

The outcome of a one-shot insertion depends on the noise level at which $B$ is inserted. Table~\ref{tab:appendix-prefill} (bottom) moves the entry point to the third and the fourth forward. The first forward is kept for retrieval and the forwards before the entry point are skipped, so these variants use three and two evaluations; they reach $18.53$ and $18.61$~dB on OBER-Test. The two-evaluation variant takes $0.70\times$ the end-to-end time of ReState and is $0.06$ to $0.22$~dB lower in all eight benchmark--backbone settings of the paper. We recommend it when latency matters: its DISTS and ReMOVE match those of the default ($0.129$ and $0.927$ against $0.129$ and $0.929$). The per-forward correction remains the default because it needs no choice of an entry point, and an early entry lowers the quality ($17.90$~dB).

\begin{table}[!htb]
\centering
\newcommand{\yes}{\textcolor{green!60!black}{\checkmark}}
\caption{\textbf{Per-forward correction versus one-shot latent initialization: the correction procedure itself adds value.} Both rows share the retrieved proposal $B$, the first generation forward, the noise realization, attention forcing, and anchoring, and use four denoiser evaluations. The initialization row writes $B$ into the noisy latent once before the second forward; ReState applies it as a per-forward prediction correction. The last row reports ReState minus initialization; masked LPIPS is the pre-declared primary metric, and every gain is significant under the protocol of Appendix~\ref{app:stats} ($p = 1.0\times10^{-4}$). Paired gains are computed from unrounded per-image values. Bottom: entry points of the one-shot insertion on OBER-Test, where a later entry point skips the forwards before it.}
\label{tab:appendix-prefill}
\setlength{\tabcolsep}{5pt}
\renewcommand{\arraystretch}{1.15}
\small
\begin{tabular}{lcc ccc ccc}
\toprule
& & & \multicolumn{3}{c}{OBER-Test} & \multicolumn{3}{c}{RORD} \\
\cmidrule(lr){4-6} \cmidrule(lr){7-9}
Method & Same $B$ & NFE & PSNR$\uparrow$ & SSIM$\uparrow$ & LPIPS$\downarrow$ & PSNR$\uparrow$ & SSIM$\uparrow$ & LPIPS$\downarrow$ \\
\midrule
Initialization with $B$ & \yes  & 4 & 17.90 & 0.454 & 0.211 & 16.39 & 0.417 & 0.284 \\
\rowcolor{ourshl}ReState (Ours) & \yes  & 4 & \textbf{18.75} & \textbf{0.498} & \textbf{0.194} & \textbf{17.21} & \textbf{0.451} & \textbf{0.273} \\
\midrule
Paired gain & & & $+0.85$ & $+0.044$ & $-0.017$ & $+0.83$ & $+0.033$ & $-0.012$ \\
\bottomrule
\end{tabular}

\vspace{6pt}
\begin{tabular}{lcc ccc}
\toprule
Entry point of $B$ & $\sigma$ at entry & NFE & PSNR$\uparrow$ & SSIM$\uparrow$ & LPIPS$\downarrow$ \\
\midrule
Before the second forward (as above) & 4.082 & 4 & 17.90 & 0.454 & 0.211 \\
Before the third forward & 1.613 & 3 & 18.53 & 0.488 & 0.196 \\
Before the fourth forward & 0.693 & 2 & 18.61 & \textbf{0.498} & 0.199 \\
\rowcolor{ourshl}ReState (per-forward correction) & -- & 4 & \textbf{18.75} & \textbf{0.498} & \textbf{0.194} \\
\bottomrule
\end{tabular}
\end{table}

% ------------------------------------------------------------
\subsection{Retrieval Ratio}
\label{app:abl-ratio}

Figure~\ref{fig:appendix-rsweep} widens the retained fraction $r$ over three orders of magnitude, up to the softmax-weighted average over the entire background. Below the default the bank is too small: $r = 0.001$ loses $0.33$~dB ($p = 1.0\times10^{-4}$) and $r = 0.005$ still registers ($p = 0.010$). Above the default the sweep locates only a shallow numerical optimum: PSNR peaks around $r = 0.1$ to $0.25$, at most $0.15$~dB above the default and not significant on any metric ($p \geq 0.08$), while LPIPS is never significantly better than at the default (all $p \geq 0.15$) and is numerically worst at the full bank. We retain $r = 0.01$: the peak is not statistically distinguishable from it, and adopting a peak located on the evaluation benchmark would in any case amount to selecting $r$ on the test set.

The sweep separates the two design choices of the retrieval. Even the full-bank softmax average stays at the level of the default on PSNR ($18.80$) while its LPIPS is the worst in the sweep, whereas removing the weighting entirely, the unweighted global mean of Table~\ref{tab:appendix-evidence}, costs $0.59$~dB. The robustness to $r$ is therefore a property of the similarity weighting rather than of a tuned retained fraction: across the entire range $r \in [0.005, 1]$, masked PSNR stays within $0.15$~dB of the default.

\begin{figure}[!htb]
\centering
\includegraphics[width=0.82\linewidth]{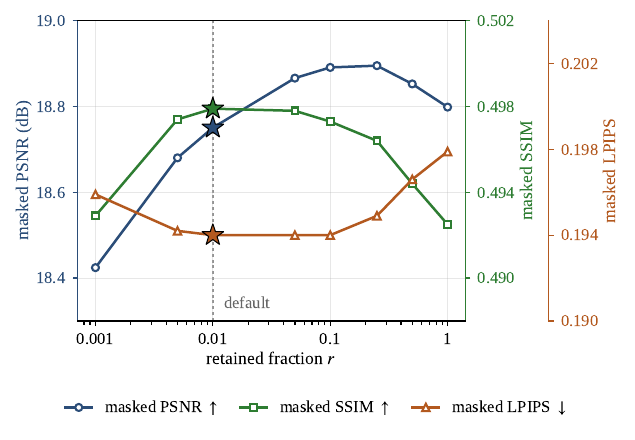}
\vspace{-13pt}
\caption{\textbf{Removal quality is insensitive to the retained fraction $r$ across three orders of magnitude.} Masked metrics on OBER-Test with the bank widened up to the similarity-weighted average over the entire background ($r = 1$); the stars mark the default $r = 0.01$. Exact test statistics are given in the text.}
\label{fig:appendix-rsweep}
\end{figure}

% ------------------------------------------------------------
\subsection{Probe Feature Source}
\label{app:abl-probe}

The background bank is built from features captured at the first generation forward ($\sigma = 14.6$, object-excluding attention active), at the output of the first up-block of the denoiser; the native resolution of this map at $512^2$ is $32 \times 32$, and it is bilinearly resampled to the $64 \times 64$ correspondence grid of the latent. Table~\ref{tab:appendix-probe} rebuilds the bank from alternative feature sources and regenerates with the otherwise identical pipeline. 

Every variant harvests its features in a separate pass of the diagnostic harness, so the generation trajectory itself is identical across rows and stays at the native four evaluations. For the suppression-off and noise-level rows this pass is required by construction: in the default the probe and the first generation forward are the same pass, run under suppression at $\sigma = 14.6$, so changing the probe condition inside that forward would change the generation as well. The layer rows need no such pass in principle and could share the generation forward exactly as the default does.

The shared probe is robust to the attention intervention under which it runs. Harvesting the same features without object-excluding attention changes nothing measurable ($18.72$ versus $18.75$~dB, $p \geq 0.46$ on every metric). Keeping the suppression active therefore costs nothing, and it requires no separate pass.

The noise level of the probe input, in contrast, is functional rather than incidental. Harvesting from a nearly clean input, $\mathcal{E}(x_{\mathrm{src}}) + \sigma_{\min}\epsilon$ at $\sigma_{\min} = 0.693$, degrades sharply to $17.44$~dB and LPIPS $0.214$ (all $p = 1.0\times10^{-4}$), while features from the second forward ($\sigma = 4.08$) remain at par ($p \geq 0.33$). This pattern is consistent with heavy noise acting as content suppression on the probe input. 

At high $\sigma$ the source object is largely erased from the input, so masked-position features are driven by surrounding context, whereas at low $\sigma$ the object remains present and reaches the masked features through convolutional receptive fields that attention suppression does not gate, and the retrieval then favors background patches resembling the object. We advance this as the consistent reading of the pattern rather than as an isolated causal claim.

The layer rows separate semantic depth from resolution. Mid-block features at native $16 \times 16$ are not significantly different from the default ($p \geq 0.62$ on every metric), while the encoder-side block at the same $32 \times 32$ native resolution as the default degrades strongly ($17.96$~dB, LPIPS $0.216$, $p \leq 0.002$), and the decoder block with native $64 \times 64$ resolution also loses ($18.42$~dB, LPIPS $0.207$; $p \leq 0.002$ on SSIM and LPIPS, $p = 0.14$ on PSNR). Correspondence quality therefore tracks where the features sit in the denoiser rather than their native resolution, and the bilinear resampling of the default is not the limiting factor: the block that needs no resampling performs worse than the resampled default.

The bank values are fixed to clean source latents, so the only quantity that can change during sampling is the correspondence itself. Re-computing the similarity weights at every generation forward from its own shared features, with no auxiliary pass and the evidence space unchanged, raises quality slightly ($19.02$ versus $18.75$~dB, LPIPS $0.190$ versus $0.194$), while the retrieved positions change substantially (mean top-$k$ overlap of $0.39$ to $0.43$ against the first-forward matching). 

The gain is consistent with the content-suppression reading above: features from the later forwards are harvested from a trajectory whose masked content has already been corrected toward background, so the probe input is object-free by prior correction rather than by noise. The default retains the single shared harvest as the minimal-intervention configuration, and we read the per-step update as the most immediate route toward the evidence headroom of Table~\ref{tab:appendix-evidence}.

\begin{table}[!htb]
\centering
\caption{\textbf{The shared noisy probe is the operative choice: near-clean inputs contaminate the bank, while attention suppression and the feature layer admit slack.} Each row rebuilds the background bank from a different feature source and regenerates with the identical four-forward pipeline; every variant harvests its features in a separate diagnostic pass, so the deployed method stays at the native count.}
\label{tab:appendix-probe}
\setlength{\tabcolsep}{10pt}
\renewcommand{\arraystretch}{1.15}
\begin{tabular}{lccc}
\toprule
Feature source & PSNR$\uparrow$ & SSIM$\uparrow$ & LPIPS$\downarrow$ \\
\midrule
\rowcolor{ourshl}Shared first forward (default) & 18.75 & 0.498 & 0.194 \\
No attention suppression & 18.72 & 0.497 & 0.194 \\
Near-clean input ($\sigma = 0.693$) & 17.44 & 0.474 & 0.214 \\
Second forward ($\sigma = 4.08$) & 18.87 & 0.499 & 0.193 \\
Encoder block (native $32^2$) & 17.96 & 0.475 & 0.216 \\
Mid block (native $16^2$) & 18.75 & 0.499 & 0.194 \\
Decoder block (native $64^2$) & 18.42 & 0.482 & 0.207 \\
\bottomrule
\end{tabular}
\end{table}

% ------------------------------------------------------------
\subsection{Number of Sampling Steps}
\label{app:abl-steps}

Table~\ref{tab:appendix-steps} runs SDXL-Turbo and SD-Turbo with one to four steps on OBER-Test. Each step count uses the default timesteps of the scheduler and the weights that Eq.~\ref{eq:update} assigns to its noise levels: $(0.954, 0)$ for two steps, $(0.986, 0.874, 0)$ for three, and $(0.993, 0.954, 0.794, 0)$ for four. All other settings are unchanged, and retrieval uses the features of the first forward at every step count. With one step, the only forward is the final one, whose weight is zero, so ReState coincides with attention forcing alone; this is the scope stated in Section~\ref{sec:discussion}. From two steps on, the correction adds $3.2$ to $3.6$~dB over attention forcing alone on both backbones, and three-step ReState is within $0.5$~dB of the four-step default.

\begin{table}[!htb]
\centering
\caption{\textbf{The direct correction is effective from two steps on.} Masked metrics on OBER-Test with one to four sampling steps. With one step, the only forward is the final forward, where no correction is applied, so ReState coincides with attention forcing alone.}
\label{tab:appendix-steps}
\small
\setlength{\tabcolsep}{7pt}
\renewcommand{\arraystretch}{1.15}
\begin{tabular}{cl ccc ccc}
\toprule
& & \multicolumn{3}{c}{SDXL-Turbo} & \multicolumn{3}{c}{SD-Turbo} \\
\cmidrule(lr){3-5}\cmidrule(lr){6-8}
Steps & Method & PSNR$\uparrow$ & SSIM$\uparrow$ & LPIPS$\downarrow$ & PSNR$\uparrow$ & SSIM$\uparrow$ & LPIPS$\downarrow$ \\
\midrule
1 & No intervention & 12.69 & 0.347 & 0.376 & 12.42 & 0.307 & 0.371 \\
  & Attention forcing only & 12.78 & 0.353 & 0.373 & 12.38 & 0.300 & 0.375 \\
\rowcolor{ourshl}  & \textbf{ReState (Ours)} & 12.78 & 0.353 & 0.373 & 12.38 & 0.300 & 0.375 \\
\midrule
2 & No intervention & 10.01 & 0.137 & 0.443 & 11.25 & 0.198 & 0.389 \\
  & Attention forcing only & 13.12 & 0.263 & 0.341 & 13.86 & 0.283 & 0.310 \\
\rowcolor{ourshl}  & \textbf{ReState (Ours)} & 16.38 & 0.429 & 0.259 & 17.06 & 0.440 & 0.214 \\
\midrule
3 & No intervention & 9.82 & 0.137 & 0.435 & 11.28 & 0.206 & 0.377 \\
  & Attention forcing only & 14.76 & 0.324 & 0.289 & 14.52 & 0.307 & 0.287 \\
\rowcolor{ourshl}  & \textbf{ReState (Ours)} & 18.27 & 0.489 & 0.204 & 18.07 & 0.479 & 0.187 \\
\midrule
4 & No intervention & 9.86 & 0.133 & 0.435 & 11.27 & 0.203 & 0.380 \\
  & Attention forcing only & 15.50 & 0.348 & 0.269 & 14.74 & 0.316 & 0.283 \\
\rowcolor{ourshl}  & \textbf{ReState (Ours)} & 18.75 & 0.498 & 0.194 & 18.19 & 0.485 & 0.186 \\
\bottomrule
\end{tabular}
\end{table}

% ------------------------------------------------------------
\subsection{Combination with Other Attention Rules}
\label{app:abl-combination}

The direct correction acts on the clean prediction and does not depend on how attention is modified, so it can be combined with other attention rules. Table~\ref{tab:appendix-combination} replaces the hard key masking of ReState with the SARG of Attentive Eraser and with the adaptive suppression of AdaEraser, each without and with the direct correction; the retrieved $B$ and all other settings are unchanged. The correction improves every attention rule, by $3.2$~dB for the re-tuned SARG and by $6.1$~dB for adaptive suppression. With the official SARG scale, the guidance dominates the four-step trajectory, and the result stays low with the correction as well. None of the combinations exceeds the default, which uses hard key masking and needs no auxiliary forward, whereas SARG and adaptive suppression each double the effective forwards (Table~\ref{tab:appendix-auxpass}).

\begin{table}[H]
\centering
\caption{\textbf{The direct correction improves every attention rule, and the default remains the best.} Each attention rule without and with the direct correction of ReState on OBER-Test with SDXL-Turbo. The first row with the correction is the default ReState.}
\label{tab:appendix-combination}
\small
\setlength{\tabcolsep}{5pt}
\renewcommand{\arraystretch}{1.15}
\begin{tabular}{l ccc ccc}
\toprule
& \multicolumn{3}{c}{Without direct correction} & \multicolumn{3}{c}{With direct correction} \\
\cmidrule(lr){2-4}\cmidrule(lr){5-7}
Attention rule & PSNR$\uparrow$ & SSIM$\uparrow$ & LPIPS$\downarrow$ & PSNR$\uparrow$ & SSIM$\uparrow$ & LPIPS$\downarrow$ \\
\midrule
\rowcolor{ourshl}Hard key masking (DesignEdit; default) & 15.50 & 0.348 & 0.269 & \textbf{18.75} & \textbf{0.498} & \textbf{0.194} \\
SARG, re-tuned scale (Attentive Eraser) & 15.34 & 0.335 & 0.277 & 18.57 & 0.495 & 0.201 \\
SARG, official scale (Attentive Eraser) & 7.09 & 0.046 & 0.514 & 9.20 & 0.149 & 0.453 \\
Adaptive suppression (AdaEraser) & 11.71 & 0.205 & 0.372 & 17.79 & 0.483 & 0.223 \\
\bottomrule
\end{tabular}
\end{table}

% ============================================================
\section{Additional Results}
\label{app:additional}

\subsection{Full Quantitative Tables}
\label{app:quant}

Table~\ref{tab:appendix-full} collects the full quantitative comparison on both benchmarks: the training-free block repeats the SDXL-Turbo rows of Table~\ref{tab:main}, and the lower block adds the training- and optimization-based methods, each run from its released checkpoint under the official default inference settings (Appendix~\ref{app:protocol}). SDXL-Inpainting and FLUX.1-Fill operate and are evaluated at $1024^2$, and all remaining methods at $512^2$. Latency is the median per-image end-to-end wall-clock on OBER-Test and VRAM is the peak allocation, both measured on a single A100 80GB for every row. 

OSOR \citep{osor2026} fine-tunes SDXL-Inpainting and FLUX.1-Fill on CORNE, a set of 280K removal pairs curated from single-edit instruction triplets, and additionally refines the FLUX variant for 1K steps on the captured subset of OBER. This refinement is described only for the FLUX variant; because both variants share the same training pipeline, we conservatively treat the same refinement as possible for the SDXL-Inpainting release as well, so out-of-distribution evaluation on OBER-Test is not guaranteed for either variant, and OSOR is excluded from the matched comparison of Table~\ref{tab:appendix-full}. 

For reference, its SDXL-Inpainting variant reaches $23.19$~dB masked PSNR on OBER-Test and $18.77$~dB on RORD at $0.40$~s and $7.2$~GiB under the identical protocol. The training data of every training-based comparison is classified in Table~\ref{tab:appendix-training-data}. The block also includes LaMa~\citep{lama2022}, an inpainting network trained without diffusion.

\begin{table}[!htb]
\centering
\caption{\textbf{Without any training, ReState remains competitive with the training-based methods, at a lower latency than every diffusion-based one.} Latency and VRAM are measured as described in the text; the training-free rows repeat the SDXL-Turbo block of Table~\ref{tab:main}.}
\label{tab:appendix-full}
\small
\setlength{\tabcolsep}{4.5pt}
\renewcommand{\arraystretch}{1.12}
\begin{tabular}{l ccc ccc cc}
\toprule
& \multicolumn{3}{c}{OBER-Test} & \multicolumn{3}{c}{RORD} & & \\
\cmidrule(lr){2-4}\cmidrule(lr){5-7}
Method & PSNR$\uparrow$ & SSIM$\uparrow$ & LPIPS$\downarrow$ & PSNR$\uparrow$ & SSIM$\uparrow$ & LPIPS$\downarrow$ & Lat.\,(s) & VRAM \\
\midrule
\multicolumn{9}{l}{\textit{Training/Optimization-based methods}} \\
LaMa\venue{WACV'22} & 17.49 & 0.540 & 0.169 & 17.30 & 0.494 & 0.210 & 0.052 & 0.42GB \\  % PENDING: latency and VRAM
SDXL-Inpainting & 11.65 & 0.365 & 0.391 & 10.57 & 0.348 & 0.403 & 3.75 & 9.0GB \\
FLUX.1-Fill & 11.04 & 0.376 & 0.405 & 12.33 & 0.415 & 0.366 & 15.9 & 33.8GB \\
PowerPaint\venue{ECCV'24} & 14.89 & 0.463 & 0.255 & 14.91 & 0.438 & 0.282 & 3.86 & 5.1GB \\
CLIPAway\venue{NeurIPS'24} & 15.71 & 0.426 & 0.246 & 14.84 & 0.406 & 0.268 & 7.83 & 9.8GB \\
EntityErasure\venue{CVPR'25} & 15.70 & 0.496 & 0.222 & 15.70 & 0.460 & 0.247 & 3.21 & 12.5GB \\
SmartEraser\venue{CVPR'25} & 18.67 & 0.510 & 0.170 & 17.81 & 0.463 & 0.193 & 6.12 & 7.7GB \\
EraseLoRA\venue{ECCV'26} & 18.78 & 0.481 & 0.175 & 16.77 & 0.433 & 0.237 & 198s & 33.5GB \\
\midrule
\multicolumn{9}{l}{\textit{Training-free methods}} \\
SDXL-Turbo (4 steps) & 9.86 & 0.133 & 0.435 & 7.55 & 0.074 & 0.516 & 0.247 & 7.1GB \\
$+$ Attentive Eraser\venue{AAAI'25} & 7.09 & 0.046 & 0.514 & 7.35 & 0.065 & 0.538 & 0.338 & 7.1GB \\
$+$ DesignEdit\venue{AAAI'25} & 15.50 & 0.348 & 0.269 & 14.34 & 0.326 & 0.334 & 0.281 & 7.1GB \\
$+$ AdaEraser\venue{ICML'26} & 11.71 & 0.205 & 0.372 & 9.47 & 0.122 & 0.462 & 1.145 & 7.1GB \\
\rowcolor{ourshl}$+$ \textbf{ReState (Ours)} & 18.75 & 0.498 & 0.194 & 17.21 & 0.451 & 0.273 & 0.373 & 7.1GB \\
\bottomrule
\end{tabular}
\end{table}

% ------------------------------------------------------------
\subsection{Full Removal on Additional Distilled Backbones}
\label{app:additional-backbones}

Table~\ref{tab:appendix-backbones} extends the full removal comparison of Table~\ref{tab:main} to SDXL-Lightning and Hyper-SDXL under the identical frozen configuration: the same $\lambda = 8$, $r = 0.01$, probe layer, and schedule keying, with no per-backbone tuning. On both backbones and both benchmarks ReState improves over DesignEdit, the strongest training-free baseline of Table~\ref{tab:main}, on every metric, by $+1.46$ to $+2.39$~dB masked PSNR, and over the no-intervention baseline by $6.9$ to $7.3$~dB; every difference is significant at $p = 1.0\times10^{-4}$ under the protocol of Appendix~\ref{app:stats}.

The margins over attention forcing are smaller than on the Turbo family ($+3.25$~dB on SDXL-Turbo), in the direction the intervention analysis predicts. On Lightning and Hyper-SDXL the first-forward terminal effect of attention forcing is already positive ($+0.034$ and $+0.030$, Table~\ref{tab:appendix-firstforward}) rather than not significantly different from zero as on the Turbo family, so attention forcing recovers part of the removal on its own and the relative advantage of the direct correction narrows. On both benchmarks the ordering of the four ReState margins follows the ordering of that first-forward effect across the four distilled backbones; we read this as consistency with the diagnosis rather than as a causal prediction test.

The removal pipeline is unchanged; the two backbones operate and are evaluated at $1024^2$ (Appendix~\ref{app:stats}), so absolute metric levels are not comparable to the $512^2$ backbones and comparisons are within each backbone, as in Table~\ref{tab:mechanism}. Attentive Eraser and AdaEraser are omitted from this extension: both already fall below DesignEdit or the no-intervention baseline on the Turbo backbones of Table~\ref{tab:main}.

\begin{table}[H]
\centering
\caption{\textbf{The single frozen configuration transfers: ReState improves every masked metric on two further distilled backbones.} Full removal at $1024^2$ with the identical hyperparameters as Table~\ref{tab:main} and no per-backbone tuning; rows with $+$ apply the method on the backbone of their block. Bold is best per column within a block; test statistics are given in the text.}
\label{tab:appendix-backbones}
\setlength{\tabcolsep}{7pt}
\renewcommand{\arraystretch}{1.15}
\begin{tabular}{l ccc ccc}
\toprule
& \multicolumn{3}{c}{OBER-Test} & \multicolumn{3}{c}{RORD} \\
\cmidrule(lr){2-4}\cmidrule(lr){5-7}
Method & PSNR$\uparrow$ & SSIM$\uparrow$ & LPIPS$\downarrow$ & PSNR$\uparrow$ & SSIM$\uparrow$ & LPIPS$\downarrow$ \\
\midrule
SDXL-Lightning (4 steps) & 10.51 & 0.393 & 0.424 & 9.60 & 0.318 & 0.445 \\
$+$ DesignEdit\venue{AAAI'25} & 15.95 & 0.517 & 0.296 & 15.08 & 0.475 & 0.325 \\
\rowcolor{ourshl}$+$ \textbf{ReState (Ours)} & \textbf{17.51} & \textbf{0.578} & \textbf{0.271} & \textbf{16.54} & \textbf{0.528} & \textbf{0.308} \\
\midrule
Hyper-SDXL (4 steps) & 10.21 & 0.346 & 0.434 & 9.21 & 0.273 & 0.455 \\
$+$ DesignEdit\venue{AAAI'25} & 15.20 & 0.487 & 0.315 & 14.14 & 0.434 & 0.358 \\
\rowcolor{ourshl}$+$ \textbf{ReState (Ours)} & \textbf{17.49} & \textbf{0.574} & \textbf{0.273} & \textbf{16.53} & \textbf{0.527} & \textbf{0.315} \\
\bottomrule
\end{tabular}
\end{table}

\subsection{Multi-Step Baseline Comparison}
\label{app:multistep-baseline}

Table~\ref{tab:appendix-multistep50} runs the attention rules of the baselines of Table~\ref{tab:main} at 50 steps inside our shared pipeline, under the identical reference configuration as Table~\ref{tab:appendix-placement} (same guidance, seed, and background anchoring). The rows are named by their attention rule, because they are not runs of the official implementations, which Table~\ref{tab:runtime} reports. The attention rules that degrade under the four-step port operate correctly at 50 steps: the SARG of Attentive Eraser reaches $18.79$~dB against $7.09$~dB in Table~\ref{tab:main}, and the adaptive suppression of AdaEraser $15.66$ against $11.71$~dB. The few-step failures of Section~\ref{sec:experiments} are therefore consistent with a regime effect rather than a port artifact, and the recovery of SARG is exactly the calibration argument of the main text: its amplification is calibrated for 50-step sampling and acts as intended there.

Against this 50-step field, ReState at its native four steps ($18.75$~dB, SSIM $0.498$, LPIPS $0.194$) exceeds every 50-step attention configuration on SSIM and LPIPS, and sits within $0.04$~dB of the best attention-configuration PSNR (SARG), at $6.7\times$ lower latency than hard key masking at 50 steps ($0.373$~s against $2.50$~s).

Transferring ReState unchanged to the 50-step schedule under-performs ($17.57$~dB, LPIPS $0.231$): the $\sigma_{t-1}$-keyed weight keeps the correction strong across the many low-noise forwards, so the smooth evidence $B$ is imprinted repeatedly in the range where fine texture is synthesized. 

Restricting the correction to transitions with $\sigma_{t-1} \geq \theta$ recovers quality: over $\theta = \infty, 4.08, 1.61, 0.693, 0.1, 0$ masked PSNR moves as $17.86, 18.82, 18.83, 18.21, 17.57, 17.57$, a plateau at $\theta = 4.08$ to $1.61$ that recovers past the naive transfer to the level of the strongest 50-step attention configuration. We report the sweep in full; the windowed row of the table is the plateau configuration and is a diagnostic, not part of the method. Notably, its LPIPS ($0.181$) improves on the four-step default ($0.194$), indicating that part of the under-textured residual of Appendix~\ref{app:frequency} reflects the limited free-synthesis budget of the coarse schedule rather than the correction itself.

This behavior refines, and is predicted by, the schedule principle of Appendix~\ref{app:abl-timing}: what must be preserved is uncorrected synthesis freedom before terminal rendering. On the native coarse schedule the closed-form boundary $\omega \to 0$ provides exactly this, because a single low-noise forward remains; on a dense schedule the boundary condition and the freedom requirement separate, and an explicit window substitutes for it. This is one more reason the operating point of the method is the native few-step schedule.

\begin{table}[H]
\centering
\caption{\textbf{Attention forcing recovers in its 50-step regime, and 4-step ReState matches the strongest 50-step configuration at $6.7\times$ lower latency than 50-step hard key masking.} OBER-Test under the 50-step reference configuration; each row applies an attention rule inside our shared pipeline and is not a run of the official implementation (Table~\ref{tab:runtime}); the 4-step port numbers of the same methods are in Table~\ref{tab:main}, and the correction-window sweep is given in the text.}
\label{tab:appendix-multistep50}
\setlength{\tabcolsep}{10pt}
\renewcommand{\arraystretch}{1.15}
\begin{tabular}{lccc}
\toprule
Configuration at 50 steps & PSNR$\uparrow$ & SSIM$\uparrow$ & LPIPS$\downarrow$ \\
\midrule
$+$ Hard key masking (DesignEdit; reference) & 17.86 & 0.441 & 0.209 \\
$+$ SARG (Attentive Eraser) & 18.79 & 0.462 & 0.225 \\
$+$ Adaptive suppression (AdaEraser) & 15.66 & 0.399 & 0.274 \\
$+$ ReState, every forward corrected & 17.57 & 0.479 & 0.231 \\
$+$ ReState, corrections at $\sigma_{t-1} \geq 1.61$ & 18.83 & 0.485 & 0.181 \\
\midrule
\rowcolor{ourshl}ReState on SDXL-Turbo (4 steps; Table~\ref{tab:main}) & 18.75 & 0.498 & 0.194 \\
\bottomrule
\end{tabular}
\end{table}

\textbf{Official implementations.} Table~\ref{tab:appendix-official} runs the public code of DesignEdit and Attentive Eraser on SDXL with all of their components (Appendix~\ref{app:ports}) at their native step counts, on OBER-Test, RORD, and MULAN (Appendix~\ref{app:mulan}), at the $512^2$ resolution of the benchmarks. Four-step ReState is better on all three masked metrics on every benchmark, at $14.2\times$ and $12.4\times$ lower latency, respectively (Table~\ref{tab:runtime}).

\begin{table}[H]
\centering
\caption{\textbf{Four-step ReState is higher than the official implementations of DesignEdit and Attentive Eraser at their native step counts on all three benchmarks.} Masked metrics. DesignEdit uses 50 inversion and 50 sampling steps, and Attentive Eraser 40 steps.}
\label{tab:appendix-official}
\small
\setlength{\tabcolsep}{3pt}
\renewcommand{\arraystretch}{1.15}
\begin{tabular}{l ccc ccc ccc}
\toprule
& \multicolumn{3}{c}{OBER-Test} & \multicolumn{3}{c}{RORD} & \multicolumn{3}{c}{MULAN} \\
\cmidrule(lr){2-4}\cmidrule(lr){5-7}\cmidrule(lr){8-10}
Method & PSNR$\uparrow$ & SSIM$\uparrow$ & LPIPS$\downarrow$ & PSNR$\uparrow$ & SSIM$\uparrow$ & LPIPS$\downarrow$ & PSNR$\uparrow$ & SSIM$\uparrow$ & LPIPS$\downarrow$ \\
\midrule
SDXL $+$ DesignEdit & 18.51 & 0.483 & 0.206 & 16.52 & 0.424 & 0.288 & 15.85 & 0.394 & 0.302 \\
SDXL $+$ Attentive Eraser & 18.19 & 0.446 & 0.245 & 15.71 & 0.362 & 0.356 & 15.17 & 0.346 & 0.319 \\
\rowcolor{ourshl}\textbf{ReState (Ours)}, 4 steps & \textbf{18.75} & \textbf{0.498} & \textbf{0.194} & \textbf{17.21} & \textbf{0.451} & \textbf{0.273} & \textbf{16.63} & \textbf{0.457} & \textbf{0.268} \\
\bottomrule
\end{tabular}
\end{table}

% ------------------------------------------------------------
\subsection{Seed Robustness}
\label{app:seeds}

All main-text numbers use seed $0$, which fixes the initial noise shared by every method. To measure the dependence on this choice, every training-free row of Table~\ref{tab:main} is re-run with seeds $\{0, 1, 2\}$ on both main backbones and both benchmarks; the seed-$0$ runs reproduce Table~\ref{tab:main} exactly. Table~\ref{tab:appendix-seeds} reports means and standard deviations over the three seeds. The ranking is unchanged in every cell: ReState is first on all three masked metrics for every seed, backbone, and benchmark (36 of 36 cells), and its paired advantage over DesignEdit is individually significant for each seed (36 tests, all $p = 1.0\times10^{-4}$). ReState also varies least across seeds: its masked-PSNR standard deviation is $0.02$ to $0.21$~dB, versus $0.31$ to $0.81$~dB for DesignEdit.

\begin{table}[!htb]
\centering
\caption{\textbf{The ranking is stable across seeds: ReState is first on every masked metric for every seed, backbone, and benchmark.} Mean $\pm$ sample standard deviation over seeds $\{0, 1, 2\}$.}
\label{tab:appendix-seeds}
\footnotesize
\setlength{\tabcolsep}{2.5pt}
\renewcommand{\arraystretch}{1.15}
\resizebox{\linewidth}{!}{%
\begin{tabular}{l ccc ccc}
\toprule
& \multicolumn{3}{c}{OBER-Test} & \multicolumn{3}{c}{RORD} \\
\cmidrule(lr){2-4}\cmidrule(lr){5-7}
Method & PSNR$\uparrow$ & SSIM$\uparrow$ & LPIPS$\downarrow$ & PSNR$\uparrow$ & SSIM$\uparrow$ & LPIPS$\downarrow$ \\
\midrule
SD-Turbo (4 steps) & $11.17 \pm 0.69$ & $0.198 \pm 0.007$ & $0.379 \pm 0.008$ & $9.99 \pm 0.30$ & $0.168 \pm 0.012$ & $0.436 \pm 0.002$ \\
$+$ Attentive Eraser\venue{AAAI'25} & $7.23 \pm 0.15$ & $0.064 \pm 0.003$ & $0.517 \pm 0.003$ & $7.28 \pm 0.06$ & $0.057 \pm 0.003$ & $0.554 \pm 0.001$ \\
$+$ DesignEdit\venue{AAAI'25} & $14.73 \pm 0.81$ & $0.319 \pm 0.010$ & $0.281 \pm 0.016$ & $13.51 \pm 0.31$ & $0.270 \pm 0.006$ & $0.354 \pm 0.007$ \\
$+$ AdaEraser\venue{ICML'26} & $12.28 \pm 0.85$ & $0.236 \pm 0.009$ & $0.346 \pm 0.013$ & $10.56 \pm 0.36$ & $0.181 \pm 0.008$ & $0.420 \pm 0.004$ \\
\rowcolor{ourshl}$+$ \textbf{ReState (Ours)} & $\mathbf{18.02 \pm 0.21}$ & $\mathbf{0.485 \pm 0.004}$ & $\mathbf{0.189 \pm 0.004}$ & $\mathbf{16.46 \pm 0.11}$ & $\mathbf{0.438 \pm 0.002}$ & $\mathbf{0.281 \pm 0.002}$ \\
\midrule
SDXL-Turbo (4 steps) & $9.91 \pm 0.67$ & $0.138 \pm 0.032$ & $0.431 \pm 0.025$ & $7.52 \pm 0.25$ & $0.065 \pm 0.017$ & $0.519 \pm 0.007$ \\
$+$ Attentive Eraser\venue{AAAI'25} & $7.14 \pm 0.09$ & $0.047 \pm 0.005$ & $0.514 \pm 0.002$ & $7.39 \pm 0.07$ & $0.060 \pm 0.005$ & $0.538 \pm 0.001$ \\
$+$ DesignEdit\venue{AAAI'25} & $15.39 \pm 0.69$ & $0.344 \pm 0.031$ & $0.272 \pm 0.016$ & $14.16 \pm 0.34$ & $0.308 \pm 0.035$ & $0.340 \pm 0.011$ \\
$+$ AdaEraser\venue{ICML'26} & $11.66 \pm 0.48$ & $0.203 \pm 0.026$ & $0.376 \pm 0.014$ & $9.25 \pm 0.43$ & $0.108 \pm 0.027$ & $0.472 \pm 0.014$ \\
\rowcolor{ourshl}$+$ \textbf{ReState (Ours)} & $\mathbf{18.77 \pm 0.13}$ & $\mathbf{0.496 \pm 0.005}$ & $\mathbf{0.192 \pm 0.002}$ & $\mathbf{17.19 \pm 0.02}$ & $\mathbf{0.447 \pm 0.004}$ & $\mathbf{0.271 \pm 0.002}$ \\
\bottomrule
\end{tabular}}
\end{table}

% ============================================================
\section{Additional Evaluation}
\label{app:evaluation}

\subsection{Perceptual and Distribution Metrics}
\label{app:metrics}

Masked PSNR and SSIM compare pixels and can favor smooth fills. We therefore add three metrics and one control on OBER-Test. DISTS~\citep{dists2020} is a full-reference perceptual metric that tolerates texture resampling; it is computed on the same composite crop as masked LPIPS (Eq.~\ref{eq:app-lpips}). CMMD~\citep{cmmd2024} is a distribution distance on CLIP embeddings, computed between the set of composite crops and the set of ground-truth crops. ReMOVE~\citep{remove2024} is a reference-free removal metric that compares the features of the masked region with those of its surroundings in the output; we report its cropped variant. As a control, we score a ground truth whose masked region is blurred with a Gaussian kernel of $\sigma = 8$ pixels, $X = M \odot \mathrm{blur}(G) + (1-M) \odot G$, as if it were a removal result.

The blurred ground truth reaches $24.07$~dB and an SSIM of $0.604$, above every method in Table~\ref{tab:appendix-metrics}, so masked PSNR and SSIM do reward smooth fills. Its LPIPS, CMMD, and ReMOVE are $0.221$, $1.639$, and $0.656$, all worse than those of ReState, so these metrics do not reward smoothness in the same way. On the methods, ReState is the best training-free method on LPIPS, DISTS, CMMD, and ReMOVE on both backbones (Table~\ref{tab:appendix-metrics}). Table~\ref{tab:appendix-metrics} also lists the training-based methods of Table~\ref{tab:appendix-full} for reference; DISTS, CMMD, and ReMOVE are computed against the public release of OBER-Test. ReState is close to the best of them on DISTS and ReMOVE, whereas all of them are ahead on CMMD, which is consistent with the smoother fill of Appendix~\ref{app:frequency}. The fill of ReState is still smoother than the real background, with $0.51\times$ the high-band energy of the ground truth (Appendix~\ref{app:frequency}); we therefore read these results as evidence that the gain is not explained by smoothing alone, and not as evidence of full texture recovery. The qualitative figures added in this revision (Figures~\ref{fig:appendix-effect}, \ref{fig:appendix-editing}, \ref{fig:appendix-dit}, and~\ref{fig:appendix-mulan}) show the outputs as generated, without contour overlays; the mask is given as a separate column.

\begin{table}[!htb]
\centering
\caption{\textbf{ReState is the best training-free method on perceptual and distribution metrics.} Additional metrics on OBER-Test. Bold marks the best method of each backbone.}
\label{tab:appendix-metrics}
\small
\setlength{\tabcolsep}{6pt}
\renewcommand{\arraystretch}{1.15}
\begin{tabular}{l cccccc}
\toprule
Method & PSNR$\uparrow$ & SSIM$\uparrow$ & LPIPS$\downarrow$ & DISTS$\downarrow$ & CMMD$\downarrow$ & ReMOVE$\uparrow$ \\
\midrule
SD-Turbo (4 steps) & 11.27 & 0.203 & 0.380 & 0.286 & 1.720 & 0.741 \\
$+$ Attentive Eraser\venue{AAAI'25} & 7.28 & 0.063 & 0.516 & 0.367 & 2.670 & 0.645 \\
$+$ Attentive Eraser, re-tuned & 14.95 & 0.329 & 0.277 & 0.207 & 1.526 & 0.850 \\
$+$ DesignEdit\venue{AAAI'25} & 14.74 & 0.316 & 0.283 & 0.208 & 1.468 & 0.848 \\
$+$ AdaEraser\venue{ICML'26} & 12.27 & 0.237 & 0.347 & 0.267 & 1.580 & 0.775 \\
\rowcolor{ourshl}$+$ \textbf{ReState (Ours)} & \textbf{18.19} & \textbf{0.485} & \textbf{0.186} & \textbf{0.119} & \textbf{1.088} & \textbf{0.936} \\
\midrule
SDXL-Turbo (4 steps) & 9.86 & 0.133 & 0.435 & 0.318 & 1.913 & 0.678 \\
$+$ Attentive Eraser\venue{AAAI'25} & 7.09 & 0.046 & 0.514 & 0.387 & 2.747 & 0.684 \\
$+$ Attentive Eraser, re-tuned & 15.34 & 0.335 & 0.277 & 0.209 & 1.655 & 0.833 \\
$+$ DesignEdit\venue{AAAI'25} & 15.50 & 0.348 & 0.269 & 0.199 & 1.582 & 0.849 \\
$+$ AdaEraser\venue{ICML'26} & 11.71 & 0.205 & 0.372 & 0.292 & 1.722 & 0.709 \\
\rowcolor{ourshl}$+$ \textbf{ReState (Ours)} & \textbf{18.75} & \textbf{0.498} & \textbf{0.194} & \textbf{0.129} & \textbf{1.358} & \textbf{0.929} \\
\midrule
\multicolumn{7}{l}{\textit{Training/Optimization-based methods}} \\
LaMa\venue{WACV'22} & 17.49 & 0.540 & 0.169 & 0.120 & 0.760 & 0.936 \\
SDXL-Inpainting & 11.65 & 0.365 & 0.391 & 0.252 & 1.316 & 0.669 \\
FLUX.1-Fill & 11.04 & 0.376 & 0.405 & 0.254 & 0.932 & 0.692 \\
PowerPaint\venue{ECCV'24} & 14.89 & 0.463 & 0.255 & 0.173 & 0.988 & 0.865 \\
CLIPAway\venue{NeurIPS'24} & 15.71 & 0.426 & 0.246 & 0.164 & 1.051 & 0.883 \\
EntityErasure\venue{CVPR'25} & 15.70 & 0.496 & 0.222 & 0.153 & 0.868 & 0.906 \\
SmartEraser\venue{CVPR'25} & 18.67 & 0.510 & 0.170 & 0.120 & 0.647 & 0.920 \\
EraseLoRA\venue{ECCV'26} & 18.78 & 0.481 & 0.175 & 0.115 & 1.095 & 0.931 \\

\bottomrule
\end{tabular}
\end{table}

% ------------------------------------------------------------
\subsection{Object-Only Benchmark}
\label{app:mulan}

OBER-Test and RORD are captured pairs, so their ground truth also lacks the shadows and reflections of the object. To evaluate the removal of the object alone, we use MULAN~\citep{mulan2024}, whose ground truth removes only the object layer (Appendix~\ref{app:protocol}); it is also the benchmark of AdaEraser~\citep{adaeraser2026}. The 500 pairs are drawn uniformly at random with a fixed seed from the 16{,}034 MULAN-COCO removal pairs released with RORem~\citep{rorem2025}. All pairs are $512^2$, so no resizing is applied, and every method runs with the settings of Table~\ref{tab:main} without re-tuning.

Besides the masked metrics, Table~\ref{tab:appendix-mulan} reports whole-image metrics over all pixels. ReState is the best training-free method on every metric, $2.6$~dB above the strongest baseline in masked PSNR and $1.6$~dB in whole-image PSNR. Among the training-based methods, ReState is higher than SmartEraser~\citep{smarteraser2025} in masked and whole-image PSNR and SSIM. Compared with LaMa~\citep{lama2022}, a trained inpainting network, ReState is higher in masked PSNR and SSIM at a similar masked LPIPS and whole-image PSNR, while LaMa is better in DISTS and in the whole-image SSIM and LPIPS. Figure~\ref{fig:appendix-mulan} shows qualitative results.

\begin{table}[!htb]
\centering
\caption{\textbf{On an object-only benchmark, ReState is the best training-free method in both masked and whole-image metrics.} MULAN, 500 pairs. Bold marks the best training-free method.}
\label{tab:appendix-mulan}
\small
\setlength{\tabcolsep}{4.5pt}
\renewcommand{\arraystretch}{1.15}
\begin{tabular}{l cccc ccc}
\toprule
& \multicolumn{4}{c}{Masked} & \multicolumn{3}{c}{Whole image} \\
\cmidrule(lr){2-5}\cmidrule(lr){6-8}
Method & PSNR$\uparrow$ & SSIM$\uparrow$ & LPIPS$\downarrow$ & DISTS$\downarrow$ & PSNR$\uparrow$ & SSIM$\uparrow$ & LPIPS$\downarrow$ \\
\midrule
\multicolumn{8}{l}{\textit{Training-based methods}} \\
LaMa\venue{WACV'22} & 15.46 & 0.380 & 0.272 & 0.146 & 21.28 & 0.704 & 0.244 \\
SmartEraser\venue{CVPR'25} & 14.44 & 0.347 & 0.275 & 0.167 & 19.33 & 0.582 & 0.280 \\
\midrule
\multicolumn{8}{l}{\textit{Training-free methods}} \\
SDXL-Turbo (4 steps) & 8.62 & 0.119 & 0.430 & 0.337 & 15.23 & 0.536 & 0.362 \\
$+$ Attentive Eraser\venue{AAAI'25} & 7.21 & 0.071 & 0.467 & 0.377 & 13.43 & 0.478 & 0.598 \\
$+$ Attentive Eraser, re-tuned & 14.00 & 0.314 & 0.315 & 0.210 & 19.53 & 0.601 & 0.296 \\
$+$ DesignEdit\venue{AAAI'25} & 14.07 & 0.325 & 0.312 & 0.206 & 19.57 & 0.604 & 0.295 \\
$+$ AdaEraser\venue{ICML'26} & 10.15 & 0.164 & 0.394 & 0.304 & 16.46 & 0.545 & 0.345 \\
\rowcolor{ourshl}$+$ \textbf{ReState (Ours)} & \textbf{16.63} & \textbf{0.457} & \textbf{0.268} & \textbf{0.156} & \textbf{21.20} & \textbf{0.639} & \textbf{0.279} \\
\bottomrule
\end{tabular}
\end{table}

% ------------------------------------------------------------
\subsection{Object-Effect Masks and Whole-Image Evaluation}
\label{app:effect}

The main experiments generate and evaluate with the object mask, so shadows and reflections outside it are preserved and are not counted by the masked metrics. Here, each method is run twice, once with the object mask and once with the object-effect mask as input, without re-tuning. Both outputs are scored on the same regions: the object region, the object-effect region, the effect-only region (the object-effect mask without the object mask), and the whole image. OBER-Test provides both masks. For RORD, this experiment uses RORD-Val, whose frames have the coarse object-effect mask of RORD and an object mask refined by the authors of OBER~\citep{zhao2026objectclear} (Appendix~\ref{app:protocol}).

On OBER-Test (Table~\ref{tab:appendix-effect-ober}), the object-effect mask improves ReState by $6.5$~dB on the effect-only region, $3.4$~dB on the object-effect region, and $1.5$~dB on the whole image, and ReState is the best method in every region under both input masks. On RORD-Val (Table~\ref{tab:appendix-effect-rord}), the effect region is much larger: the effect-only region covers a median of $11.0\%$ of the frame, against $2.5\%$ on OBER-Test. Many of these masks are coarse polygons that also cover background which the object does not affect (Figure~\ref{fig:appendix-effect}). 
In the median image, the input already matches the ground truth within $8/255$ on $58\%$ of the effect-only pixels, against $9\%$ on OBER-Test. Left untouched, this region is already close to the ground truth ($20.90$~dB with the object mask, against $15.59$~dB on OBER-Test), so regenerating it lowers the effect-only and whole-image scores of all five methods. ReState remains the best method under both input masks. Figure~\ref{fig:appendix-effect} shows examples of the two masks and the corresponding outputs on both sets.

\begin{table}[!htb]
\centering
\caption{\textbf{With the object-effect mask as input, the unchanged ReState also removes the effects on OBER-Test.} Object mask versus object-effect mask as input (SDXL-Turbo). The three region columns are masked PSNR$\uparrow$ on that region.}
\label{tab:appendix-effect-ober}
\small
\setlength{\tabcolsep}{5pt}
\renewcommand{\arraystretch}{1.15}
\begin{tabular}{ll ccc cc}
\toprule
& & \multicolumn{3}{c}{Masked PSNR$\uparrow$ on the region} & \multicolumn{2}{c}{Whole image} \\
\cmidrule(lr){3-5}\cmidrule(lr){6-7}
Method & Input mask & Object & Object-effect & Effect-only & PSNR$\uparrow$ & LPIPS$\downarrow$ \\
\midrule
SDXL-Turbo (4 steps) & Object & 9.86 & 10.52 & 12.73 & 19.91 & 0.178 \\
 & Object-effect & 9.29 & 9.51 & 10.11 & 18.71 & 0.215 \\
$+$ Attentive Eraser\venue{AAAI'25} & Object & 7.09 & 8.03 & 10.42 & 16.75 & 0.478 \\
 & Object-effect & 7.15 & 7.26 & 7.57 & 15.76 & 0.530 \\
$+$ DesignEdit\venue{AAAI'25} & Object & 15.50 & 14.76 & 14.90 & 23.24 & 0.139 \\
 & Object-effect & 16.31 & 16.74 & 18.62 & 24.06 & 0.143 \\
$+$ AdaEraser\venue{ICML'26} & Object & 11.71 & 12.05 & 13.71 & 21.23 & 0.162 \\
 & Object-effect & 11.18 & 11.49 & 12.56 & 20.35 & 0.194 \\
\rowcolor{ourshl}$+$ \textbf{ReState (Ours)} & Object & 18.75 & 16.60 & 15.59 & 24.49 & 0.125 \\
\rowcolor{ourshl} & Object-effect & \textbf{19.68} & \textbf{19.96} & \textbf{22.11} & \textbf{25.95} & \textbf{0.123} \\
\bottomrule
\end{tabular}
\end{table}

\begin{table}[!htb]
\centering
\caption{\textbf{On RORD-Val, where the effect regions are large, ReState remains the best method under both input masks.} Same layout as Table~\ref{tab:appendix-effect-ober}. The object-effect masks of this set are coarse and also cover unaffected background; the object masks were refined by the authors of OBER.}
\label{tab:appendix-effect-rord}
\small
\setlength{\tabcolsep}{5pt}
\renewcommand{\arraystretch}{1.15}
\begin{tabular}{ll ccc cc}
\toprule
& & \multicolumn{3}{c}{Masked PSNR$\uparrow$ on the region} & \multicolumn{2}{c}{Whole image} \\
\cmidrule(lr){3-5}\cmidrule(lr){6-7}
Method & Input mask & Object & Object-effect & Effect-only & PSNR$\uparrow$ & LPIPS$\downarrow$ \\
\midrule
SDXL-Turbo (4 steps) & Object & 8.04 & 10.26 & 15.60 & 14.07 & 0.366 \\
 & Object-effect & 7.26 & 7.42 & 7.73 & 11.26 & 0.526 \\
$+$ Attentive Eraser\venue{AAAI'25} & Object & 7.28 & 9.20 & 12.73 & 13.00 & 0.679 \\
 & Object-effect & 7.49 & 7.58 & 7.79 & 11.17 & 0.772 \\
$+$ DesignEdit\venue{AAAI'25} & Object & 14.54 & 16.19 & 19.86 & 19.36 & 0.263 \\
 & Object-effect & 13.61 & 13.95 & 14.71 & 17.18 & 0.387 \\
$+$ AdaEraser\venue{ICML'26} & Object & 9.94 & 12.04 & 17.17 & 15.65 & 0.340 \\
 & Object-effect & 8.18 & 8.40 & 8.89 & 12.11 & 0.511 \\
\rowcolor{ourshl}$+$ \textbf{ReState (Ours)} & Object & \textbf{17.12} & \textbf{18.29} & \textbf{20.90} & \textbf{21.15} & \textbf{0.231} \\
\rowcolor{ourshl} & Object-effect & 15.49 & 15.68 & 16.60 & 18.65 & 0.353 \\
\bottomrule
\end{tabular}
\end{table}

\begin{figure}[!htb]
\centering
\includegraphics[width=\textwidth]{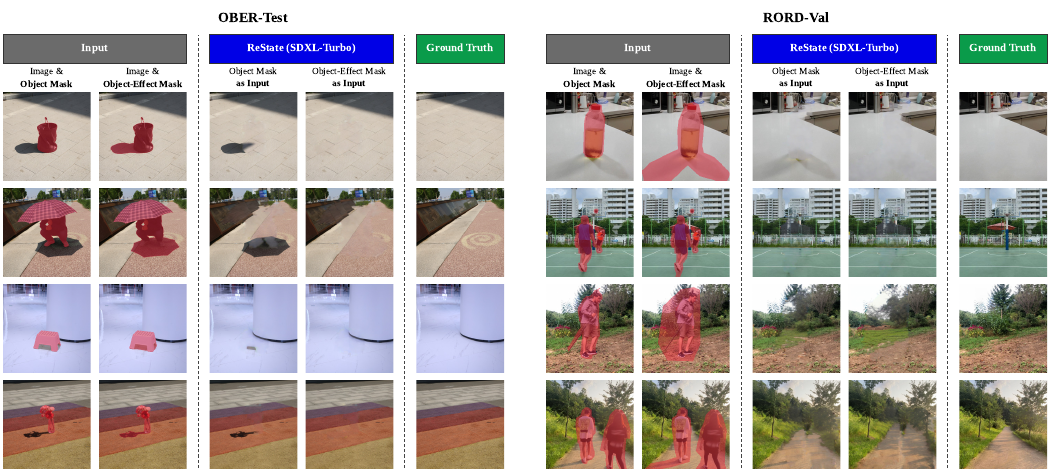}
\caption{\textbf{Object mask versus object-effect mask as input.} Left: OBER-Test. Right: RORD-Val. From left to right in each block: input with the object mask and with the object-effect mask overlaid in red, ReState with the object mask, ReState with the object-effect mask, and the ground truth. With the object mask, the shadow or reflection of the object remains; the effect regions of RORD-Val are much larger than those of OBER-Test.}
\label{fig:appendix-effect}
\end{figure}

% ============================================================
\section{Additional Comparisons and Extensions}
\label{app:extensions}

\subsection{Few-Step Text-Driven Editors}
\label{app:editing}

TurboEdit~\citep{turboedit2024}, InfEdit~\citep{infedit2024}, and ChordEdit~\citep{chordedit2026} edit an image from a source prompt to a target prompt on a few-step model (Section~\ref{sec:related}). They take a pair of prompts and no mask, so we evaluate them on object removal as follows. For each image of OBER-Test, a vision-language model (Qwen2.5-VL-7B, fixed templates, greedy decoding) names the object inside the mask and writes a source caption with the object and a target caption without it. Each method runs its official code with its default backbone and settings. InfEdit supports local blending, so we additionally run a variant that receives the ground-truth object mask in its blending step; this gives an editor the same spatial information as ReState.

Table~\ref{tab:appendix-editing} reports the results. Inside the mask, the three editors stay at the level of the output without any intervention ($9.29$ to $9.95$~dB, against $9.86$~dB for SDXL-Turbo in Table~\ref{tab:main}): the object largely remains in the image. With the ground-truth mask, InfEdit preserves the background as well as ReState ($27.84$ against $27.59$~dB), but its masked PSNR changes only from $9.29$ to $9.69$~dB. Knowing where to edit is therefore not sufficient. Removal also needs a target for what lies behind the object, and a prompt that omits the object does not provide one. Text-driven editors cover many edit types with a prompt alone, whereas ReState addresses mask-based removal only and, in return, builds this background target from the image itself and writes it into the clean prediction. We regard the two lines as complementary. Figure~\ref{fig:appendix-editing} shows examples.

\begin{table}[H]
\centering
\caption{\textbf{Few-step text-driven editors leave the object in place, also when the mask is given.} Object removal on OBER-Test. Each editor runs its official code with its default backbone. The third row gives InfEdit the ground-truth (GT) object mask. The background (Bg.) PSNR is computed outside the mask against the input image.}
\label{tab:appendix-editing}
\small
\setlength{\tabcolsep}{3.5pt}
\renewcommand{\arraystretch}{1.15}
\begin{tabular}{llc cccc c}
\toprule
& & & \multicolumn{4}{c}{Masked} & Bg. \\
\cmidrule(lr){4-7}\cmidrule(lr){8-8}
Method & Backbone & Steps & PSNR$\uparrow$ & SSIM$\uparrow$ & LPIPS$\downarrow$ & DISTS$\downarrow$ & PSNR$\uparrow$ \\
\midrule
TurboEdit\venue{SIGGRAPH Asia'24} & SDXL-Turbo & 4 & 9.95 & 0.225 & 0.420 & 0.290 & 18.72 \\
InfEdit\venue{CVPR'24} & LCM-Dreamshaper-v7 & 12 & 9.29 & 0.146 & 0.437 & 0.304 & 17.71 \\
InfEdit $+$ GT mask & LCM-Dreamshaper-v7 & 12 & 9.69 & 0.162 & 0.427 & 0.298 & \textbf{27.84} \\
ChordEdit\venue{CVPR'26} & SD-Turbo & 1 & 9.89 & 0.190 & 0.418 & 0.291 & 18.42 \\
\rowcolor{ourshl}\textbf{ReState (Ours)} & SDXL-Turbo & 4 & \textbf{18.75} & \textbf{0.498} & \textbf{0.194} & \textbf{0.129} & 27.59 \\
\bottomrule
\end{tabular}
\end{table}

\begin{figure}[!htb]
\centering
\includegraphics[width=0.84\textwidth]{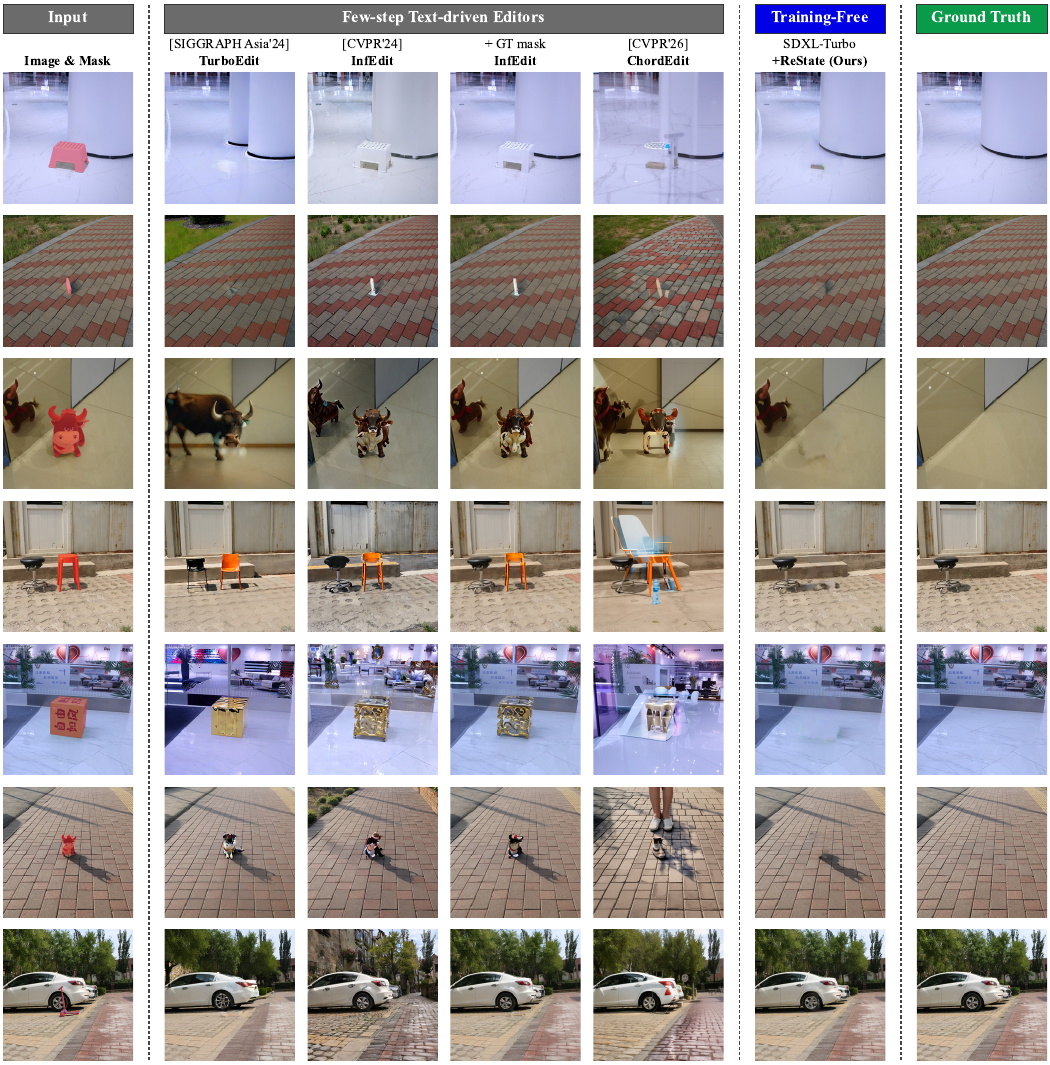}
\caption{\textbf{Few-step text-driven editors on object removal.} OBER-Test. From left to right: input, mask, TurboEdit, InfEdit, InfEdit with the ground-truth mask, ChordEdit, ReState, and the ground truth. The editors change the appearance of the object or of the whole image, but the object remains. No contour is drawn on the results.}
\label{fig:appendix-editing}
\end{figure}

% ------------------------------------------------------------
\subsection{DiT Backbones}
\label{app:dit}

We apply ReState to five DiT backbones trained with flow-matching objectives: FLUX.2-klein-4B and FLUX.2-klein-9B~\citep{flux2klein4b,flux2klein9b}, FLUX~3 Image Klein~\citep{flux3klein2026}, Qwen-Image-Flash~\citep{qwenimageflash2026}, and Z-Image-Turbo~\citep{zimage2025}. The method is defined on the clean prediction (Appendix~\ref{app:parameterization}), and three parts are adapted. First, the clean prediction is obtained from the predicted velocity, the weight of Eq.~\ref{eq:update} is evaluated at the equivalent noise level $s = \sigma/(1-\sigma)$ of the next level, and the corrected prediction is mapped back to a velocity, so the sampler runs unmodified. FLUX~3 Image Klein is a flow-map model that also takes the next noise level; its output, the average velocity over the step, is used in place of the velocity in the same conversion. Second, object-excluding attention masks the keys of the object image tokens in the transformer blocks. Third, the retrieval features are taken from one transformer block at the first forward. As in the UNet backbones, no forward is added.

\textbf{Spatial unit.} All five models generate at $1024^2$ with a VAE of downsampling factor $8$, and each transformer token packs a $2\times2$ block of latent positions, that is, $16\times16$ pixels. The token features of the first forward are therefore bilinearly upsampled from the $64^2$ token grid to the $128^2$ latent grid, and the evidence $B$ and the correction are computed per latent position exactly as in the UNet backbones. The corrected latent is packed into tokens again at every forward.

\textbf{Settings and results.} Each model runs its native number of forwards with an empty prompt, and sampling starts from the source latent noised to the start level of Table~\ref{tab:appendix-dit-settings}. The feature block, $\lambda$, and $r$ are selected per model on 20 development images drawn from RORD-Val, not on OBER-Test. Outputs are resized to $512^2$ and evaluated with the protocol of Appendix~\ref{app:stats}. On OBER-Test, ReState improves all four metrics on all five models, raising masked PSNR from $10.0$--$11.3$~dB to $17.3$--$19.6$~dB (Table~\ref{tab:appendix-dit}; Figure~\ref{fig:appendix-dit}).

\textbf{Limitations.} Although the formulation carries over directly, the spatial unit differs from that of a UNet: a DiT reads and writes patch tokens, whereas the retrieval and the correction are defined per latent position. How the correction interacts with the token unit and with flow-matching trajectories requires a dedicated analysis, as Appendix~\ref{app:analysis} provides for the UNet backbones. We therefore report these results as a first extension.

\begin{table}[H]
\centering
\caption{\textbf{Settings of the DiT backbones.} The start level is the noise level $\sigma$ of the flow-matching schedule at which sampling starts from the noised source latent. The feature block provides the retrieval features at the first forward.}
\label{tab:appendix-dit-settings}
\small
\setlength{\tabcolsep}{7pt}
\renewcommand{\arraystretch}{1.15}
\begin{tabular}{l cc l cc}
\toprule
Backbone & NFE & Start level $\sigma$ & Feature block & $\lambda$ & $r$ \\
\midrule
FLUX.2-klein-4B & 4 & 0.975 & single-stream block 3 & 8 & 0.01 \\
FLUX.2-klein-9B & 4 & 0.975 & single-stream block 8 & 8 & 0.05 \\
FLUX~3 Image Klein & 16 & 0.974 & single-stream block 15 & 2 & 0.01 \\
Qwen-Image-Flash & 4 & 0.923 & double-stream block 40 & 2 & 0.01 \\
Z-Image-Turbo & 8 & 0.960 & main layer 15 & 2 & 0.01 \\
\bottomrule
\end{tabular}
\end{table}

\begin{table}[!htb]
\centering
\caption{\textbf{ReState transfers to DiT backbones with flow-matching objectives.} OBER-Test. All models generate at $1024^2$ with their native number of forwards (NFE).}
\label{tab:appendix-dit}
\small
\setlength{\tabcolsep}{7pt}
\renewcommand{\arraystretch}{1.15}
\begin{tabular}{ll c cccc}
\toprule
Backbone & Method & NFE & PSNR$\uparrow$ & SSIM$\uparrow$ & LPIPS$\downarrow$ & DISTS$\downarrow$ \\
\midrule
FLUX.2-klein-4B & No intervention & 4 & 10.89 & 0.261 & 0.362 & 0.264 \\
\rowcolor{ourshl} & $+$ \textbf{ReState (Ours)} & 4 & \textbf{19.44} & \textbf{0.550} & \textbf{0.190} & \textbf{0.118} \\
FLUX.2-klein-9B & No intervention & 4 & 11.22 & 0.278 & 0.360 & 0.259 \\
\rowcolor{ourshl} & $+$ \textbf{ReState (Ours)} & 4 & \textbf{19.58} & \textbf{0.555} & \textbf{0.187} & \textbf{0.118} \\
FLUX~3 Image Klein & No intervention & 16 & 11.26 & 0.290 & 0.355 & 0.257 \\
\rowcolor{ourshl} & $+$ \textbf{ReState (Ours)} & 16 & \textbf{18.30} & \textbf{0.505} & \textbf{0.215} & \textbf{0.128} \\
Qwen-Image-Flash & No intervention & 4 & 10.91 & 0.260 & 0.407 & 0.283 \\
\rowcolor{ourshl} & $+$ \textbf{ReState (Ours)} & 4 & \textbf{17.95} & \textbf{0.471} & \textbf{0.226} & \textbf{0.151} \\
Z-Image-Turbo & No intervention & 8 & 10.03 & 0.211 & 0.395 & 0.298 \\
\rowcolor{ourshl} & $+$ \textbf{ReState (Ours)} & 8 & \textbf{17.30} & \textbf{0.451} & \textbf{0.246} & \textbf{0.163} \\
\bottomrule
\end{tabular}
\end{table}

\begin{figure}[!htb]
\centering
\includegraphics[width=0.84\textwidth]{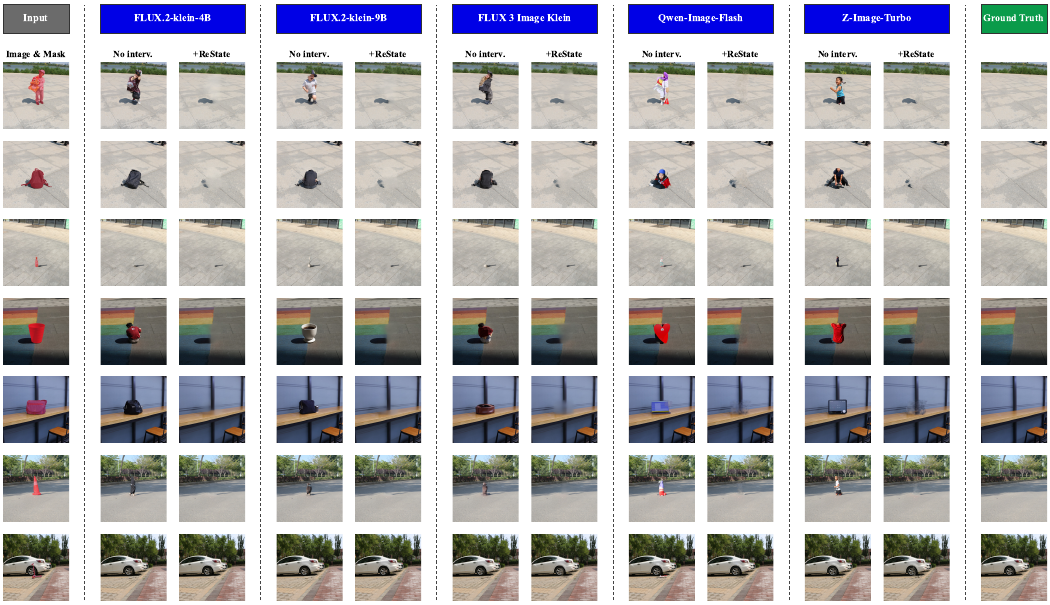}
\caption{\textbf{ReState on DiT backbones.} OBER-Test. For each of FLUX.2-klein-4B, FLUX.2-klein-9B, FLUX~3 Image Klein, Qwen-Image-Flash, and Z-Image-Turbo, the output without intervention and the output of ReState are shown next to the input, the mask, and the ground truth. No contour is drawn on the results.}
\label{fig:appendix-dit}
\end{figure}

% ============================================================
\clearpage
\section{Qualitative Results}
\label{app:qualitative}

\subsection{Qualitative Results on OBER}
\label{app:qual-ober}
\begin{figure}[H]
\centering
\includegraphics[width=\textwidth,height=0.87\textheight,keepaspectratio]{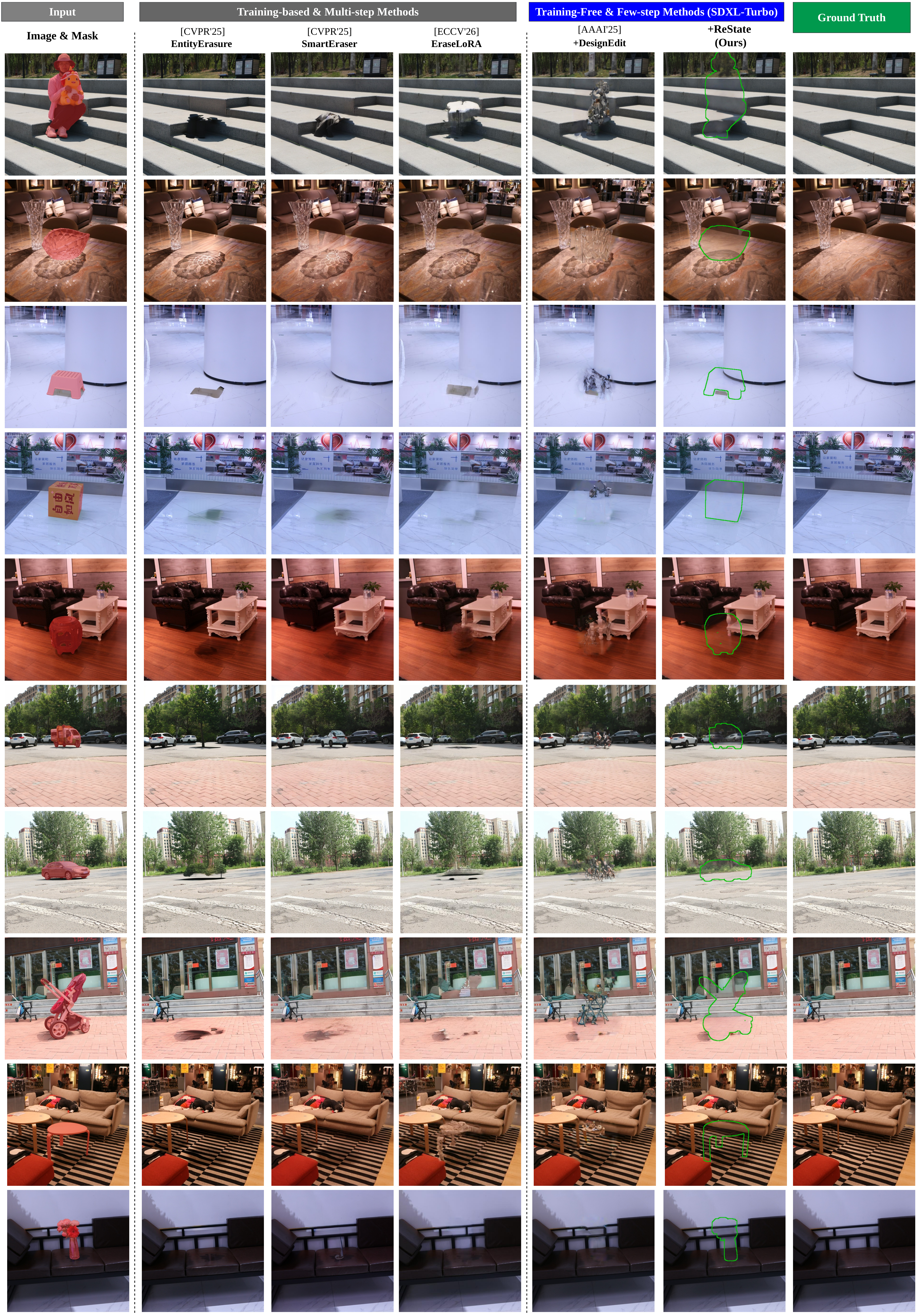}
\vspace{-8pt}
\caption{\textbf{Qualitative results on OBER-Test.}}
\label{fig:qual-ober}
\end{figure}

\clearpage
\subsection{Qualitative Results on RORD}
\label{app:qual-rord}
\begin{figure}[H]
\centering
\includegraphics[width=\textwidth,height=0.87\textheight,keepaspectratio]{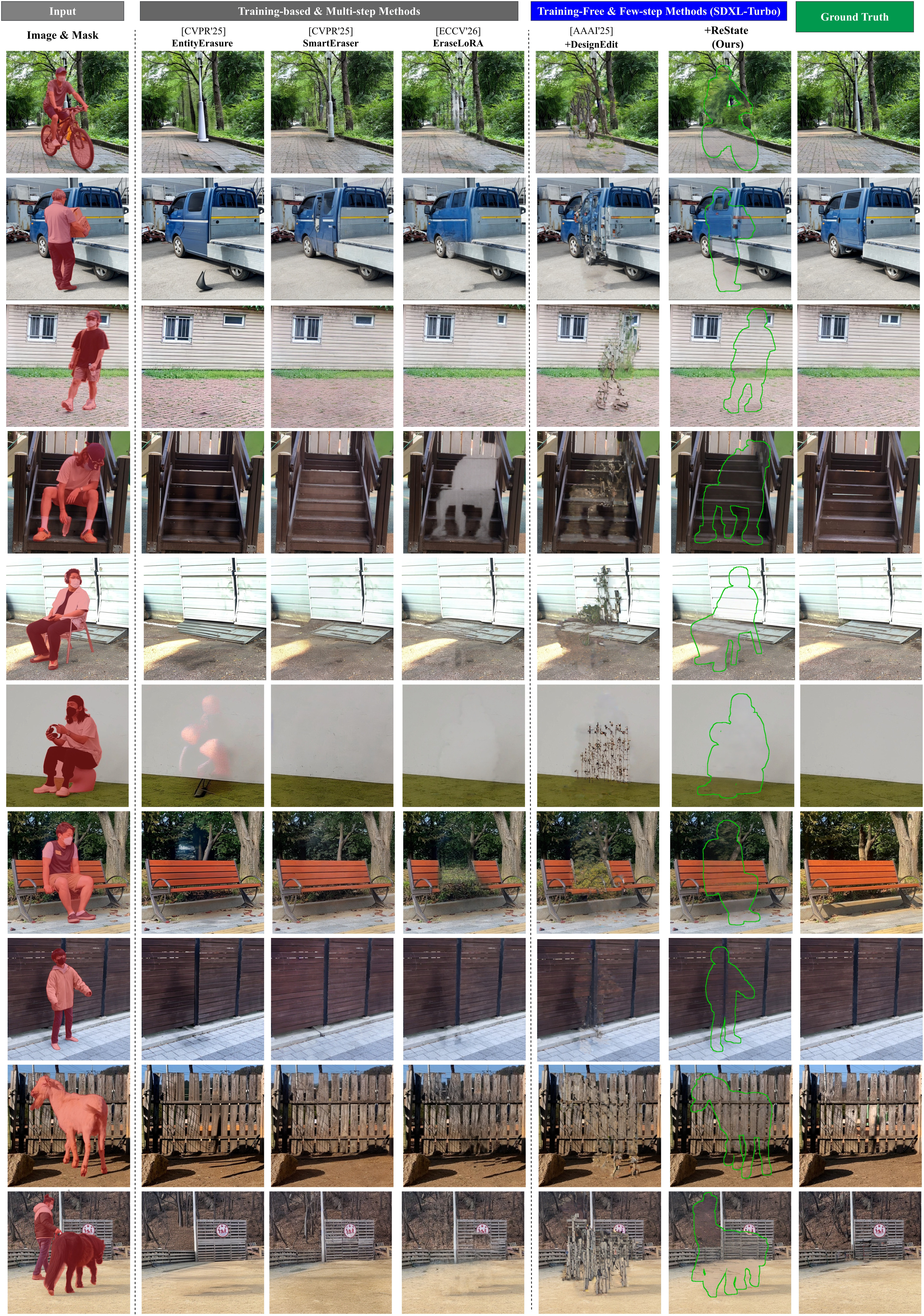}
\vspace{-8pt}
\caption{\textbf{Qualitative results on RORD.}}
\label{fig:qual-rord}
\end{figure}

\clearpage
\subsection{Qualitative Results on MULAN}
\label{app:qual-mulan}
\begin{figure}[H]
\centering
\includegraphics[width=0.70\textwidth]{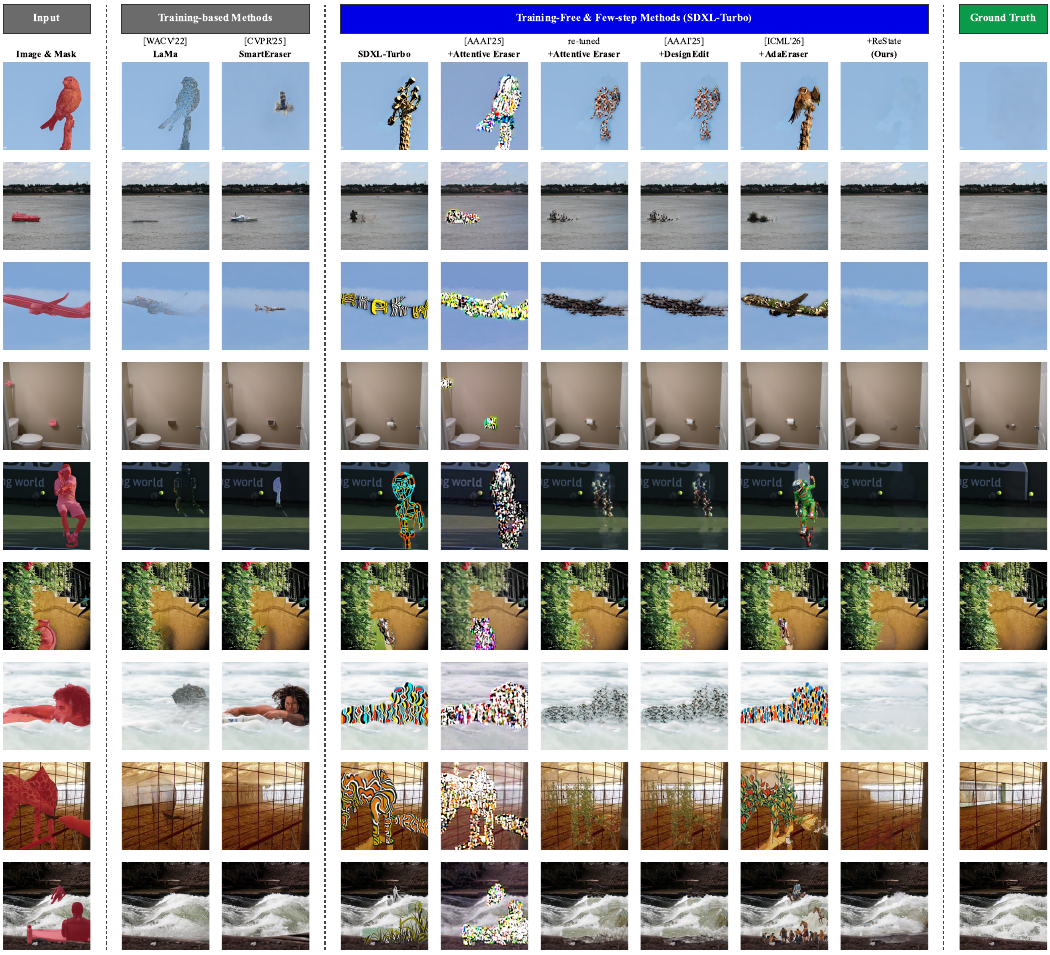}
\caption{\textbf{Qualitative results on MULAN.} From left to right: input with the mask overlaid in red, LaMa, SmartEraser, SDXL-Turbo without intervention, Attentive Eraser, re-tuned Attentive Eraser, DesignEdit, AdaEraser, ReState, and the ground truth.}
\label{fig:appendix-mulan}
\end{figure}

\subsection{Fill Texture under Magnification}
\label{app:qual-limit}
\begin{figure}[H]
\centering
\includegraphics[width=0.54\textwidth]{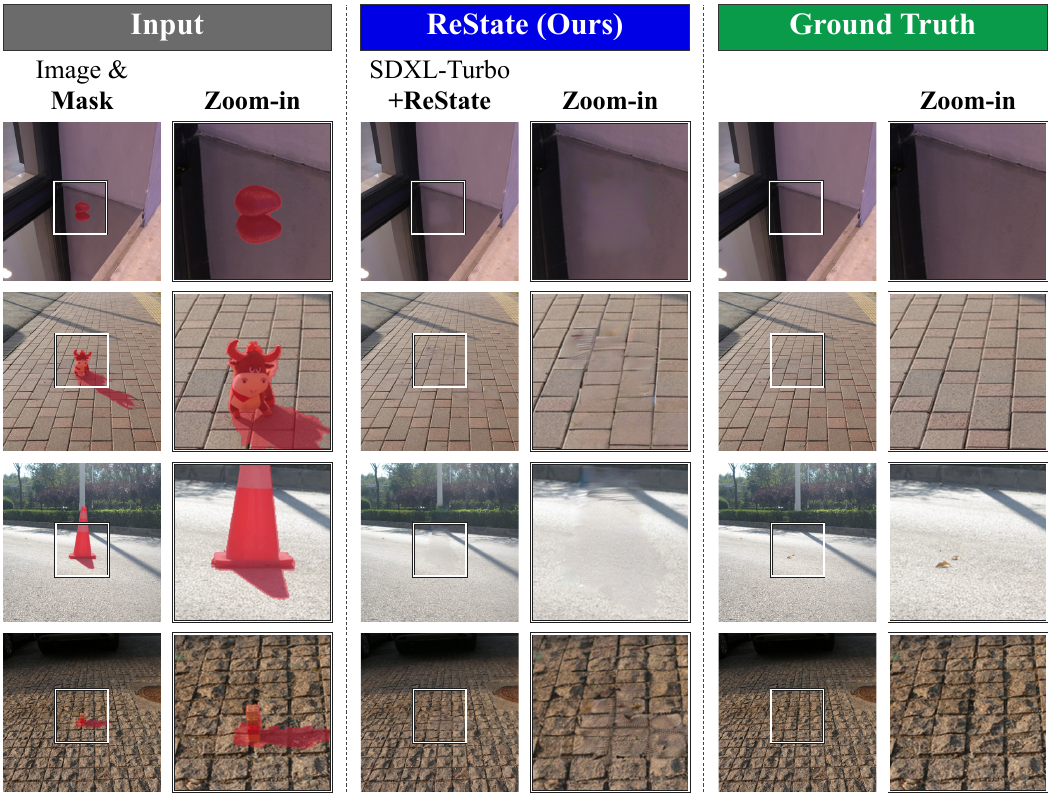}
\caption{\textbf{The fill is plausible at the original scale and softer than the ground truth under magnification.} OBER-Test, object-effect mask as input. The white box marks the region shown at $3\times$ magnification. The object and its shadow are removed in all rows; the zoom-in shows a softer, less regular texture, increasingly from top to bottom.}
\label{fig:appendix-limit}
\end{figure}

\clearpage
\section[User Study and VLM Evaluation]{User Study \& VLM Evaluation}
\label{app:user-study}

Since masked PSNR and SSIM can favor smooth fills (Appendix~\ref{app:metrics}), we also conduct a human evaluation (Figure~\ref{fig:user-study-screen}). For each image, a participant is shown the input with the removal region outlined in red alongside the anonymized outputs of five methods in randomized order: LaMa~\citep{lama2022}, ReState, and the three training-free baselines on SDXL-Turbo (DesignEdit, Attentive Eraser with the re-tuned SARG scale of Appendix~\ref{app:ports}, and AdaEraser). Participants mark every output in which the object is cleanly removed (i.e., the target object is absent and the masked region blends naturally with the surrounding background).

We evaluate $40$ images drawn uniformly at random: $20$ from OBER-Test with the object-effect mask (as in Table~\ref{tab:appendix-effect-ober}) and $20$ from the MULAN pairs of Appendix~\ref{app:mulan} with the object mask. In total, $15$ participants, none of whom is an author or was otherwise involved in this work, evaluate all $40$ images, yielding $600$ total responses. Because multiple outputs can be selected for a single image, we report each method's selection rate, defined as the fraction of responses in which it is marked (Table~\ref{tab:user-study}). Across the $40$ images, ReState achieves the highest selection rate ($61.8\%$), followed by LaMa ($51.5\%$), while all three training-free baselines remain below $18\%$. Restricting the same records to ReState and LaMa, $223$ of the $600$ responses select ReState but not LaMa, $161$ select LaMa but not ReState, $148$ select both, and $68$ select neither.

To complement the human evaluation at scale, we repeat the study using a vision-language model (VLM) as an automated judge. Claude Opus 5.5 is presented with the same images, instructions, and criterion, and each image is queried five times with independently randomized output orderings and no conversation history ($200$ queries total, all adhering to the required output format). The model is invoked via the non-interactive mode of the Claude Code command-line tool with all tools disabled (retaining only the tool's default system prompt); the exact prompt is shown at the end of this section. ReState again achieves the highest selection rate ($66.0\%$), followed by LaMa ($57.0\%$). Across the $200$ image--method pairs, the selection rates of human participants and the VLM exhibit a Spearman correlation of $\rho = 0.75$, and their majority decisions agree on $176$ of $200$ pairs ($88.0\%$).

\begin{figure}[H]
\centering
\includegraphics[width=0.9\textwidth]{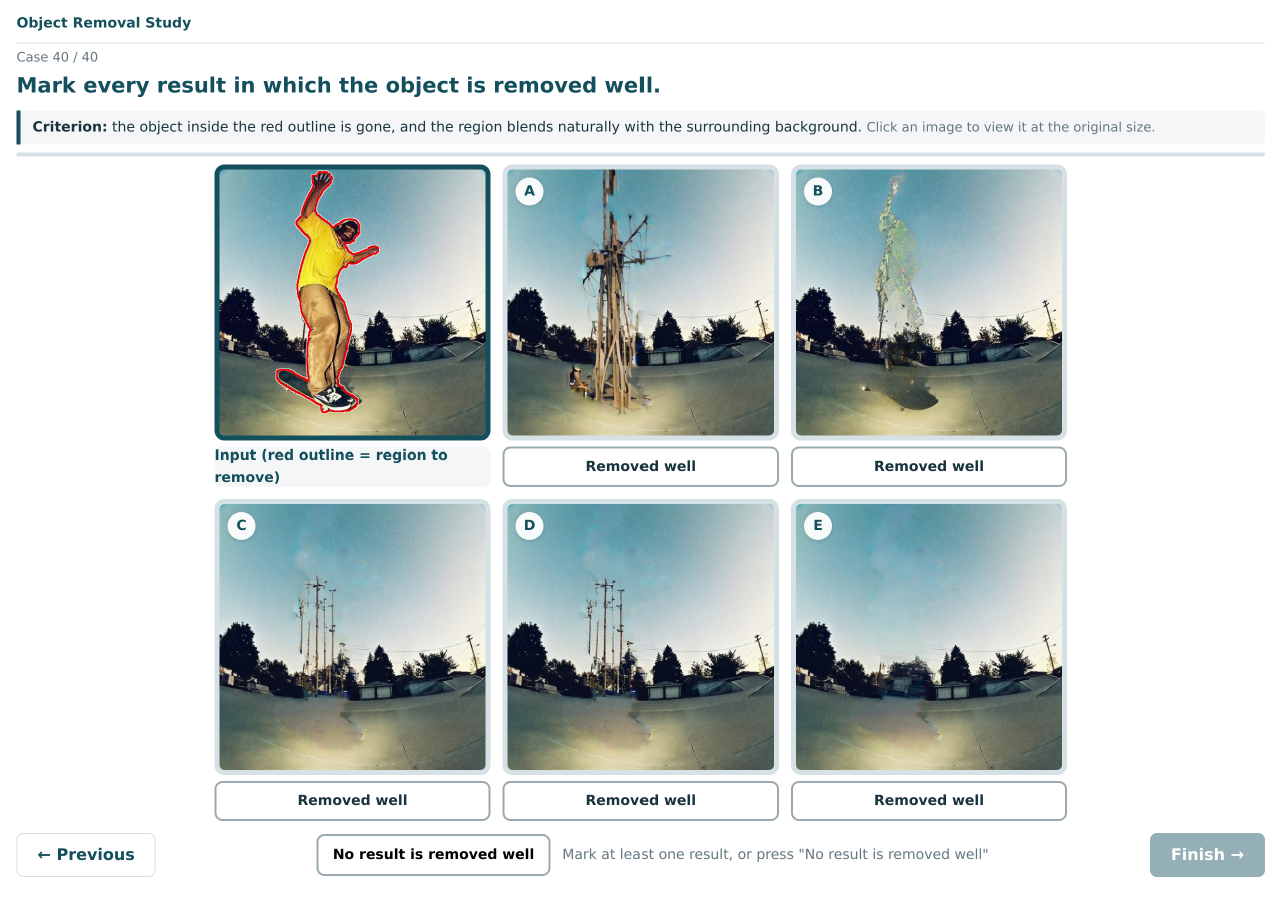}
\caption{\textbf{User study interface.} The input image with the removal region outlined in red and the five anonymized outputs shown in randomized order.}
\label{fig:user-study-screen}
\end{figure}
\vspace{-8pt}

\begin{table}[H]
\centering
\caption{\textbf{ReState achieves the highest selection rate among the five methods in both the human and VLM studies.} Selection rates across the $40$ evaluated images: $15$ human participants ($600$ responses) and a vision-language model (VLM), Claude Opus 5.5, queried five times per image ($200$ responses). All training-free methods use SDXL-Turbo with four steps. Bold marks the highest rate in each column and underline the second highest.}
\label{tab:user-study}
\small
\setlength{\tabcolsep}{10pt}
\renewcommand{\arraystretch}{1.15}
\begin{tabular}{l c cc}
\toprule
Method & Training-free & \shortstack{Multi-Select\\User Study (\%)} & \shortstack{Multi-Select\\VLM Study (\%)} \\
\midrule
LaMa\venue{WACV'22} & \textcolor{red}{\ding{55}} & \underline{51.5} & \underline{57.0} \\
Attentive Eraser, re-tuned\venue{AAAI'25} & \textcolor{green!60!black}{\checkmark} & 15.2 & 2.5 \\
DesignEdit\venue{AAAI'25} & \textcolor{green!60!black}{\checkmark} & 17.2 & 7.5 \\
AdaEraser\venue{ICML'26} & \textcolor{green!60!black}{\checkmark} & 2.2 & 0.0 \\
\rowcolor{ourshl}\textbf{ReState (Ours)} & \textcolor{green!60!black}{\checkmark} & \textbf{61.8} & \textbf{66.0} \\
\bottomrule
\end{tabular}
\end{table}
\vspace{-6pt}

{\scriptsize\noindent\textbf{Images.} OBER-Test: 003, 008, 009, 014, 032, 034, 036, 041, 045, 091, 100, 119, 121, 141, 146, 148, 149, 154, 160, 161. MULAN: 000000049384, 000000049901, 000000097754, 000000101979, 000000123208, 000000135316, 000000174504, 000000175012, 000000192827, 000000211013, 000000219094, 000000258027, 000000260696, 000000267717, 000000477494, 000000508215, 000000518604, 000000522612, 000000551214, 000000569750.\par}

\vspace{6pt}
{\noindent\textbf{Prompt of the VLM evaluation.} Every call sends the text below, followed by the input image after the line ``Input image:'' and by the five results after the lines ``Result A:'' to ``Result E:''.\par}
\begin{tcolorbox}[enhanced, colback=gray!4, colframe=black!35, boxrule=0.4pt, arc=1pt, left=6pt, right=6pt, top=4pt, bottom=4pt]
\small\ttfamily\raggedright
You are taking part in a study of object removal results.\\
The first image is the input image: the region to remove is outlined in red.\\
The next five images are five results for this input, labeled A, B, C, D and E.\\[4pt]
Mark every result in which the object is removed well.\\
Criterion: the object inside the red outline is gone, and the region blends naturally with the surrounding background.\\
Mark all results that meet the criterion. If no result meets it, answer NONE.\\[4pt]
End your reply with one line in exactly this form:\\
ANSWER: \textless letters separated by commas, or NONE\textgreater\\
for example \textasciigrave ANSWER: A, C\textasciigrave{} or \textasciigrave ANSWER: NONE\textasciigrave.
\end{tcolorbox}
\clearpage

\end{document}